\documentclass{article}

\usepackage{microtype}
\usepackage{graphicx}
\usepackage{subfigure}
\usepackage{booktabs} 

 \usepackage{hyperref}

\usepackage[accepted]{icml2018}

\usepackage{graphicx}  
\usepackage{algorithm}
\usepackage{algorithmic}
\usepackage{subfigure}
\usepackage{multirow}

\usepackage{amsfonts}       
\usepackage{amsmath}
\usepackage{amssymb}
\usepackage{nicefrac}       
\usepackage{microtype}      
\usepackage[english]{babel}

\usepackage{amsthm}
\usepackage{algorithm}
\usepackage{algorithmic}
\usepackage{color}
\usepackage{bm}

\theoremstyle{plain}
\newtheorem{thm}{Theorem}
\newtheorem{lem}[thm]{Lemma}
\newtheorem{prop}[thm]{Proposition}
\newtheorem{cor}{Corollary}

\theoremstyle{definition}

\theoremstyle{remark}
\newtheorem*{rem}{Remark}

\renewcommand{\algorithmicrequire}{\textbf{Input:}}
\renewcommand{\algorithmicensure}{\textbf{Output:}}

\newcommand{\phiv}{\boldsymbol \phi}

\newcommand{\omegav}{\boldsymbol \omega}
\newcommand{\alphav}{\boldsymbol \alpha}

\newcommand{\Sigmav}{\boldsymbol \Sigma}
\newcommand{\muv}{\boldsymbol \mu}

\newcommand{\zerov}{\boldsymbol 0}

\newcommand{\fv}{\mathbf{f}}

\newcommand{\Jv}{\mathbf{J}}
\newcommand{\yv}{\mathbf{y}}
\newcommand{\nv}{\mathbf{n}}

\newcommand{\zv}{\mathbf{z}}

\newcommand{\xv}{\mathbf{x}}

\newcommand{\Rv}{\mathbf{R}}

\newcommand{\Tv}{\bm{T}}

\newcommand{\Gv}{\mathbf{G}}
\newcommand{\bv}{\mathbf{b}}

\newcommand{\Iv}{\mathbf{I}}

\newcommand{\argmax}{\operatornamewithlimits{argmax}}
\newcommand{\argmin}{\operatornamewithlimits{argmin}}
\newcommand{\tabincell}[2]{\begin{tabular}{@{}#1@{}}#2\end{tabular}}
\usepackage{caption}
\icmltitlerunning{Message Passing SVGD}

\begin{document}

\twocolumn[
\icmltitle{Message Passing Stein Variational Gradient Descent}



\icmlsetsymbol{equal}{*}

\begin{icmlauthorlist}
\icmlauthor{Jingwei Zhuo}{to}
\icmlauthor{Chang Liu}{to}
\icmlauthor{Jiaxin Shi}{to}
\icmlauthor{Jun Zhu}{to}
\icmlauthor{Ning Chen}{to}
\icmlauthor{Bo Zhang}{to}
\end{icmlauthorlist}

\icmlaffiliation{to}{Dept. of Comp. Sci. \& Tech., BNRist Center, State Key Lab for Intell. Tech. \& Sys., THBI Lab, Tsinghua University, Beijing, 100084, China}

\icmlcorrespondingauthor{Jingwei Zhuo}{zjw15@mails.tsinghua.edu.cn}
\icmlcorrespondingauthor{Jun Zhu}{dcszj@tsinghua.edu.cn}

\icmlkeywords{Message Passing, Graphical Models, Variational Inference}

\vskip 0.3in
]



\printAffiliationsAndNotice{}  

\begin{abstract}
Stein variational gradient descent (SVGD) is a recently proposed particle-based Bayesian inference method, which has attracted a lot of interest due to its remarkable approximation ability and particle efficiency compared to traditional variational inference and Markov Chain Monte Carlo methods. 
However, we observed that particles of SVGD tend to collapse to modes of the target distribution, and this particle degeneracy phenomenon becomes more severe with higher dimensions.
Our theoretical analysis finds out that there exists a negative correlation between the dimensionality and the repulsive force of SVGD which should be blamed for this phenomenon.
We propose \emph{Message Passing SVGD} (MP-SVGD) to solve this problem.
By leveraging the conditional independence structure of probabilistic graphical models (PGMs), MP-SVGD converts the original high-dimensional global inference problem into a set of local ones over the Markov blanket with lower dimensions.
Experimental results show its advantages of preventing vanishing repulsive force in high-dimensional space over SVGD, and its particle efficiency and approximation flexibility over other inference methods on graphical models.
\end{abstract}

\section{Introduction}

Stein variational gradient descent (SVGD) \cite{Liu2016SVBP} is a recently proposed inference method.
To approximate an intractable but differentiable target distribution, 
it constructs a set of particles iteratively along the optimal gradient direction in a vector-valued reproducing kernel Hilbert space (RKHS) towards minimizing the KL divergence.
SVGD does not confine the approximation within parametric families as commonly done in traditional variational inference (VI) methods.
Besides, SVGD is more particle efficient than traditional Markov Chain Monte Carlo (MCMC) methods: 
it generates diverse particles due to the deterministic repulsive force induced by kernels instead of Monte Carlo randomness.
These benefits make SVGD an appealing method and gain a lot of interest \cite{Pu2017vae,haarnoja2017rldeep,liu2017svpg,feng2017steingan}.

As a kernel-based method, the performance of SVGD relies on the choice of kernels and corresponding RKHS.
In previous work, an isotropic vector-valued RKHS with a kernel defined by some distance metric (Euclidean distance) over all the dimensions is used. Examples include the RBF kernel \cite{Liu2016SVBP} and the IMQ kernel \cite{Gorham2017IMQ}.
However, as discussed in \citet{Aggarwal2001surprising} and \citet{Ramdas2015decreasing}, distance metrics and corresponding kernels suffer from the curse of dimensionality.
Thus, a natural question is, 
is the performance of SVGD also affected by the dimensionality?

We observe that the dimensionality negatively affects the performance of SVGD: its particles tend to collapse to modes and this phenomenon becomes more severe with higher dimensions. 
To understand this phenomenon, we analyze the impact of dimensionality on the repulsive force, which is critical for SVGD to work as an inference method for minimizing the KL divergence,
and attribute the reason partially to the negative correlation between the repulsive force and the dimensionality under some assumption about the variational distribution through theoretical analysis and experimental verifications.
Our analysis takes an initial step towards understanding the non-asymptotic behavior of SVGD with a finite number of particles, which is important since inferring high dimensional distributions with a limited computational and storage resource is common in practice.
We propose Message Passing SVGD (MP-SVGD) to solve this problem when the target distribution is compactly described via a probabilistic graphical model (PGM) and thus the conditional independence structure can be leveraged.
MP-SVGD converts the original high-dimensional inference problem into a set of local ones with lower dimensions according to a decomposition of the KL divergence, and solves each local problem iteratively in an RKHS with a local kernel defined over the Markov blanket. 
Experimental results on both synthetic and real-world settings demonstrate the power of MP-SVGD over SVGD and other inference methods on graphical models.


\textbf{Related work} The idea of converting a global inference problem into several local ones is not new. Traditional methods such as (loopy) belief propagation (BP) \cite{Pearl1988probabilistic}, expectation propagation (EP) \cite{Minka2001EP} and variational message passing (VMP) \cite{Winn2005VMP} all share this spirit. 
However, 
VMP makes a strong mean-field and conjugate exponential family assumption, EP requires an exponential family approximation, and loopy BP does not guarantee $q$ to be a globally valid distribution as it relaxes 
the solution of marginals to be in an outer bound of the marginal polytope~\citep{Wainwright2008graphical}. 
Moreover, loopy BP requires further approximation in message to handle complex potentials, which restricts its expressive power. 
For example, Nonparametric BP (NBP) \cite{Sudderth2003NBP} approximates the messages with mixtures of Gaussians; Particle BP (PBP) \cite{Ihler2009PBP} approximates the message using an important sampling approach with either the local potential or the estimated beliefs as proposals; and Expectation Particle BP (EPBP) \cite{Lienart2015EPBP} extends PBP with adaptive proposals produced by EP.
Another drawback of loopy BP and its variants is that except some special cases where beliefs are tractable (e.g., Gaussian BP), numerical integration is required when using beliefs in subsequent tasks like evaluating the expectation over some test function.

On the other hand, MCMC methods like Gibbs sampling avoid these problems since the expectation can be estimated directly from samples.
However, Gibbs sampling can only be used in some cases where the conditional distribution can be sampled efficiently (e.g., \citet{Martens2010parallelizable}). 

Compared to the aforementioned methods, MP-SVGD is more appealing since it requires neither tractable conditional distribution nor restrictions over potentials, which makes it suitable as a general purpose inference tool for graphical models with differentiable densities.

Finally, we note that the idea of improving SVGD over graphical models by leveraging the conditional independence property was developed concurrently and independently by \citet{Wang2018graph}. 
The difference between their work and ours lies in the derivation of the method and the implications that are explored.
\citet{Wang2018graph} also observed the particle degeneracy phenomenon of SVGD and proposed a similar method called Graphical SVGD by introducing graph structured kernels and corresponding Kernelized Stein Discrepancy (KSD). Rather than that, MP-SVGD is derived by a decomposition of the KL divergence.
Moreover, we develop a theoretical explanation for the particle degeneracy phenomenon by analyzing the relation between the dimensionality and the repulsive force.

\section{Preliminaries}


Given an intractable distribution $p(\xv)$ where $\xv = [x_1,...,x_D]^\top \in \mathcal{X} \subset \mathbb{R}^D$, variational inference aims to find a tractable distribution $q(\xv)$ supported on $\mathcal{X}$ to approximate $p(\xv)$ by minimizing some distribution measure, e.g., the (exclusive) KL divergence $\mathrm{KL}(q\|p)$. Instead of assigning a parametric assumption over $q(\xv)$, Stein variational gradient descent (SVGD) \cite{Liu2016SVBP} constructs $q(\xv)$ from some initial distribution $q_0(\xv)$ via a sequence of density transformations induced by the transformation on random variable: $\Tv(\xv) = \xv + \epsilon \phiv(\xv)$, where $\epsilon$ is the step size and $\phiv(\cdot) : \mathcal{X} \to \mathbb{R}^D$ denotes the transformation direction.
To be tractable and flexible, $\phiv$ is restricted to a vector-valued reproducing kernel Hilbert space (RKHS) $\mathcal{H}^D = \mathcal{H}_0 \times \cdots \times \mathcal{H}_0$, where $\mathcal{H}_0$ is the scalar-valued RKHS of kernel $k(\cdot,\cdot)$ which is chosen to be positive definite and in the Stein class of $p$ \cite{Liu2016KSD}.
Examples include the RBF kernel $k(\xv,\yv) = \exp\big(-\|\xv-\yv\|_2^2 / (2h) \big)$ \cite{Liu2016SVBP} and the IMQ kernel $k(\xv,\yv) = 1/\sqrt{1+\|\xv-\yv\|_2^2 / (2h)}$ \cite{Gorham2017IMQ}, where the bandwidth $h$ is commonly chosen according to the median heuristic \cite{Scholkopf2001kernel}\footnote{$h = \mathrm{med}^2$, where $\mathrm{med}$ is the median of the pairwise distances $\|\xv-\yv\|_2$, $\xv,\yv \sim q$.}.

Now, let $q_{[\Tv]}$ denote the density of the transformed random variable $\Tv(\xv) = \xv + \epsilon \phiv(\xv)$ where $\xv \sim q$ and $\epsilon$ is small enough so that $\Tv$ is invertible. Under this notion, we have
\begin{equation}\label{Eq:SVGD}
\min_{\|\phiv\|_{\mathcal{H}^D} \leq 1} \nabla_\epsilon \mathrm{KL}(q_{[\Tv]} \| p) |_{\epsilon = 0} = - \max_{\|\phiv\|_{\mathcal{H}^D} \leq 1} \mathbb{E}_{\xv \sim q} [\mathcal{A}_p \phiv(\xv)],
\end{equation}
where $\mathcal{A}_p$ is the Stein operator and
$$
\mathcal{A}_p \phiv(\xv) = \phi(\xv)^\top \nabla_{\xv} \log p(\xv) + \mathrm{trace}(\nabla_{\xv} \phiv(\xv)).
$$
As shown in \cite{Liu2016KSD} and \cite{Chwialkowski2016KSD}, the right hand side of Eq. (\ref{Eq:SVGD}) has a closed-form solution $\phiv^*/\|\phiv^*\|_{\mathcal{H}^D}$ where
\begin{equation}\label{Eq:svg}
\phiv^*(\xv) = \mathbb{E}_{\yv \sim q}\left[ k(\xv,\yv)\nabla_{\yv} \log p(\yv) + \nabla_{\yv} k(\xv,\yv) \right].
\end{equation}
$\phiv^*(\xv)$ consists of two parts: the kernel smoothed gradient $\Gv(\xv;p,q) = \mathbb{E}_{\yv \sim q}\left[ k(\xv,\yv)\nabla_{\yv} \log p(\yv)\right]$  and the repulsive force $\Rv(\xv; q) = \mathbb{E}_{\yv \sim q}\left[ \nabla_{\yv} k(\xv,\yv) \right]$.
By doing the transformation $\xv \leftarrow \xv + \epsilon \phiv^*(\xv)$ iteratively, $q_{[\Tv]}$ decreases the KL divergence along the steepest direction in $\mathcal{H}^D$.
The iteration ends when $\phiv^*(\xv) \equiv 0$ and thus $\Tv$ reduces to the identity mapping.
This condition is equivalent to $q=p$ when $k(\xv,\yv)$ is strictly positive definite in a proper sense \cite{Liu2016KSD,Chwialkowski2016KSD}.

In practice, a set of particles $\{\xv^{(i)}\}_{i=1}^M$ are used to practically represent $q(\xv)$ by the empirical distribution $\hat{q}_M(\xv) = \frac{1}{M}\sum_{i=1}^M \delta_{\xv^{(i)}}(\xv)$, where $\delta$ is the Dirac delta function.
These particles are are updated iteratively via $\xv^{(i)} \leftarrow \xv^{(i)} + \epsilon \hat{\phiv}^*(\xv^{(i)})$, where
\begin{equation}\label{Eq:particlesvgd}
\hat{\phi}^*(\xv) = \mathbb{E}_{\yv \sim \hat{q}_M} \left[k(\xv,\yv) \nabla_{\yv} \log p(\yv) + \nabla_{\yv} k(\xv, \yv)\right].
\end{equation}
When $M = 1$, the update rule becomes $\xv^{(1)} \leftarrow \xv^{(1)} + \epsilon \nabla_{\xv^{(1)}} \log p(\xv^{(1)})$, which corresponds to the gradient method to find the mode of $p(\xv)$.

\section{Towards Understanding the Impact of Dimensionality for SVGD}

Kernel-based methods suffers from the curse of dimensionality. For example, \citet{Ramdas2015decreasing} demonstrates that the power of nonparametric hypothesis testing using Maximum Mean Discrepancy (MMD) drops polynomially with increasing dimensions. 
It is reasonable to suspect that SVGD also suffers from similar problems.
In fact, as shown in the upper row of Fig.~\ref{figtoy}, even for $p(\xv) = \mathcal{N}(\xv|\zerov, \Iv)$, the performance of SVGD is unsatisfactory: though it correctly estimates the mean of $p(\xv)$, it underestimates the marginal variance, and this problem becomes more severe with higher dimensions. 
In other words, SVGD suffers from particle degeneracy in high dimensions in which particles become less diverse and tend to collapse to modes of $p(\xv)$.
In this section, we take an initial step toward understanding this 
through analyzing\footnote{All the derivation details can be found in the supplemental materials. We consider only the RBF kernel $k(\xv,\yv) = \exp\left(-\frac{\|\xv-\yv\|_2^2}{2h}\right)$. The IMQ kernel also shares similar properties and corresponding results can be found in the supplemental materials as well.} 
the repulsive force $\Rv(\xv;q)$.

\begin{figure}[!htb]
	\centering
  \includegraphics[width=0.5\textwidth]{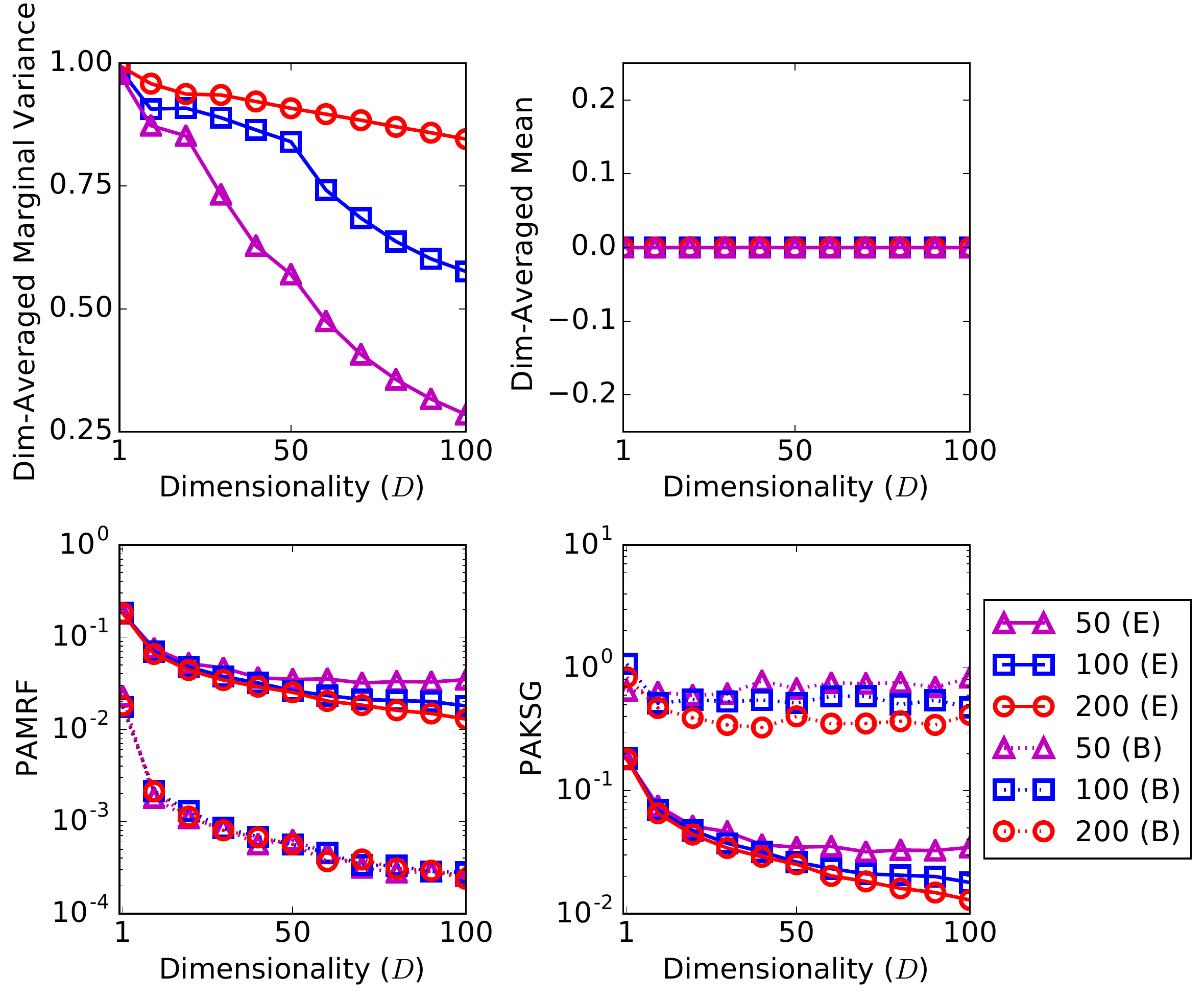}
	\caption{Results for inferring $p(\xv)=\mathcal{N}(\xv|\zerov,\Iv)$ using SVGD with the RBF kernel, where particles are initialized by $\mathcal{N}(\xv|\zerov,25 \Iv)$. Top two figures show the dimension-averaged marginal variance $\frac{1}{D}\sum_{d=1}^D \mathrm{Var}_{\hat{q}_M} (x_d)$ and mean $\frac{1}{D}\sum_{d=1}^D \mathbb{E}_{\hat{q}_M}[x_d]$ respectively, and bottom two figures show the particle-averaged magnitude of the repulsive force (PAMRF) $\frac{1}{M} \sum_{i=1}^M \|\Rv(\xv^{(i)};\hat{q}_M)\|_\infty$ and kernel smoothed gradient (PAKSG) $\frac{1}{M} \sum_{i=1}^M \|\Gv(\xv^{(i)};p,\hat{q}_M)\|_\infty$ respectively, at both the beginning (dotted;B) and the end of iterations (solid;E) with different number of particles $M = 50, 100$ and $200$.}
	\label{figtoy}
\end{figure}


First we highlight the importance of $\Rv(\xv;q)$. 
Referring to Eq. (\ref{Eq:svg}), we have $\phi^*(\xv) = \Gv(\xv;p,q)+\Rv(\xv;q)$, and
we can show that the kernel smoothed gradient $\Gv(\xv;p,q)$ corresponds to the steepest direction for maximizing $\mathbb{E}_{\xv \sim q}[\log p(\xv)]$, i.e.,
\begin{equation}\label{Eq:kernelgrad}
\frac{\Gv (\xv; p, q)}{\|\Gv(\xv;p,q)\|_{\mathcal{H}^D}} = \argmax_{\|\phiv\|_{\mathcal{H}^D} \leq 1} \nabla_{\epsilon} \mathbb{E}_{\zv \sim q_{[\Tv]}}[\log p(\zv)]|_{\epsilon = 0},
\end{equation}
where $\zv = \Tv(\xv) = \xv + \epsilon \phiv(\xv)$.
The convergence condition $\Gv(\xv;p,q) \equiv \zerov$ corresponds to $\nabla_{\yv} \log p(\yv) = \zerov$ for $q(\yv) \neq 0$, i.e., the optimal $q(\xv)$ collapses to modes of $p(\xv)$. 
This implies that without $\Rv(\xv;q)$, $\Gv(\xv;p,q)$ alone corresponds to the gradient method to find modes of $p(\xv)$.
So $\Rv(\xv;q)$ is critical for SVGD to work as an inference algorithm for minimizing the KL divergence. 

However, with a kernel measuring the global similarity in $\mathcal{X}\subset \mathbb{R}^D$ (e.g., the RBF kernel), the repulsive force becomes
$$
\Rv(\xv;q) = \mathbb{E}_{\yv \sim q}\left[ \exp\left(-\frac{\|\xv-\yv\|_2^2}{2h}\right) \frac{\xv - \yv}{h} \right].
$$
Unlike $\Gv(\xv;p,q)$ in which the bandwidth $h$ only appears as a denominator for $\|\xv-\yv\|_2^2$ and can be chosen using the median heuristic, the bandwidth in $\Rv(\xv;q)$ also appears as a denominator for $\xv-\yv$. As a result, finding a proper $h$ for $\Rv(\xv;q)$ will be hard 
and the magnitude of $\Rv(\xv;q)$ is bounded as
\begin{eqnarray}\label{Eq:repulmag}
\|\Rv(\xv; q)\|_\infty \leq \mathbb{E}_{\yv \sim q}\left[ \frac{2}{e}\cdot \frac{\|\xv-\yv\|_\infty}{\|\xv-\yv\|_2^2} \right]
\end{eqnarray}
for any $h > 0$.
Intuitively, when $\|\xv-\yv\|_\infty/\|\xv-\yv\|_2^2 \ll 1$ for most regions of $q$, $\|\Rv(\xv;q)\|_\infty$ would be small, making SVGD dynamics greatly dependent on $\Gv(\xv;p,q)$, especially in the beginning stage where $q$ does not match $p$ and $\|\Gv(\xv;p,q)\|_\infty$ is large.
Besides, though the theoretical convergence condition that $\phiv^*(\xv) \equiv \zerov$ {\it iff} $q = p$ still holds, the vanishing repulsive force weakens it in reducing the difference between $\Gv(\xv;p,q) \equiv \zerov$ and $\phiv^*(\xv) \equiv \zerov$. 
These characteristics would bring problems in practice when $q$ is approximated by a set of particles $\{\xv^{(i)}\}_{i=1}^M$:
the empirical convergence condition $\hat{\phi}^*(\xv^{(i)}) = \zerov, \forall i \in \{1,...,M\}$ does not guarantee\footnote{An extreme case is as follows: when $\xv^{(i)} = \xv^*$ with $\xv^* = \argmax_{\xv} \log p(\xv)$ (i.e., the MAP) holds for any $i \in \{1,...,M\}$, the empirical convergence condition is satisfied.} $\{\xv^{(i)}\}_{i=1}^M$ to be a good approximation of $p$, 
and the $\Gv(\xv;p,q)$-dominant dynamic would result in collapsing particles. 


Now, a natural question is, for which $q$ does this intuition hold?
One example is $q$ to be Gaussian as summarized in the following proposition:
\begin{prop}\label{thm:gaussian}
Given the RBF kernel $k(\xv,\yv)$ and $q(\yv)=\mathcal{N}(\yv|\muv,\Sigmav)$, the repulsive force satisfies
$$
\|\Rv(\xv;q)\|_\infty \leq \frac{\sqrt{D}}{\lambda_{\min}(\Sigmav)(\frac{D}{2}+1)(1+\frac{2}{D})^{\frac{D}{2}}} \|\xv-\muv\|_\infty,
$$
where $\lambda_{\min}(\Sigmav)$ is the smallest eigenvalue of $\Sigmav$. By using $\lim_{x \to 0}(1+x)^{1/x}=e$, we have $\|\Rv(\xv;q)\|_\infty \lesssim \|\xv-\muv\|_\infty / (\lambda_{\min}(\Sigmav)\sqrt{D})$.
\end{prop}
Proposition \ref{thm:gaussian} indicates that the upper bound of $\|\Rv(\xv;q)\|_\infty$ negatively correlates with $D$.
In practice, since $\Rv(\xv;\hat{q}_M)$ is an unbiased estimate of $\Rv(\xv;q)$, we can also bound $\|\Rv(\xv;\hat{q}_M)\|_\infty \lesssim \|\xv-\muv\|_\infty / (\lambda_{\min}(\Sigmav)\sqrt{D})$. Apart from the Gaussian distribution, we can prove that such a negative correlation exists for $\Rv(\xv;\hat{q}_M)$ in a more general case:
\begin{prop}\label{mainthm}
Let $k(\xv,\yv)$ be an RBF kernel. Suppose $q(\yv)$ is supported on a bounded set $\mathcal{X}$ which satisfies $\|\yv\|_\infty \leq C$ for $\yv \in \mathcal{X}$, and $\mathrm{Var}(y_d|y_1,...,y_{d-1}) \geq C_0$ almost surely for any $1 \leq d \leq D$. Let $\{\xv^{(i)}\}_{i=1}^M$ be a set of samples of $q$ and $\hat{q}_M$ the corresponding empirical distribution. Then, for any $\|\xv\|_\infty \leq C$, $\alpha, \delta \in (0,1)$, there exists $D_0 > 0$, such that for any $D > D_0$,
\begin{eqnarray}\label{tailrepulsiveforce}
\|R(\xv;\hat{q}_M)\|_\infty \leq \frac{2}{eD^{\alpha}}
\end{eqnarray}
holds with at least probability $1 - \delta$. 
\end{prop}
In proposition \ref{mainthm}, the bounded support assumption is relatively mild: examples include distributions defined on the images, in which the pixel intensity lies in a bounded interval. 
Requiring the conditional variance is larger than some constant reflects that the stochasticity for each dimension will not be eliminated by knowing the values of other dimensions, which is a quite strong assumption.
However, as evaluated in experiments, the negative correlation exists for $q$ even when such assumptions do not hold. Thus proposition 2 may be improved with weaker assumptions.

Given these intuitions, we would like to explain the particle degeneracy phenomenon in Fig.~\ref{figtoy}. As shown in the bottom row, there exists a negative correlation between $\|\Rv(\xv;\hat{q}_M)\|_\infty$ and $D$, at both the beginning and the end of iterations. In the beginning stage, $\|\Gv(\xv;p,\hat{q}_M)\|_\infty$ keeps almost unchanged while $\|\Rv(\xv;\hat{q}_M)\|_\infty$ negatively correlates with $D$. This implies that the SVGD dynamics becomes more $\Gv(\xv;p,\hat{q}_M)$-dominant with larger $D$ at the beginning. When converged, $\hat{\phiv}^*(\xv^{(i)}) = 0$, which corresponds to $\|\Gv(\xv^{(i)};p,\hat{q}_M)\|_\infty = \|\Rv(\xv^{(i)};\hat{q}_M)\|_\infty$. In this case, we find an interesting phenomenon that $\|\Rv(\xv;\hat{q}_M)\|_\infty$ tends to be constant with $M=50$ but the marginal variance still decreases with increasing dimensions. A possible explanation for this case is that assuming $q$ is Gaussian with $\Sigmav = \sigma^2 \Iv$, the variance $\sigma^2 = \lambda_{\min}(\Sigmav) \lesssim \|\xv^{(i)}-\muv\|_\infty/(\|\Rv(\xv^{(i)};\hat{q}_M)\|_\infty \sqrt{D})$ as proved in proposition \ref{thm:gaussian}. When $\|\Rv(\xv^{(i)};\hat{q}_M)\|_\infty$ is almost constant (and $\|\xv^{(i)}-\muv\|_\infty$ does not increase faster than $\sqrt{D}$), $\sigma^2$ will decrease as $D$ increases.
\section{Message Passing SVGD}

As discussed in Section 3, the unsatisfying property of SVGD 
comes from 
the negative correlation between the dimensionality and the repulsive force.
Though the high-dimensional nature of $p(\xv)$ is inevitable in practice, this problem can be solved for $p(\xv)$ with conditional independence structure, which is commonly described by probabilistic graphical models (PGMs).
Based on this idea, we propose \emph{Message Passing SVGD}, which converts the original high-dimensional inference problem 
into a set of local inference problems with lower dimensions. 

More specifically, we assume $p(\xv)$ can be factorized\footnote{Such a $p(\xv)$ is usually described using a factor graph, which unifies both directed and undirected graphical models. We refer the readers to \cite{Koller2009probabilistic} for details.} as $p(\xv) \propto \prod_{F \in \mathcal{F}} \psi_F(\xv_F)$ where the factor $F \subset \{1,...,D\}$ denotes the index set and $\xv_F=[x_d]_{d \in F}$. The Markov blanket $\Gamma_d = \cup\{F:F \ni d\} \setminus \{d\}$ contains neighborhood nodes of $d$ such that $p(x_d|\xv_{\neg d}) = p(x_d|\xv_{\Gamma_d})$.


\subsection{A Decomposition of the KL Divergence}

Our method relies on the key observation that we can decompose $\mathrm{KL}(q\|p)$ as 
\begin{equation}\label{Eq:KLdecomp}
\begin{aligned}
\mathrm{KL}(q\|p) & = \mathrm{KL}\big(q(x_d|\xv_{\neg d})q(\xv_{\neg d})\big\|p(x_d|\xv_{\Gamma_d})q(\xv_{ \neg d})\big) \\
&~~~~+ \mathrm{KL}\big(q(\xv_{\neg d})\big\|p(\xv_{\neg d})\big),
\end{aligned}
\end{equation}
where $\neg d = \{1,...,D\} \setminus \{d\}$ denotes the index set other than $d$.
Eq. (\ref{Eq:KLdecomp}) provides another perspective for minimizing $\mathrm{KL}(q\|p)$: instead of solving a global problem which minimizes $\mathrm{KL}(q\|p)$ over $q(\xv)$, we can iteratively solve a set of local problems which minimizes the localized divergence over $q(x_d|\xv_{\neg d})$ by keeping $q(\xv_{\neg d})$ fixed, i.e.,
\begin{equation} \label{eq:min-sub-kl}
\argmin_{q(x_d|\xv_{\neg d})} \mathrm{KL}\big(q(x_d|\xv_{\neg d})q(\xv_{\neg d})\big\|p(x_d|\xv_{\Gamma_d})q(\xv_{\neg d})\big).
\end{equation}
This idea resembles EP, which also performs local minimizations iteratively, however, for a localized version of the inclusive KL divergence $\mathrm{KL}(p\|q)$ \cite{Minka2001EP}. 
Another difference is that each local step in EP does not guarantees minimizing a global divergence \cite{Minka2005divergence},
while solving Problem (\ref{eq:min-sub-kl}) iteratively corresponds to minimizing the original $\mathrm{KL}(q\|p)$ due to the decomposition\footnote{In fact, when both $q$ and $p$ are differentiable, we can show that each localized divergence equals zero {\it iff} $q=p$, as detailed in the supplemental material.} in Eq. (\ref{Eq:KLdecomp}).

Eq. (\ref{Eq:KLdecomp}) requires decomposing $q(\xv)$ as $q(x_d|\xv_{\neg d})$ for each $d$, which makes it useless for VI methods with a parametric $q(\xv)$ except some special cases (e.g., $q(\xv)$ is Gaussian or fully factorized as $q(\xv)=\prod_{d=1}^D q(x_d)$). 
However, this decomposition is very suitable for transformation based methods like SVGD. 
Consider the transformation $\zv = \Tv(\xv) = [x_1,...,T_d(x_d),...,x_D]^\top$ for $\xv \sim q$,
where only the $d$th dimension is transformed and other dimensions are kept unchanged, we have $q_{[\Tv]}(\zv_{\neg d}) = q(\zv_{\neg d})$.
In other words, minimizing $\mathrm{KL}(q_{[\Tv]}\|p)$ over $T_d$ is equivalent to minimizing $\mathrm{KL}\big(q_{[T_d]}(x_d|\xv_{\neg d})q(\xv_{\neg d})\big\|p(x_d|\xv_{\Gamma_d})q(\xv_{\neg d})\big)$.

Thus SVGD can be applied to Problem~(\ref{eq:min-sub-kl}) directly, 
by following $T_d: x_d \to x_d + \epsilon\phi_d(\xv)$ where $\phi_d \in \mathcal{H}_0$ is associated with the global kernel $k:\mathcal{X} \times \mathcal{X} \to \mathbb{R}$, and solving $\min_{\|\phi_d\|_{\mathcal{H}_0} \leq 1} \nabla_\epsilon \mathrm{KL}(q_{[\Tv]}\|p)\big|_{\epsilon = 0}$.
This produces a coordinate-wise version of SVGD.
However, the high-dimensional problem still exists due to the global kernel.
To reduce the dimensionality, 
we further restrict the transformation to be dependent only on $x_d$ and its Markov blanket, i.e., $T_d: x_d \to x_d + \epsilon\phi_d(\xv_{S_d})$, $\phi_d \in \mathcal{H}_d$, where $S_d = \{d\} \cup \Gamma_d$. Here $\mathcal{H}_d$ is the RKHS induced by kernel $k_d: \mathcal{X}_{S_d}\times\mathcal{X}_{S_d}\to \mathbb{R}$, where $\mathcal{X}_{S_d} = \{\xv_{S_d}, \xv \in \mathcal{X}\}$.
By doing so, we have the following proposition\footnote{Proof can be found in the supplemental material.}:
\begin{prop}
Let $\zv = \Tv(\xv) = [x_1,...,T_d(x_d),...,x_D]^\top$ with $T_d: x_d \to x_d + \epsilon \phi_d(\xv_{S_d})$, $\phi_d \in \mathcal{H}_d$ where $\mathcal{H}_d$ is associated with the local kernel $k_d:\mathcal{X}_{S_d} \times \mathcal{X}_{S_d} \to \mathbb{R}$. Then, we have
\begin{equation}\label{mpsvgdprob}
\begin{aligned}
&\nabla_{\epsilon} \mathrm{KL}(q_{[\Tv]}||p) = \\
&~~~~~~\nabla_{\epsilon} \mathrm{KL}\big(q_{[T_d]}(z_d|\zv_{\Gamma_d})q(\zv_{\Gamma_d})\big\|p(z_d|\zv_{\Gamma_d})q(\zv_{\Gamma_d})\big),
\end{aligned}
\end{equation}
and the solution for $\min_{\|\phi_d\|_{\mathcal{H}_d} \leq 1} \nabla_{\epsilon} \mathrm{KL}(q_{[\Tv]}|\|p) |_{\epsilon = 0}$ is $\phi_d^*/\|\phi_d^*\|_{\mathcal{H}_d}$, where
\begin{equation}\label{Eq:mpsvgdphi}
\begin{aligned}
\phi_d^*(\xv_{S_d}) = & \mathbb{E}_{\yv_{S_d} \sim q}\big[ k_d(\xv_{S_d}, \yv_{S_d}) \nabla_{y_d} \log p(y_d|\yv_{\Gamma_d}) \\
&  + \nabla_{y_d} k_d(\xv_{S_d},\yv_{S_d}) \big].
\end{aligned}
\end{equation}
\end{prop}
As shown in Eq.~(\ref{Eq:mpsvgdphi}), computing $\phi^*(\xv_{S_d})$ only requires $\xv_{S_d} \in \mathcal{X}_{S_d}$ instead of $\xv \in \mathcal{X}$, which reduces the dimension from $D$ to $|S_d|$. This would alleviate the vanishing repulsive force problem
, especially for the case where $p(\xv)$ is highly sparse structured such that $|S_d| \ll D$ (e.g., pairwise Markov Random Fields), as verified in the experiments on both synthetic and real-world problems.

Similar to the original SVGD, the convergence condition $\phi_d^*(\xv_{S_d}) \equiv 0$ holds {\it iff} $q(x_d|\xv_{\Gamma_d})q(\xv_{\Gamma_d}) \equiv p(x_d|\xv_{\Gamma_d})q(\xv_{\Gamma_d})$, i.e., $q(x_d|\xv_{\Gamma_d}) \equiv p(x_d|\xv_{\Gamma_d})$ when $k_d(\xv_{S_d},\yv_{S_d})$ is strictly positive definite in a proper sense \cite{Liu2016KSD,Chwialkowski2016KSD}.
In other words, to reduce the dimension, we have to pay the price that $q$ is only conditionally consistent with $p$. 
This relaxation of $q$ also appears in traditional methods like loopy BP and its variants, in which only marginal consistency is guaranteed \cite{Wainwright2008graphical}.

\subsection{Markov Blanket based Kernels}
Now, the remaining question is the choice of the local kernel $k_d:\mathcal{X}_{S_d} \times \mathcal{X}_{S_d} \to \mathbb{R}$.
We can simply define $k_d(\xv_{S_d},\yv_{S_d}) = f\big(\|\xv_{S_d} - \yv_{S_d}\|_2^2 / (2h_{S_d})\big)$ for $f(z) = \exp(-z)$ (i.e., the RBF kernel), 
with the implicit assumption that all nodes in the Markov blanket contribute equally for node $d$. We call such a $k_d$ the {\it Single-Kernel}.
However, as the factorization of $p(\xv) \propto \prod_{F \in \mathcal{F}}\psi_F(\xv_F)$ is known, we can also define the {\it Multi-Kernel} where $k_d(\xv_{S_d},\yv_{S_d}) = \frac{1}{K}\sum_{F \ni d} f\big(\|\xv_F - \yv_F\|_2^2 / (2h_F)\big)$, where $K$ is the number of factors containing $d$, to reflect the assumption that nodes in different factors may contribute in a different way.
By doing so, $R_d(\xv;q) = \frac{1}{K} \sum_{F \ni d} \mathbb{E}_{\yv_F \in q}\big[ \nabla_{y_d} f\big(\|\xv_F - \yv_F\|_2^2 / (2h_F)\big) \big]$ and we can further reduce the dimension for $R_d(\xv;q)$ from $|S_d|$ to $\max\{|F|:F \ni d\}$.

\subsection{Final Algorithm}

Similar to SVGD, we use a set of particles $\{\xv^{(i)}\}_{i=1}^M$ to approximate $q$ and this procedure is summarized in Algorithm \ref{alg1}. According to the choice of kernels, we abbreviate MP-SVGD-s for the {\it Single-Kernel} and MP-SVGD-m for the {\it Multi-Kernel}.

\begin{algorithm}
\begin{algorithmic}
\caption{Message Passing SVGD}
\label{alg1}
\REQUIRE A differentiable target distribution $p(\xv)$ whose $d$th conditional distribution is $p(x_d|\xv_{\Gamma_d}) = p(x_d|\xv_{-d})$, and a set of initial particles $\{\xv^{(i)}\}_{i=1}^M$.
\ENSURE A set of particles $\{\xv^{(i)}\}_{i=1}^M$ as an approximation of $p(\xv)$.
\FOR{iteration $t$}
\FOR{$d \in \{1,...,D\}$}
\STATE $\xv^{(i)}_d \leftarrow \xv_d^{(i)} + \epsilon \hat{\phi}_d^*(\xv_{S_d}^{(i)})$ where $\epsilon$ is the stepsize, and
$$ 
\begin{aligned}
\hat{\phi}^*_d(\xv_{S_d}) & = \mathbb{E}_{\yv_{S_d} \sim \hat{q}_M} [ k_d(\xv_{S_d}, \yv_{S_d}) \nabla_{y_d} \log p(y_d|\yv_{\Gamma_d}) \\
&~~~~+ \nabla_{y_d} k_d(\xv_{S_d}, \yv_{S_d})].
\end{aligned}
$$
\ENDFOR
\ENDFOR
\end{algorithmic}
\end{algorithm}

Updating particles acts in a message passing way as shown in Figure \ref{fig:mp}: node $d$ receives particles (messages) from its neighbors (i.e., $\{\xv_{\Gamma_d}^{(i)}\}_{i=1}^M$); updates its own particles $\{\xv_d^{(i)}\}_{i=1}^M$; and sends them to its neighbors. Unlike loopy BP, each node sends the same message to its neighbors. This resembles VMP, where messages from the parent to its children in a directed graph are also identical \cite{Winn2005VMP}. 
\begin{figure}
	\centering
    \includegraphics[width=0.35\textwidth]{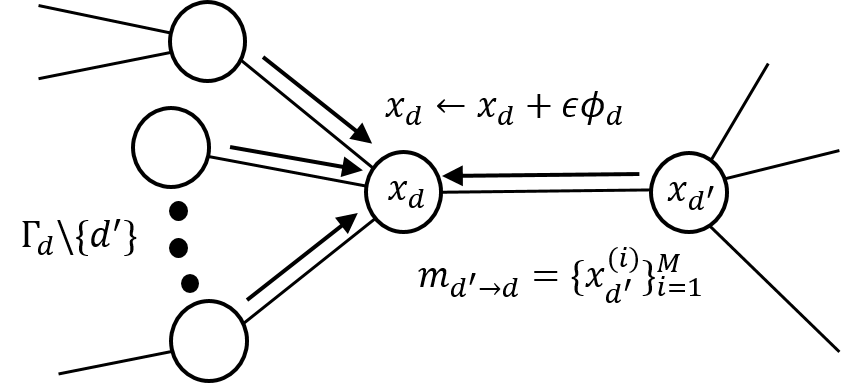}
    \caption{Message passing procedure for node $d$.}
    \label{fig:mp}
\end{figure}

\section{Experiments}
In this section, we experimentally verify our analysis and evaluate the performance of MP-SVGD with other inference methods on both synthetic and real-world examples. We use the RBF kernel with the bandwidth chosen by the median heuristic for all experiments.

\subsection{Synthetic Markov Random Fields}
\begin{figure*}[!htb]
	\centering
	\includegraphics[width=1\textwidth]{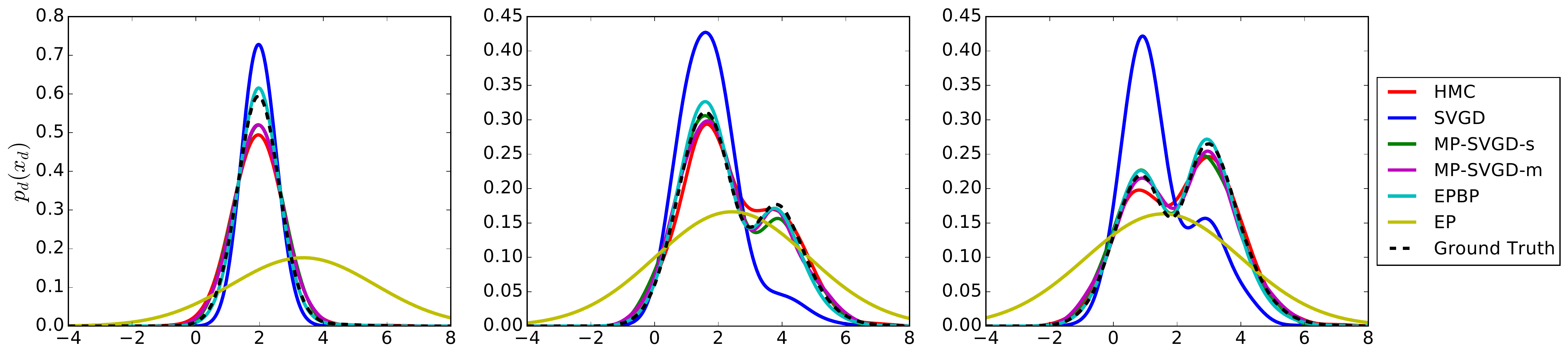}
	\caption{A qualitative comparison of inference methods with 100 particles (except EP) for marginal densities of three randomly selected nodes. Density curves of SVGD, MP-SVGD and HMC are estimated by kernel density estimator with RBF kernels. For EPBP, the curve is drawn by normalizing its beliefs over the fixed interval $[-5,10)$ with bin size $0.01$. }
	\label{fig5}
\end{figure*}
\begin{figure*}[!htb]
	\centering
	\includegraphics[width=1\textwidth]{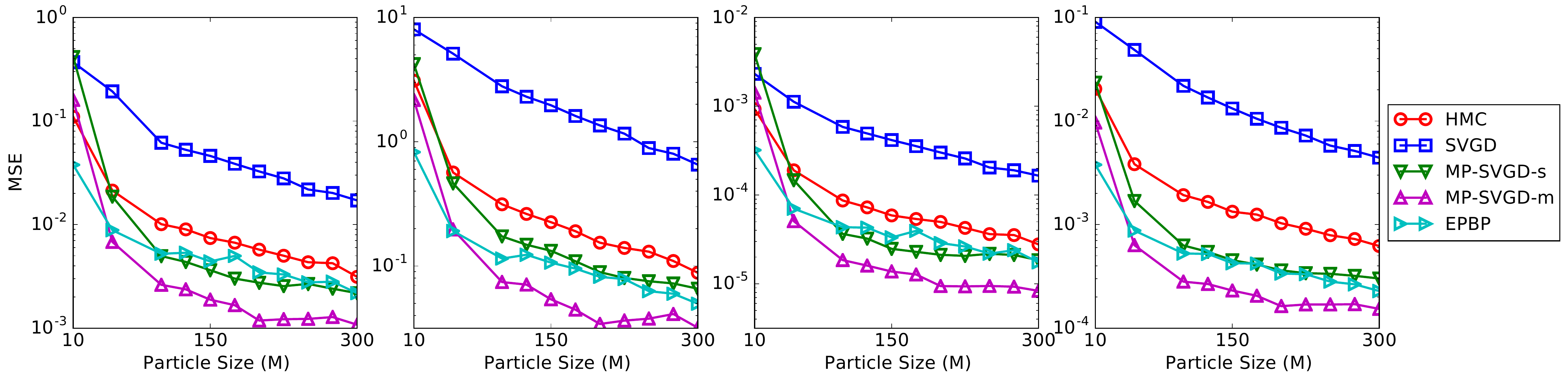}	
	\caption{A quantitative comparison of inference methods with varying number of particles. Performance is measured by the MSE of the estimation of expectation $\mathbb{E}_{\xv \sim \hat{q}_M}[\fv(\xv)]$ for test functions $\fv(\xv) = \xv$, $\xv^2$, $1/ (1+\exp(\omegav \circ \xv + \bv))$ and $\cos(\omegav \circ \xv + \bv)$, arranged from left to right, where $\circ$ denotes the element-wise product. Results are averaged over 10 random draws of $\omegav$ and $\bv$, where $\omegav, \bv \in \mathbb{R}^{100}$ with $\omega_d \sim \mathcal{N}(0,1)$ and $b_d \in \mathrm{Uniform}[0, 2\pi]$, $\forall d \in \{1,...,100\}$.}
	\label{fig6}
\end{figure*}
\begin{figure}[!htb]
	\centering
    \includegraphics[width=0.5\textwidth]{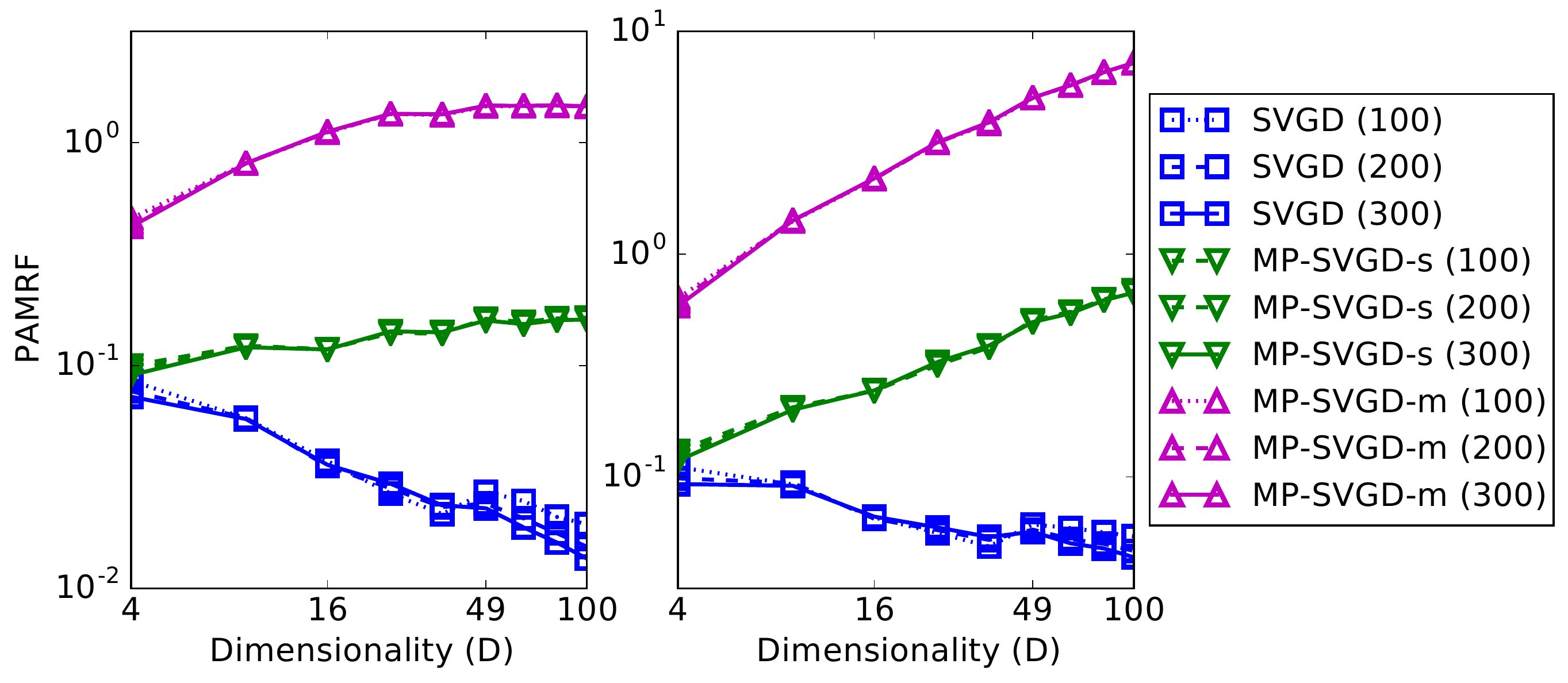}
    \caption{
    PAMRF $\frac{1}{M} \sum_{i=1}^M \|\Rv(\xv^{(i)};\hat{q}_M)\|_r$ for converged $\{\xv^{(i)}\}_{i=1}^M$ with $r = \infty$ (left) and $r = 2$ (right). The grid size ranges from $2\times 2$ to $10 \times 10$. The number in the bracket denotes the number of particles.}
    \label{fig:repul}
    \vspace{-0.27cm}
\end{figure}
We follow the settings of \cite{Lienart2015EPBP} and focus on a pairwise MRF on the 2D grid
$p(\xv) \propto \prod_{d \in \mathcal{V}} \psi_d(x_d) \prod_{(d,t)\in \mathcal{E}} \psi_{dt}(x_d,x_t)$ with the random variable in each node taking values on $\mathbb{R}$. The node and edge potentials are chosen such that $p(\xv)$ and its marginals are multimodal, non-Gaussian and skewed:
\begin{equation}
\left\{
\begin{array}{l}
\psi_d(x_d) = \alpha_1\mathcal{N}(x_d-y_d|-2,1) + \alpha_2\mathcal{G}(x_d-y_d|2,1.3) \\
\psi_{dt}(x_d,x_t) = \mathcal{L}(x_d-x_t|0,2)
\end{array}
,\right.
\end{equation}
where $y_d$ denotes the observed value at node $d$ and is initialized randomly as $y_d \sim \alpha_1\mathcal{N}(y_d-2|-2,1) + \alpha_2\mathcal{G}(y_d-2|2,1.3)$, $\mathcal{N}(x|\mu,\sigma^2) \propto \exp(-(x-\mu)^2/(2\sigma^2))$, $\mathcal{G}(x|\mu,\beta) \propto \exp(-(x-\mu)/\beta + \exp(-(x-\mu)/\beta))$ and $\mathcal{L}(x|\mu,\beta) \propto \exp(-|x-\mu|/\beta)$ denote the density of Gaussian, Gumbel and Laplace distribution, respectively. Parameters $\alpha_1$ and $\alpha_2$ are set to 0.6 and 0.4. 
We consider a $10 \times 10$ grid except Fig. \ref{fig:repul}, whose grid size ranges from $2 \times 2$ to $10 \times 10$.
All experimental results are averaged over 10 runs with random initializations.

Since $p(\xv)$ is intractable, we recover the ground truth by samples drawn by an adaptive version of Hamiltonian Monte Carlo (HMC) \cite{Neal2011HMC,Shi2017zhusuan}.
We run 100 chains in parallel with 40,000 samples for each chain after 10,000 burned-in, i.e. 4 million samples in total.

We compare MP-SVGD with SVGD, HMC, EP and EPBP \cite{Lienart2015EPBP}. 
Although widely used in MRFs, Gibbs sampling does not suit this task since the conditional distribution $p(x_d|\xv_{\Gamma_d})$ cannot be sampled directly. 
So we use uniformly randomly chosen particles from the 4 million ground truth samples as a strong baseline\footnote{This is strong in the sense that when all the 4 million samples are used, corresponding approximation error would be zero.} and regard the method as HMC.
For EP, we use the Gaussian distribution as the factors, and the moment matching step is done by numerical integration due to the non-Gaussian nature of $p(\xv)$.
EPBP is a variant of BP methods and the original state-of-the-art method on this task. 
It uses weighted samples to estimate the messages while other methods (except EP) use unweighted samples to approximate $p(\xv)$ directly.


Fig.~\ref{fig5} provides a qualitative comparison of these methods by visualizing the estimated marginal densities on each node.
As is shown, SVGD estimates the marginals undesirably in all cases. For both unimodal and bimodal cases, the SVGD curve exhibits a sharp, peaked behavior, indicating the collapsing trend of its particles. It is interesting to note that in the rightmost figure, more SVGD particles gather around the lower mode. One possible reason is that the lower marginal mode may correspond to a larger mode in the joint distribution, into which SVGD particles tends to collapse.
EP provides quite a crude approximation as expected, due to the mismatch between its Gaussian assumption and the non-Gaussian nature of $p(\xv)$.
Other methods perform similarly. 

Fig.~\ref{fig6} provides a quantitative comparison\footnote{We omit EP results for the clearness, the figure which includes EP results can be found in the supplementary materials.} in particle efficiency measured by the mean squared error (MSE) of estimated expectation with different particle sizes $M$.
For EPBP, $M$ denotes the number of node particles when approximating messages.
We observe that SVGD achieves the highest RMSE compared to other methods, reflecting its particle degeneracy.
Compared to HMC and EPBP, MP-SVGD achieves comparable and even better results.
Besides, MP-SVGD-m achieves lower RMSE than MP-SVGD-s, reflecting the benefits of MP-SVGD-m of leveraging more structural information in designing local kernels.

Fig.~\ref{fig:repul} compares MP-SVGD with SVGD in particle-averaged magnitude of the repulsive force (PAMRF) $\frac{1}{M}\sum_{i=1}^M \|\Rv(\xv^{(i)};\hat{q}_M)\|_r$ at the end of iterations for various dimensions.
As expected, the SVGD PAMRF negatively correlates with the dimensionality $D$ while the MP-SVGD PAMRF does not.
Besides, both of them exhibit roughly a log linear relationship with the dimensionality.
This verifies Proposition \ref{mainthm} 
that $\|\Rv(\xv^{(i)};\hat{q}_M)\|_\infty$ is upper bounded by $D^{-\alpha}$ for some constant $\alpha$.
This also reflects the power of MP-SVGD in preventing the repulsive force from being too small by reducing dimensionality using local kernels.
Besides, we observe that MP-SVGD-m PAMRF is higher than MP-SVGD-s PAMRF, which verifies our analysis regarding {\it Single-Kernel} and {\it Multi-Kernel}:
for the pairwise MRF, the dimension for {\it Single-Kernel} is $5$ at most (the Markov blanket and the node itself) while the dimension for {\it Multi-Kernel} is $2$ at most (the edge) and thus lower dimensionality corresponds to higher repulsive force.

\subsection{Image Denoising}
\begin{figure*}[!htb]
	\centering
	\includegraphics[width=0.16\textwidth]{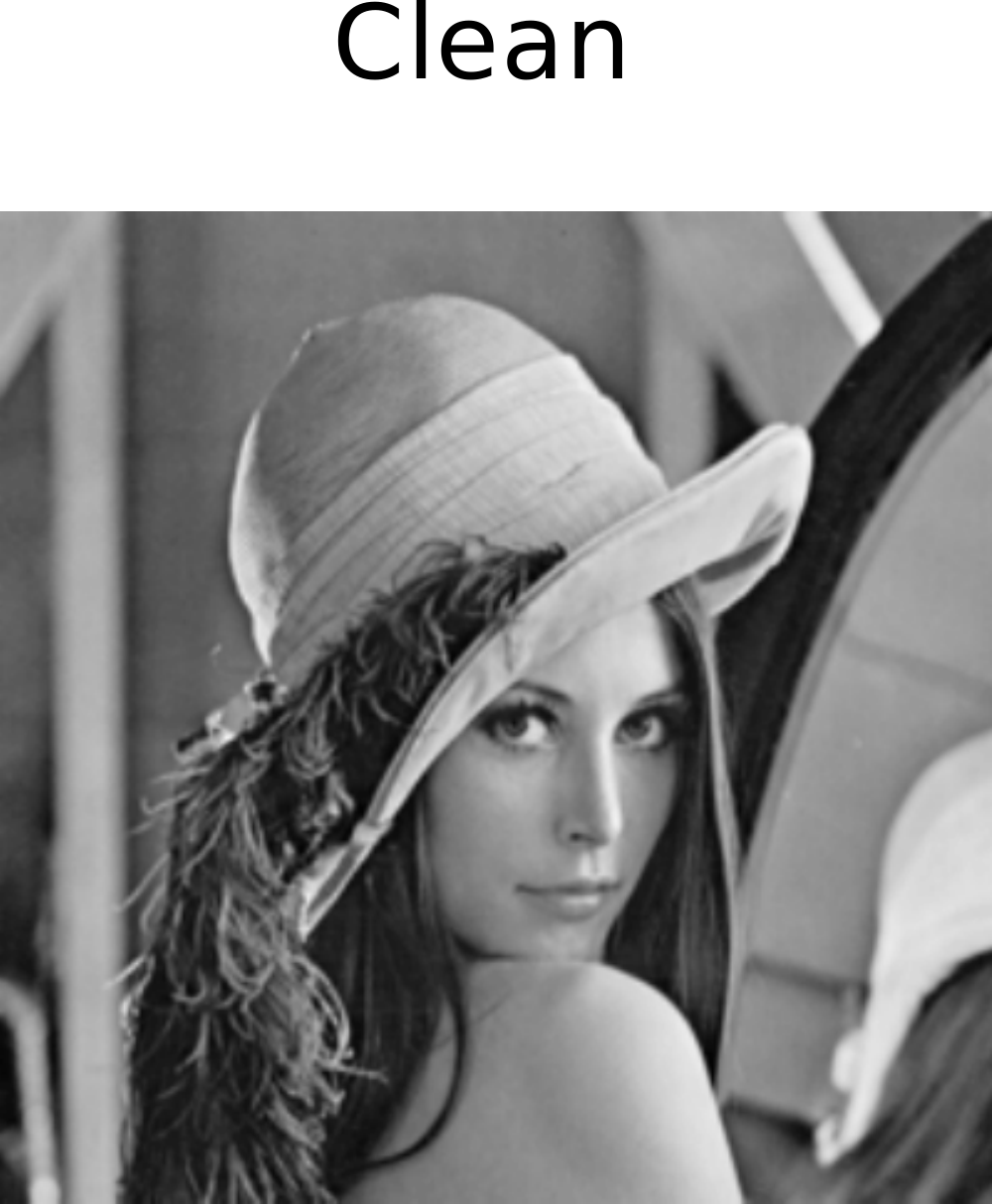}
	\includegraphics[width=0.16\textwidth]{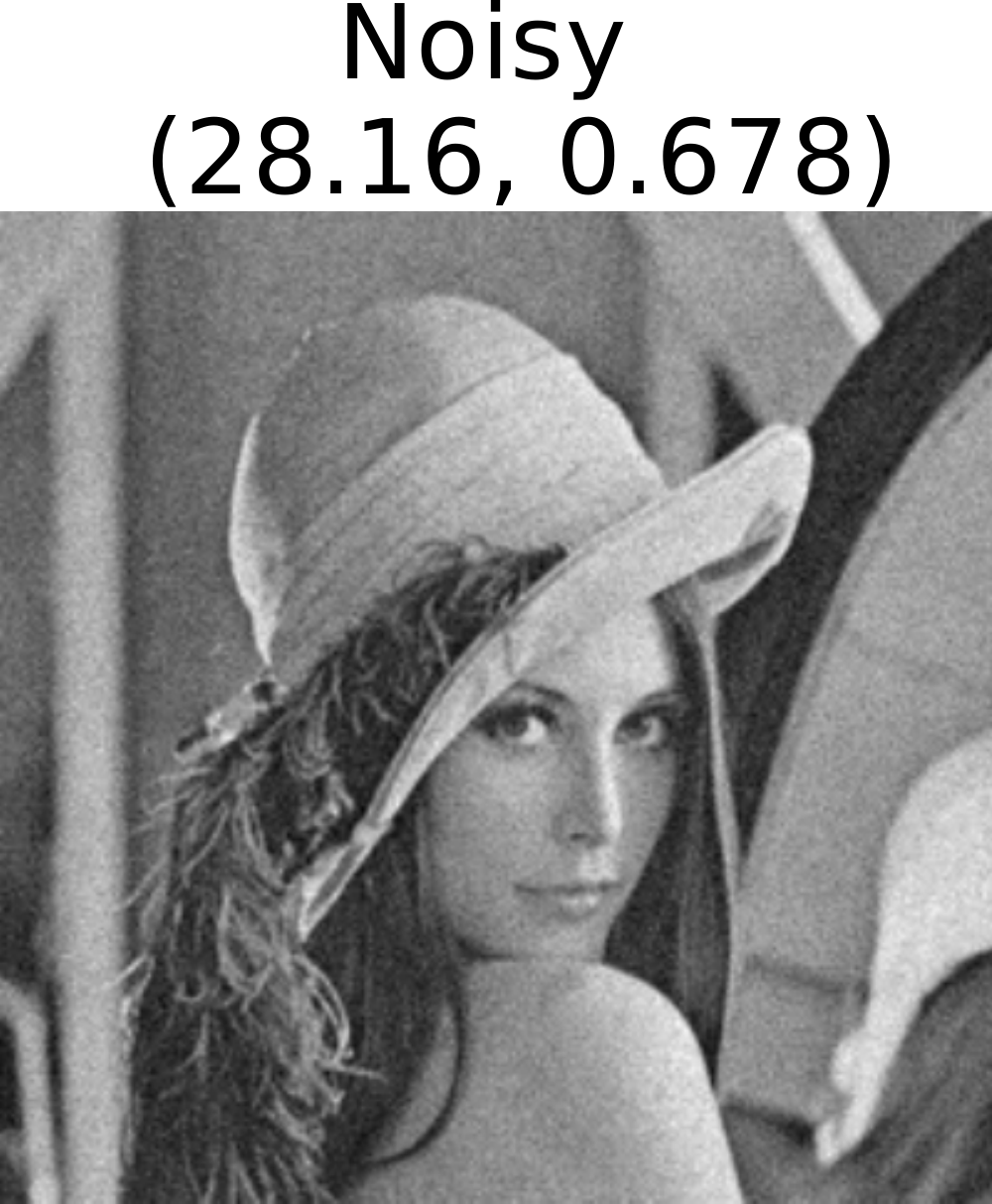}
	\includegraphics[width=0.16\textwidth]{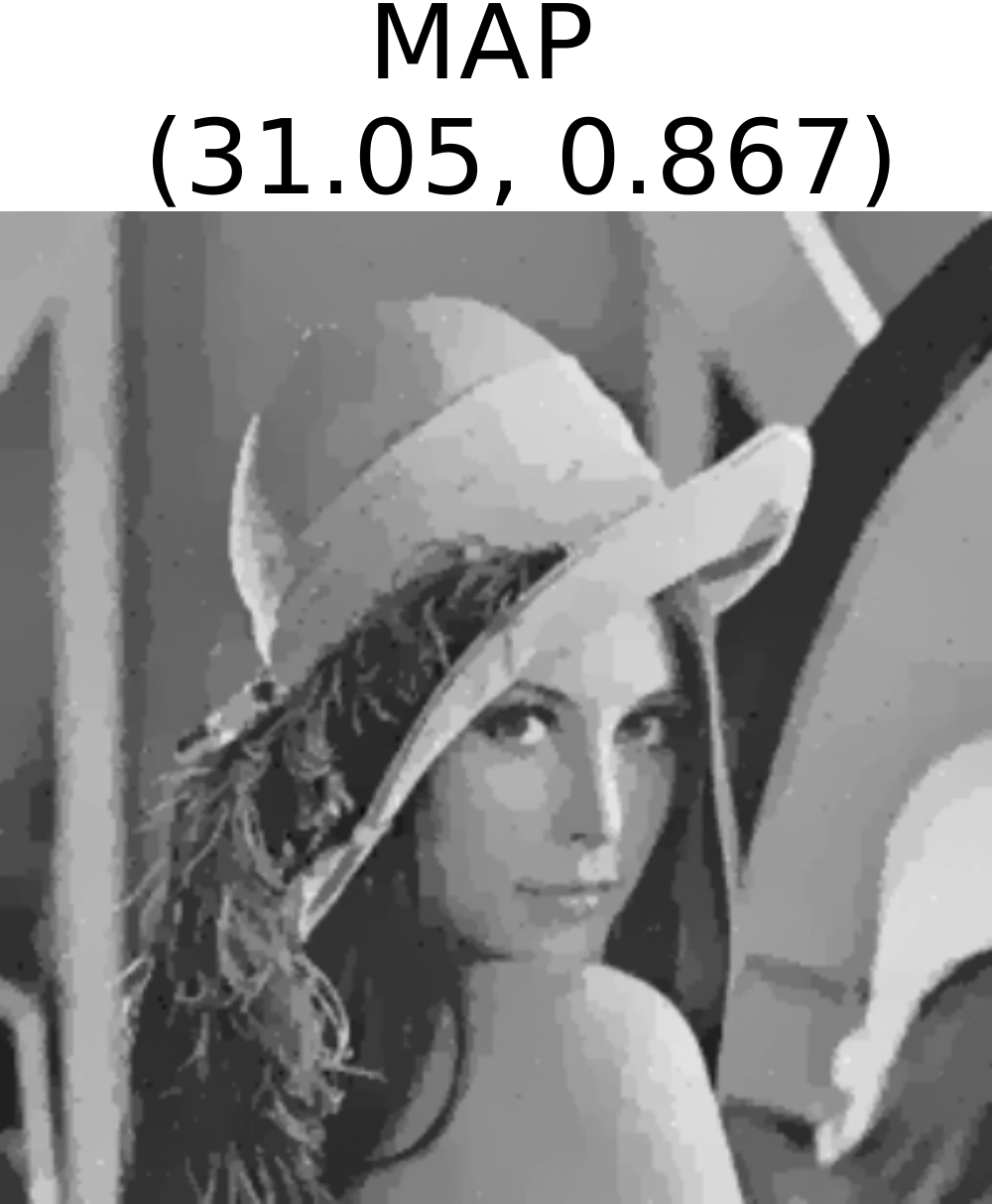}
	\includegraphics[width=0.16\textwidth]{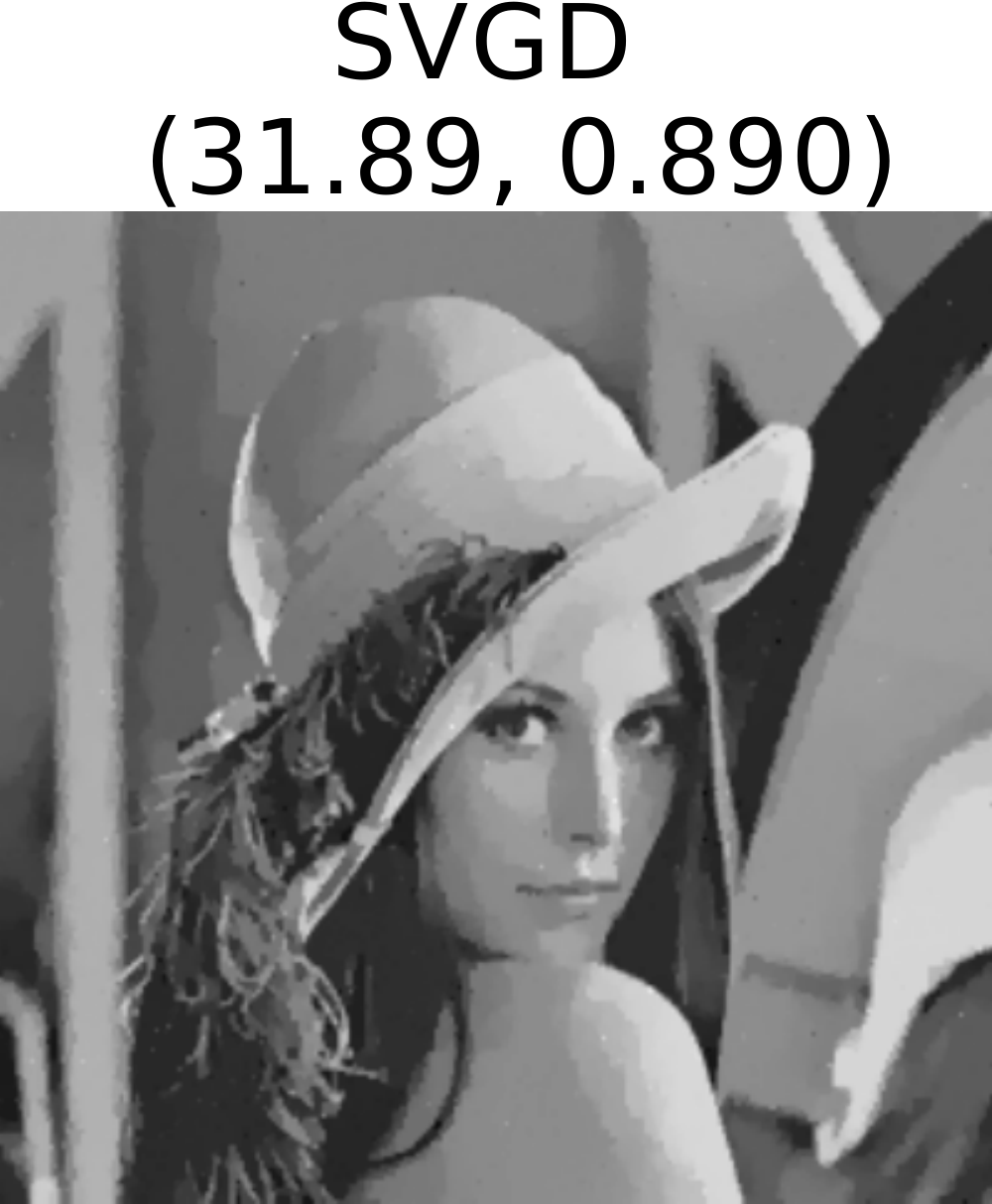}
	\includegraphics[width=0.16\textwidth]{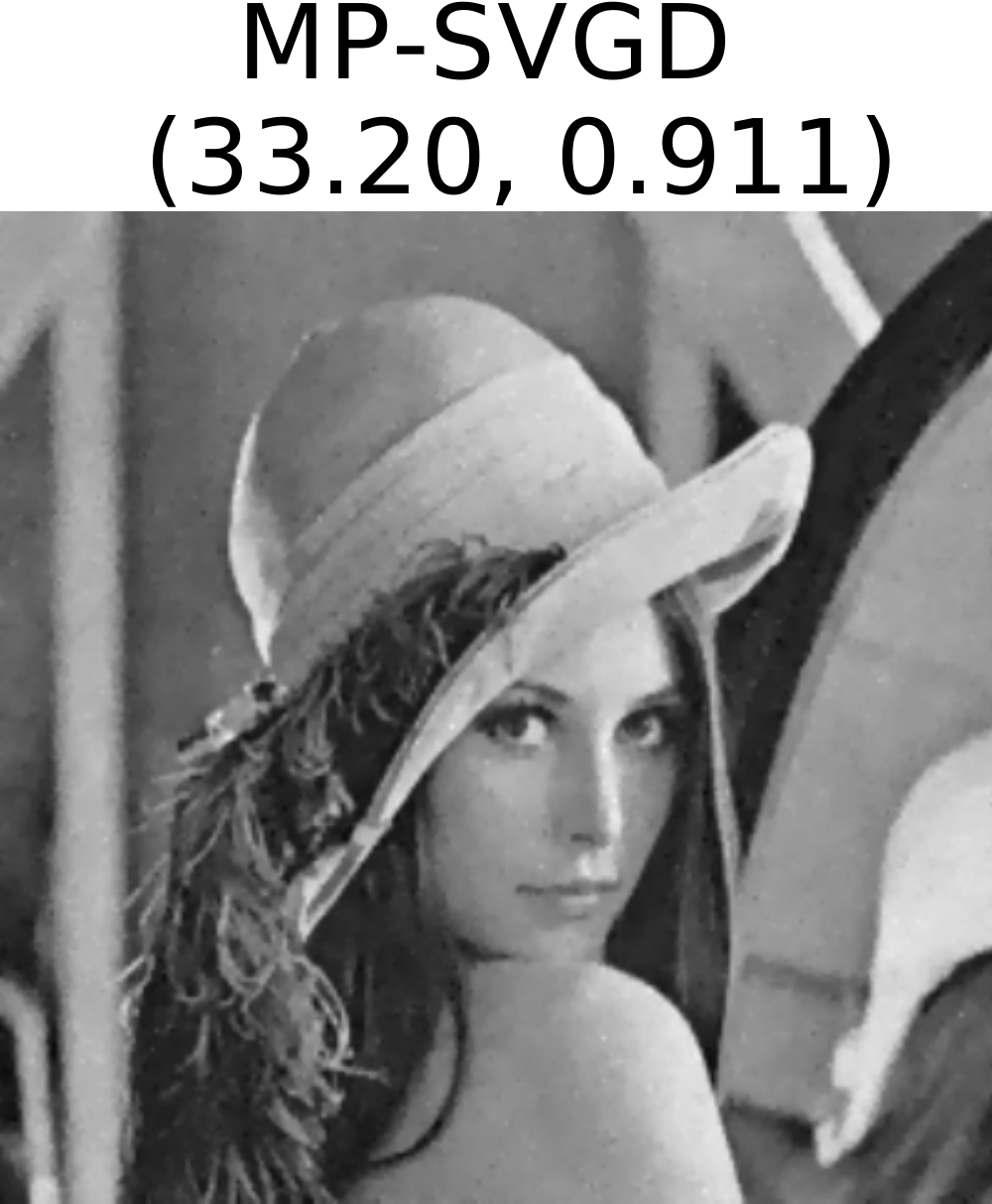}
	\includegraphics[width=0.16\textwidth]{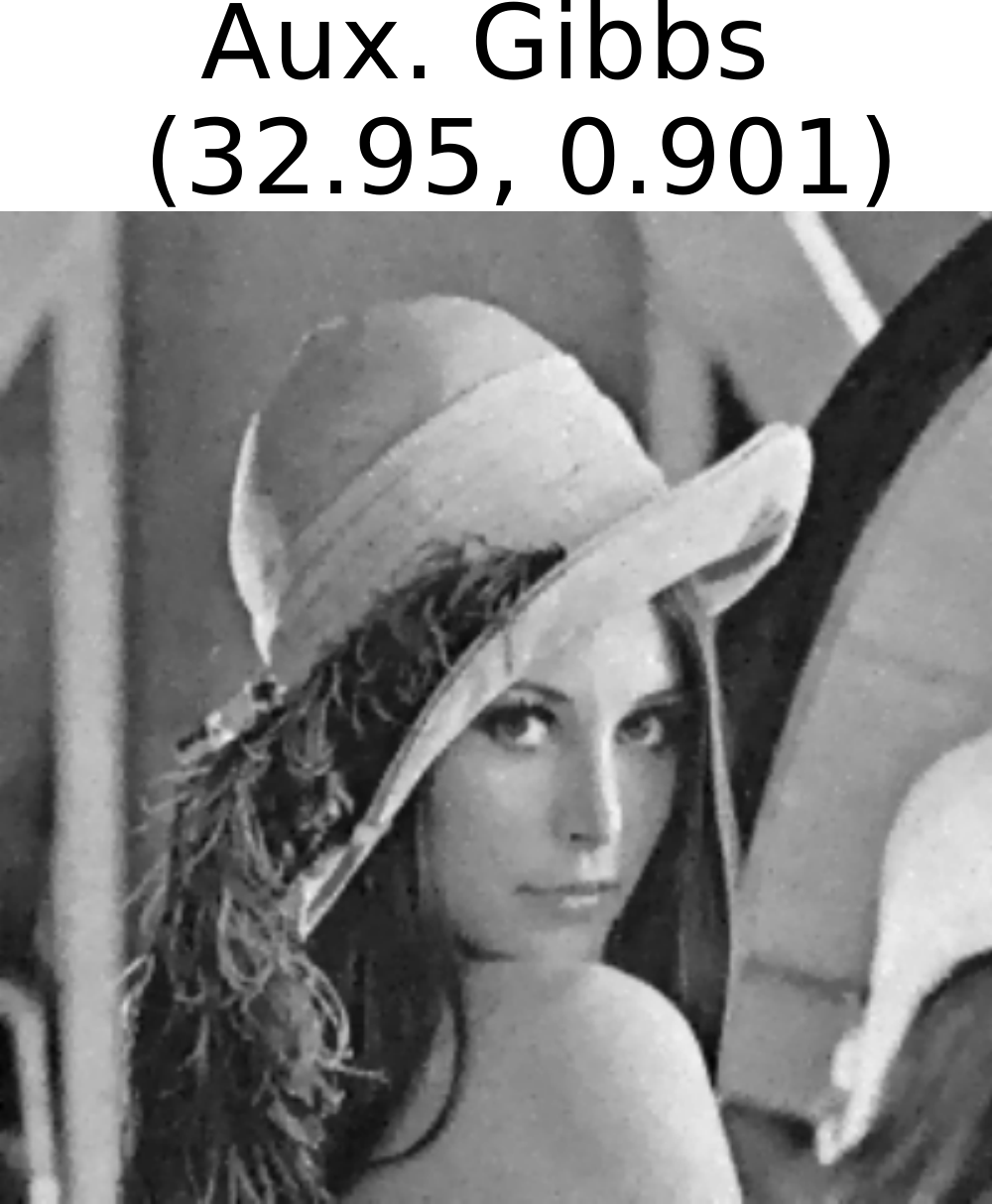}	
	\\
	\includegraphics[width=0.16\textwidth]{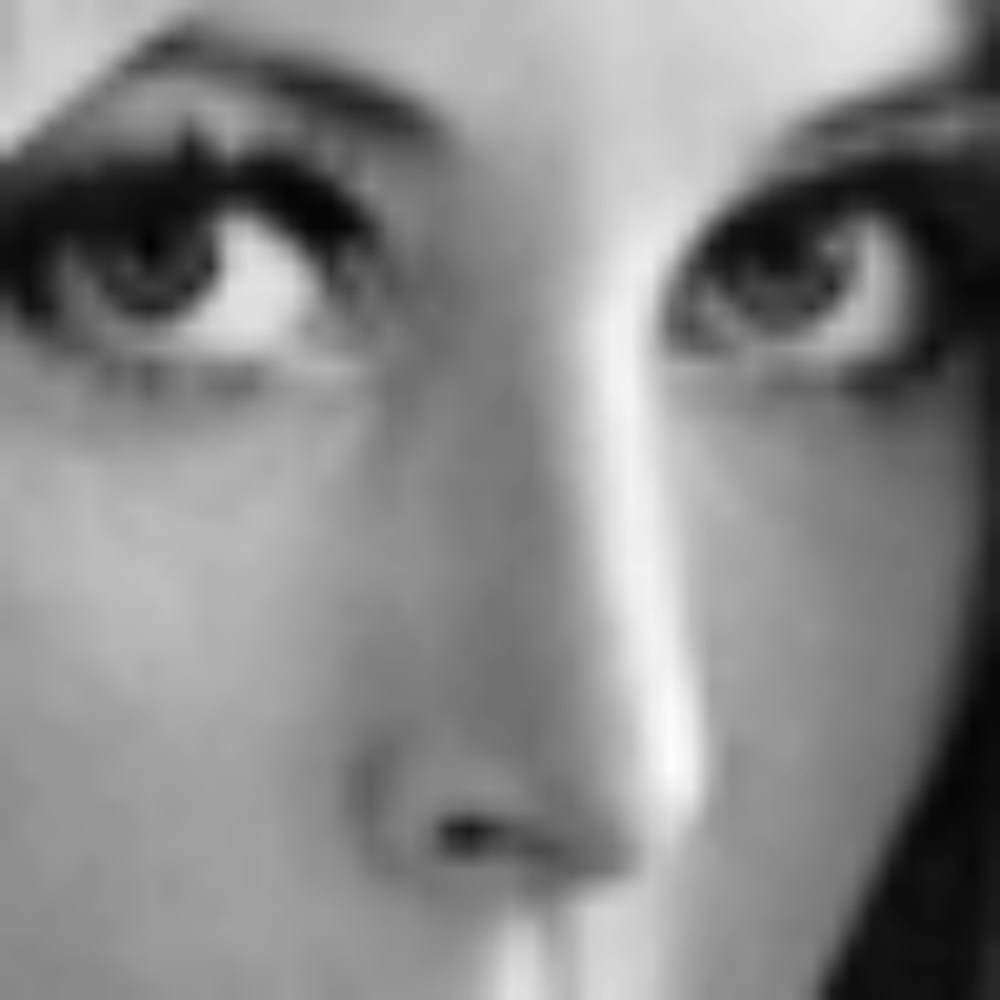}
	\includegraphics[width=0.16\textwidth]{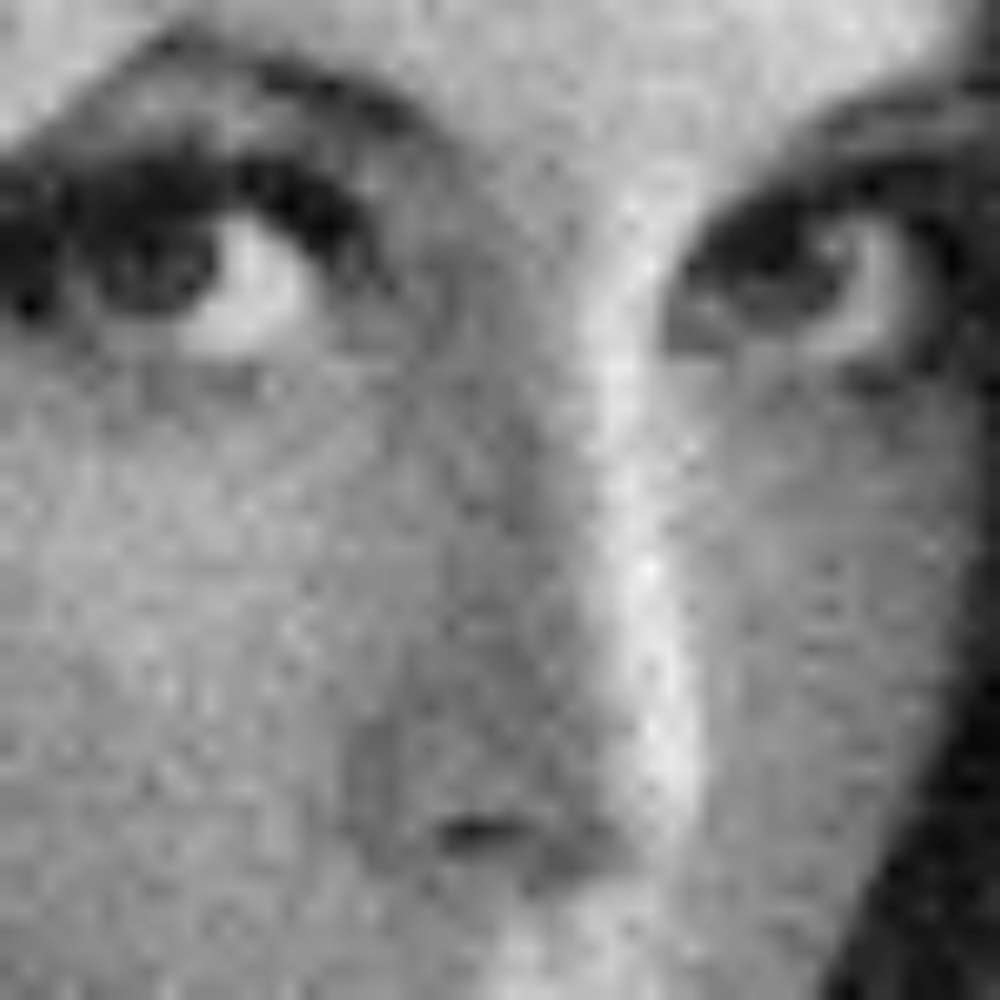}
	\includegraphics[width=0.16\textwidth]{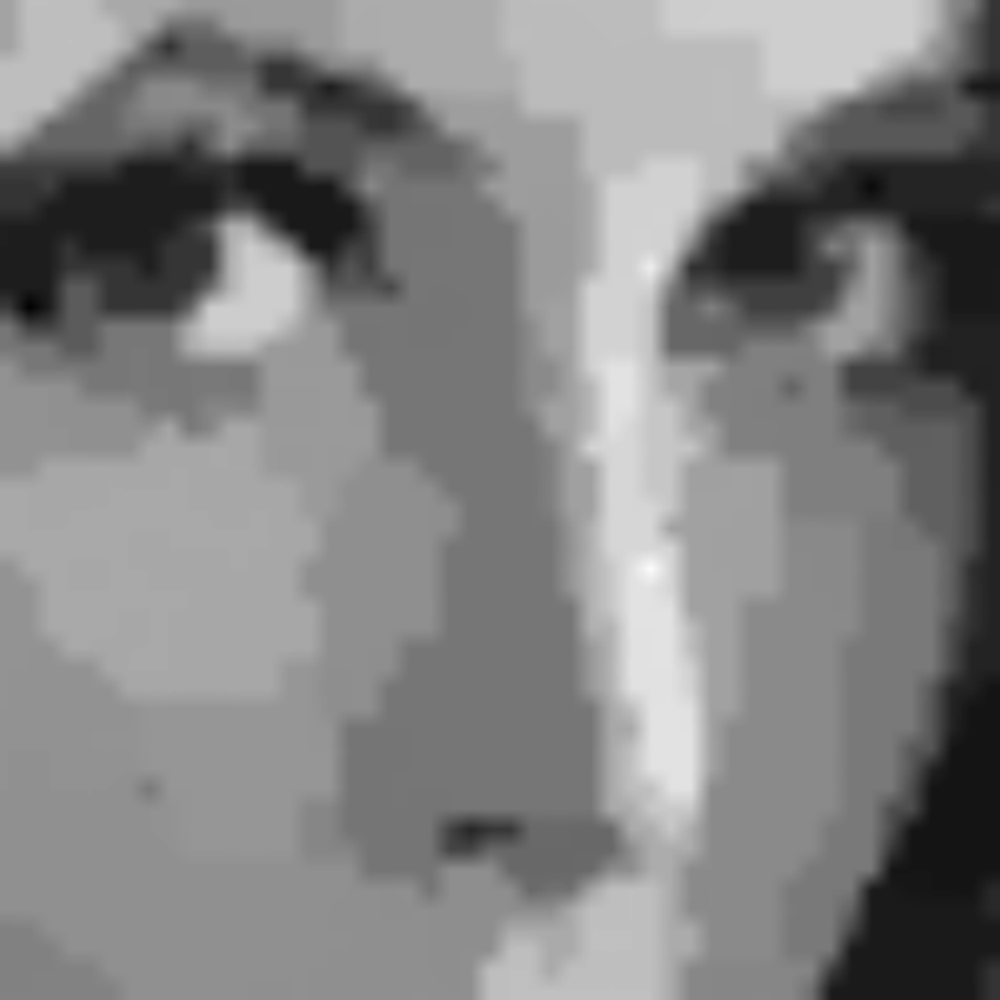}
	\includegraphics[width=0.16\textwidth]{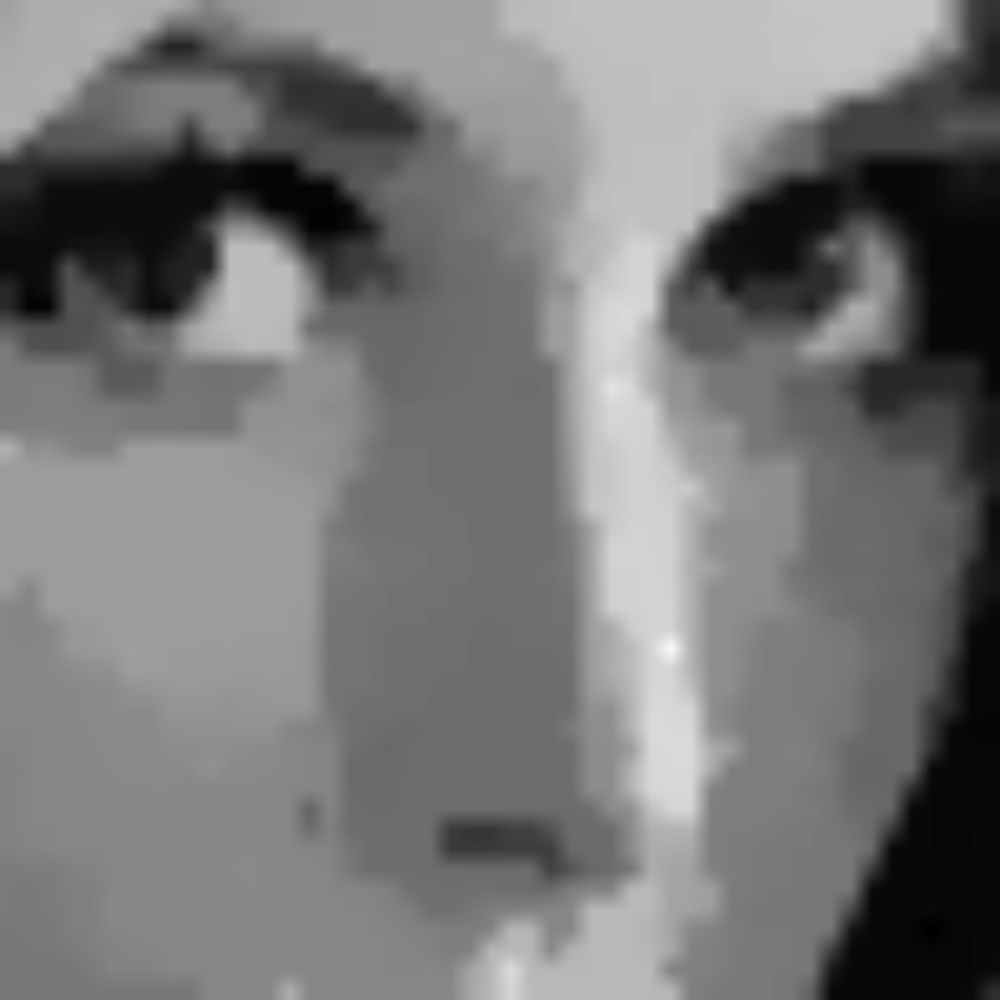}
	\includegraphics[width=0.16\textwidth]{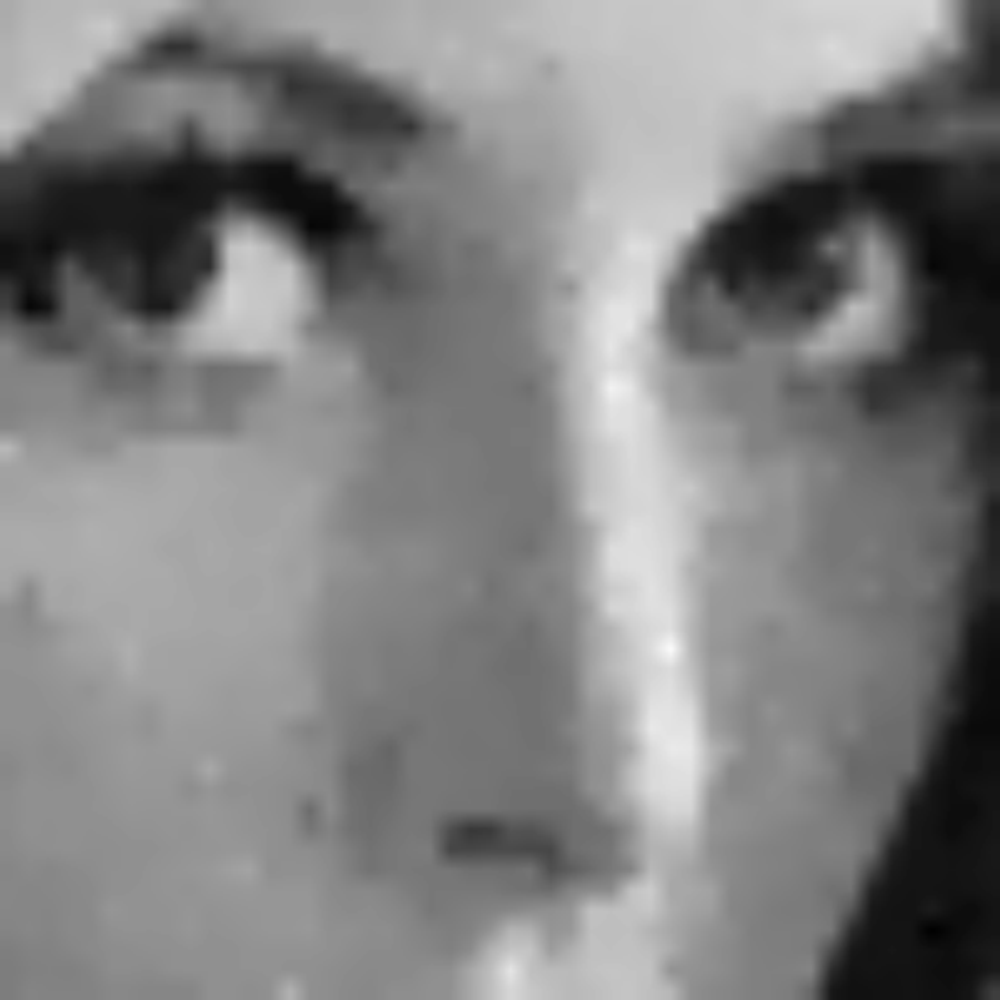}
	\includegraphics[width=0.16\textwidth]{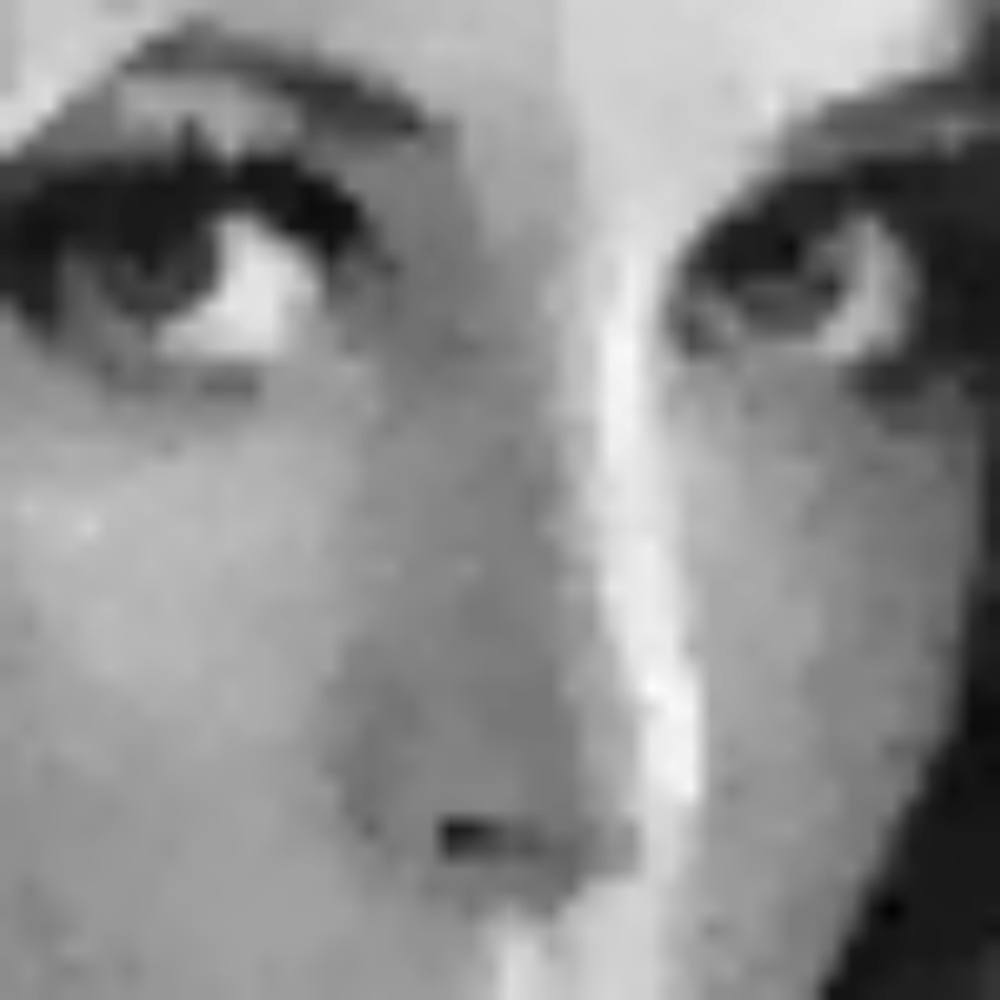}				
	\caption{Denoising results for {\it Lena} using 50 particles, $256 \times 256$ pixels, $\sigma_n = 10$. The number in bracket is PSNR and SSIM. Upper Row: The full size image; Bottom Row: The $50 \times 50$ patches.}
	\label{Fig:lena}
\end{figure*}
Our last experiment is designed for verifying the power of MP-SVGD in real-world application.
Following \cite{Schmidt2010generative}, we formulate image denoising via finding the posterior mean\footnote{It is called the Bayesian minimum mean squared error estimate (MMSE) in the original paper.} of $p(\xv|\yv) \propto p(\yv|\xv) p(\xv)$, where the likelihood $p(\yv|\xv) = \mathcal{N}(\yv|\xv,\sigma_n^2\Iv)$ denotes that the observed image $\yv = \xv + \nv$ for some unknown natural image $\xv$ corrupted by Gaussian noise $\nv$ with the noise level $\sigma_n$.
The prior $p(\xv)$ encodes the statistics of natural images, which is a Fields-of-Experts (FOE) \cite{Roth2009FOE} MRF:
\begin{equation}
p(\xv) \propto \exp(-\frac{\epsilon \|\xv\|_2^2}{2}) \prod_{F \in \mathcal{F}}\prod_{i=1}^N \phi(\Jv_i^{\mathrm{T}}\xv_F;\alphav_i),
\end{equation}
where 
$\{\Jv_i\}_{i=1}^N$ is a bank of linear filters and the expert function $\phi(\Jv_i^{\mathrm{T}}\xv_F;\alphav_i) = \sum_{j=1}^J \alpha_{ij} \mathcal{N}(\Jv_i^{\mathrm{T}}\xv_F|0,\sigma_i^2/s_j)$ is the Gaussian scale mixtures \cite{Woodford2009global}. 
We focus on the pairwise MRF where $\mathcal{F}$ indexes all the edge factors, $\Jv_i = [1,-1]^\top$, $N = 1$ and $J = 15$.
All the parameters (i.e., $\epsilon$, $J_i$, $\sigma_i$ and $s_j$) are pre-learned and details can be found in \cite{Schmidt2010generative}.

We compare SVGD and MP-SVGD\footnote{As MP-SVGD-m is shown to be better than MP-SVGD-s, we only use MP-SVGD-m here, and without notification, MP-SVGD stands for MP-SVGD-m.} with Gibbs sampling with auxiliary variables (Aux. Gibbs), the state-of-the-art method reported in the original paper. 
The recovered image is obtained by averaging all the particles and its quality is evaluated using the peak signal-to-noise ratio (PSNR) and structural similarity index (SSIM) \cite{Wang2004SSIM}. Higher PSNR/SSIM generally means better image quality. 

\begin{figure}[!htb]
	\centering
    \includegraphics[width=0.5\textwidth]{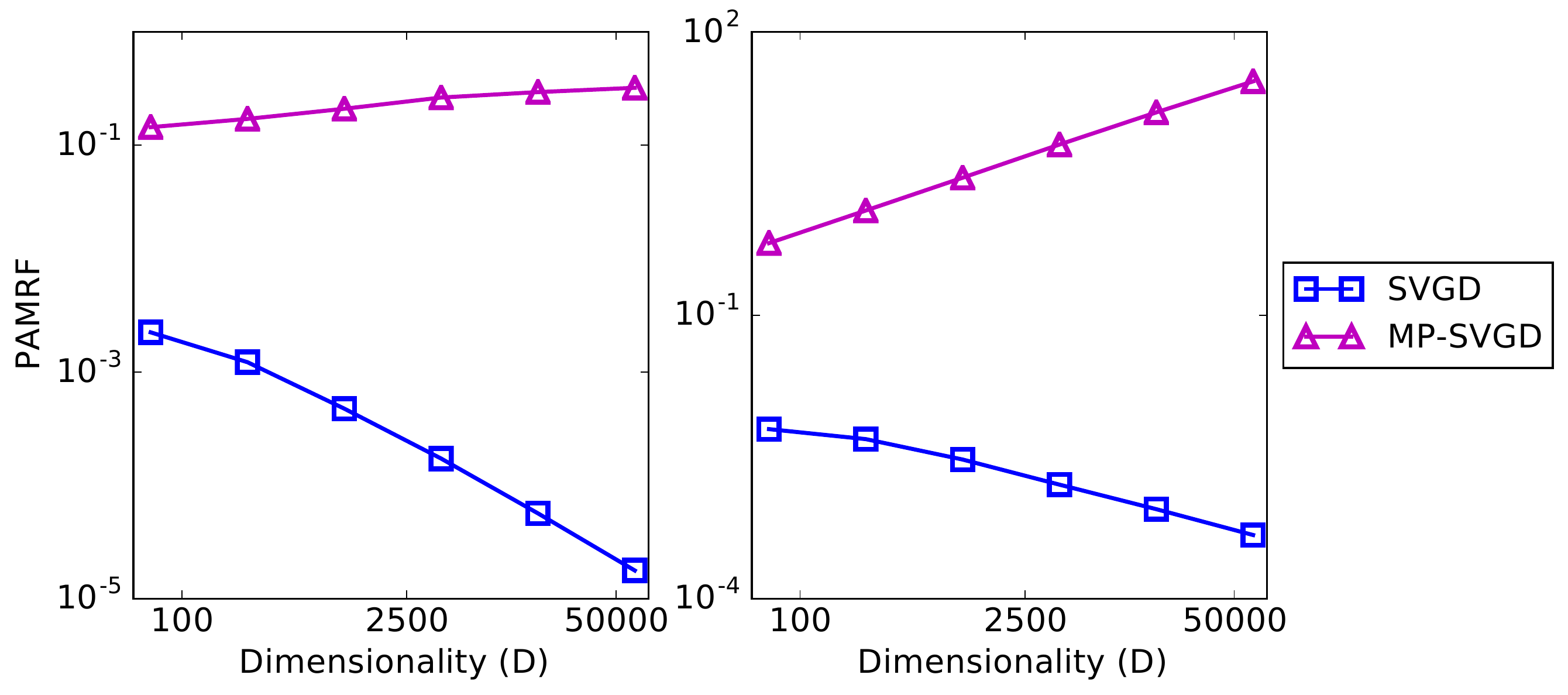}
    \caption{PAMRF $\frac{1}{M} \sum_{i=1}^M \|\Rv(\xv^{(i)};\hat{q}_M)\|_r$ for converged $\{\xv^{(i)}\}_{i=1}^M$ with $r = \infty$ (left) and $r = 2$ (right) over rescaled {\it Lena} ranged from $8 \times 8$ to $256 \times 256$. $M=50$ particles are used.}
    \label{Fig:repullena}
\end{figure}

Fig.~\ref{Fig:lena} shows example results for {\it Lena}, a benchmark image for comparing denoising methods.
Despite the difference in PSNR and SSIM, the image recovered by SVGD lack texture details (especially for the region near the nose of {\it Lena} shown in the $50\times 50$ patches), which resembles the image reconstructed by MAP. 

Table \ref{Tab:10} compares these method quantitatively. As expected, MP-SVGD achieves the best result given the same number of particles. We also find that Aux. Gibbs requires about $200$ particles for $\sigma_n = 10$ and $400$ to $800$ particles for $\sigma_n =20$, to achieve a similar performance of MP-SVGD.

From Fig.~\ref{Fig:repullena} we observe a negative/non-negative correlation between the repulsive force and the dimensionality for SVGD/MP-SVGD, respectively, which verifies our analysis again.

\begin{table}
\center
\caption{
Denoising results for 10 test images \protect\cite{Lan2006efficient} from BSD dataset \protect\cite{Martin2001BSDS}. The first two rows are cited from \protect\cite{Schmidt2010generative} while the other rows are based on our implementation. $M$ denotes the number of particles.}
\scalebox{0.8}{
\begin{tabular}{ c  c  c  c  c }
\hline
\multirow{2}{*}{Inference}  & \multicolumn{2}{c}{avg. PSNR}  & \multicolumn{2}{c}{avg. SSIM} \\
  &$\sigma_n = 10$ & $\sigma_n = 20$ & $\sigma_n = 10$ & $\sigma_n = 20$ \\
\hline
MAP  & 30.27 & 26.48 & 0.855 & 0.720\\
Aux. Gibbs  & 32.09 & {\bf 28.32} & 0.904 & 0.808\\
\hline
Aux. Gibbs ($M = 50$) & 31.87 & 28.05 & 0.898 & 0.795\\
Aux. Gibbs ($M = 100$) & 31.98 & 28.17 & 0.901 & 0.801\\
SVGD ($M = 50$)& 31.58 & 27.86 & 0.894 & 0.766\\
SVGD ($M = 100$)& 31.65 & 27.90 & 0.896 & 0.767\\
MP-SVGD ($M = 50$) & 32.09 & 28.21 & 0.905 & 0.808\\
MP-SVGD ($M = 100$) & {\bf 32.12} & 28.27 & {\bf 0.906} & {\bf 0.809}\\
\hline
\end{tabular}
}
\label{Tab:10} 
\end{table}

\section{Conclusions and Future Work}
In this paper, we analyze the particle degeneracy phenomenon of SVGD in high dimensions and attribute it to the negative correlation between the repulsive force and dimensionality. 
We also propose {\it message passing SVGD} (MP-SVGD), which converts the original problem into several local inference problem with lower dimensions, to solve this problem. Experiments on both synthetic and real-world applications show the effectiveness of MP-SVGD.

For future work, we'd like to settle the analysis of the repulsive force and its impact on SVGD dynamics completely and formally. 
We also want to apply MP-SVGD to more complex real-world applications like pose estimation \cite{Pacheco2014preserving}. 
Besides, investigating robust kernel with non-decreasing repulsive force is also an interesting direction.

\section*{Acknowledgements}
We thank anonymous reviewers for their insightful comments and suggestions. 
This work was supported by NSFC Projects (Nos. 61620106010, 61621136008, 61332007), Beijing NSF Project (No. L172037), Tiangong Institute for Intelligent Computing, NVIDIA NVAIL Program, Siemens and Intel.

\bibliography{ref}
\bibliographystyle{icml2018}

\end{document}


\twocolumn[
\icmltitle{Appendix for Message Passing Stein Variational Gradient Descent}



\icmlsetsymbol{equal}{*}

\icmlcorrespondingauthor{}{}

\icmlkeywords{Machine Learning, ICML}

\vskip 0.3in
]




\appendix

\section{Detailed Derivation and Proof for Section 3}

\subsection{Derivation of Eq. (4)}
By using the change of variable theorem, we have
$$
\begin{aligned}
&~~~~ \nabla_{\epsilon} \mathbb{E}_{\zv \sim q_{[\Tv]}}[\log p(\zv)] |_{\epsilon = 0}\\
& = \nabla_{\epsilon} \mathbb{E}_{\xv \sim q}[\log p(\xv + \epsilon \phiv(\xv))] |_{\epsilon = 0}\\
& = \mathbb{E}_{\xv \sim q}[\nabla_{\epsilon} \log p(\xv + \epsilon \phiv(\xv)) |_{\epsilon = 0}] \\
& = \mathbb{E}_{\xv \sim q}[\nabla_{\xv} \log p(\xv)^\top \phiv(\xv) ] \\
& = \mathbb{E}_{\xv \sim q}\left[\sum_{d=1}^D \nabla_{x_d} \log p(\xv) \phi_d(\xv) \right] \\
& = \sum_{d=1}^D \mathbb{E}_{\xv \sim q}\left[ \nabla_{x_d} \log p(\xv) \langle k(\xv,\cdot), \phi_d(\cdot) \rangle_{\mathcal{H}_0} \right] \\
& = \sum_{d=1}^D  \langle \mathbb{E}_{\xv \sim q}\left[ \nabla_{x_d} \log p(\xv) k(\xv,\cdot) \right], \phi_d(\cdot) \rangle_{\mathcal{H}_0}. \\
\end{aligned}
$$
The maximum is attained when $\phiv = \phiv^* / \|\phiv^*\|_{\mathcal{H}^D}$ with $\phiv^*(\cdot) = \mathbb{E}_{\xv \sim q}[k(\xv,\cdot) \nabla_{\xv} \log p(\xv)]$, i.e., the kernel smoothed gradient $G(\xv;p,q)$. This relationship holds for both $q$ and the empirical distribution $\hat{q}_M$.


When converged, $\Gv(\xv;p,q) \equiv \zerov$, which corresponds to
$$
\begin{aligned}
& \int_{\mathcal{X}} k(\xv,\yv) \gv(\yv) d \yv = \zerov,~\forall \xv \in \mathcal{X} \\
\Longrightarrow & \int_{\mathcal{X}} k(\xv,\yv) g_d(\yv) d \yv = 0,~\forall \xv \in \mathcal{X},~d \in \{1,...,D\} \\
\Longrightarrow & \int_{\mathcal{X}} k(\xv,\yv) g_d(\xv) g_d(\yv) d \yv = 0,~\forall \xv \in \mathcal{X},~d \in \{1,...,D\} \\
\Longrightarrow & \int_{\mathcal{X} \times \mathcal{X}} k(\xv,\yv) g_d(\xv) g_d(\yv) d \xv d \yv = 0,~\forall d \in \{1,...,D\} \\
\end{aligned}
$$
where $\gv(\yv) = q(\yv) \nabla_{\yv} \log p(\yv)$ and $g_d(\yv) = q(\yv) \nabla_{y_d} \log p(\yv)$. Given $k(\xv,\yv)$ is strictly positive definite, i.e., $\int_{\mathcal{X} \times \mathcal{X}} k(\xv,\yv) f(\xv) f(\yv) d \xv d \yv = 0$ if and only if $f(\yv) = 0$, $\forall \yv \in \mathcal{X}$, we have $g_d(\yv) \equiv 0$, $\forall d$, which corresponds to
$$
\gv(\yv) = q(\yv) \nabla_{\yv} \log p(\yv) = 0,~\forall \yv \in \mathcal{X}.
$$
In other words, for $\yv$ such that $q(\yv) \neq 0$, $\nabla_{\yv} \log q(\yv) = 0$, which reflects that $q$ collapses to the modes of $p$.

\subsection{Derivation of Eq. (5)}\label{repulmag}
{\bf RBF Kernel} Notice that 
$$
\begin{aligned}
\|\Rv(\xv; q)\|_\infty & \leq \mathbb{E}_{y \sim q} \left[ \exp(-\frac{\|\xv-\yv\|_2^2}{2h}) \frac{\|\xv-\yv\|_\infty}{h} \right], \\
\end{aligned}
$$
where the inequality holds according to Jensen's inequality. For notation simplicity, let $f(h,\xv,\yv) = \exp(-\frac{\|\xv-\yv\|_2^2}{2h}) \frac{\|\xv-\yv\|_\infty}{h}$, we have
$$
\begin{aligned}
\|\Rv(\xv; q)\|_\infty & \leq \mathbb{E}_{y \sim q} \left[ f(h,\xv,\yv) \right] \\
& \leq \max_h \mathbb{E}_{y \sim q} \left[ f(h,\xv,\yv) \right] \\
& \leq \mathbb{E}_{y \sim q} \left[ \max_h f(h,\xv,\yv) \right]. \\
\end{aligned}
$$
By taking gradient of $f(h,\xv,\yv)$ over $h$ we can show that when $h = \|\xv-\yv\|_2^2/2$, $f(h,\xv,\yv)$ attains its maximum, which is $f_{\max}(\xv,\yv)= \max_h f(h,\xv,\yv) = 2e^{-1}\frac{\|\xv-\yv\|_\infty}{\|\xv-\yv\|_2^2}$. And we have
$$
\begin{aligned}
\|\Rv(\xv; q)\|_\infty & \leq \mathbb{E}_{y \sim q} \left[ \frac{2}{e}\cdot \frac{\|\xv-\yv\|_\infty}{\|\xv-\yv\|_2^2} \right]. \\
\end{aligned}
$$
In fact, we can bound $\|\Rv(\xv;q)\|_r$ with any $r \geq 1$ by using the norm inequality $\|\zv\|_r \leq D^{1/r} \|\zv\|_\infty$, i.e.,
$$
\|\Rv(\xv, q)\|_r \leq \mathbb{E}_{y \sim q} \left[ \frac{2D^{1/r}}{e}\cdot \frac{\|\xv-\yv\|_\infty}{\|\xv-\yv\|_2^2} \right]. 
$$

{\bf IMQ Kernel} For the IMQ kernel, we have
$$
\begin{aligned}
\|\Rv(\xv; q)\|_\infty & \leq \mathbb{E}_{y \sim q} \left[ \frac{1}{2\left(1+\frac{\|\xv-\yv\|_2^2}{2h}\right)^{3/2}} \frac{\|\xv-\yv\|_\infty}{h} \right]. \\
\end{aligned}
$$
Let $f(h,\xv,\yv) = \frac{1}{2\left(1+\frac{\|\xv-\yv\|_2^2}{2h}\right)^{3/2}} \frac{\|\xv-\yv\|_\infty}{h}$ and take the maximum over $h$, we have
$h = \frac{\|\xv-\yv\|_2^2}{4}$, and corresponding 
$$
f_{\max}(h,\xv,\yv) = \frac{2}{3^{3/2}}\frac{\|\xv-\yv\|_\infty}{\|\xv-\yv\|_2^2},
$$
where the only difference compared to the RBF kernel is the constant $\frac{2}{3^{3/2}}$.

\subsection{Derivation of Proposition 1}
Here we derive the kernel smoothed gradient $\Gv(\xv;p,q)$ the repulsive force $\Rv(\xv;q)$ when $q(\yv) = \mathcal{N}(\yv | \muv, \Sigmav)$ is a Gaussian distribution. This example will be useful for understanding the relationship between SVGD and dimensionality, and will be helpful for understanding the convergence condition when $p$ is also a Gaussian.

Since $q(\yv)$ is Gaussian and $k(\xv,\yv)$ is the RBF kernel, $q(\yv)k(\xv,\yv)$ can be regarded as a rescaled Gaussian distribution over $\yv$, i.e., $q(\yv)k(\xv,\yv) = $
$$
\begin{aligned}
& ~~~ \frac{\sqrt{\det \Sigmav^{-1}}}{\sqrt{2\pi }}\exp\left( -\frac{1}{2}(\yv-\muv)^\top \Sigmav^{-1} (\yv-\muv) - \frac{||\xv-\yv||_2^2}{2h} \right) \\
& = \frac{\sqrt{\det \Sigmav^{-1}}}{\sqrt{2\pi}} \exp\left( - \frac{||\xv||_2^2}{2h} - \frac{1}{2}\muv^\top \Sigmav^{-1} \muv \right)\\
&~~~~~~\cdot \exp\left( -\frac{1}{2}\yv^\top (\Sigmav^{-1} + \frac{1}{h}\Iv)\yv + (\Sigmav^{-1} \muv + \frac{\xv}{h})^\top \yv \right)  \\
& = \frac{\sqrt{\det \Sigmav^{-1}}}{\sqrt{\det(\Sigmav^{-1} + \frac{1}{h}\Iv)}} \mathcal{N}(\yv|\tilde{\muv},\tilde{\Sigmav})\\
&~~~~~~ \cdot \exp\left( \frac{1}{2} \tilde{\muv}^\top \tilde{\Sigmav}^{-1} \tilde{\muv} -\frac{||\xv||_2^2}{2h} -\frac{1}{2} \muv^\top \Sigmav^{-1} \muv  \right) \\
& = \frac{\sqrt{\det \Sigmav^{-1}}}{\sqrt{\det(\Sigmav^{-1} + \frac{1}{h}\Iv)}} \mathcal{N}(\yv|\tilde{\muv},\tilde{\Sigmav})\\
&~~~~~~ \cdot \exp\left( -\frac{1}{2} (\xv - \muv)^\top (\Sigmav + h \Iv)^{-1} (\xv-\muv) \right) \\
& = \frac{\sqrt{\det \Sigmav^{-1}}}{\sqrt{\det(\Sigmav^{-1} + \frac{1}{h}\Iv)}} \exp\left( -\frac{1}{2}d(\xv,\muv)\right) \mathcal{N}(\yv|\tilde{\muv},\tilde{\Sigmav})\\
\end{aligned}
$$
with $\tilde{\muv} = (\Sigmav^{-1} + \frac{1}{h}\Iv)^{-1} (\Sigmav^{-1} \muv + \frac{1}{h}\xv)$, $\tilde{\Sigmav} = (\Sigmav^{-1} + \frac{1}{h}\Iv)^{-1}$ and $d(\xv,\muv) = (\xv - \muv)^\top (\Sigmav + h \Iv)^{-1} (\xv-\muv) $.

Given $q(\yv)k(\xv,\yv)$, the repulsive force can be computed, i.e., 
$$
\begin{aligned}
\Rv(\xv;q) & = \frac{\sqrt{\det \Sigmav^{-1}}}{\sqrt{\det(\Sigmav^{-1} + \frac{1}{h}\Iv)}} \exp\left( -\frac{1}{2}d(\xv,\muv)\right) \cdot \frac{\xv - \tilde{\muv}}{h}  \\
	& = \frac{h^{D/2} \exp\left( -\frac{1}{2}d(\xv,\muv)\right)}{\sqrt{\det(\Sigmav + h\Iv)}} (\Sigmav+h\Iv)^{-1} (\xv - \muv)
\end{aligned}
$$
and thus we have
$$
\|\Rv(\xv;q)\|_2 \leq \frac{h^{D/2}}{\sqrt{\det(\Sigmav + h\Iv)}} \|(\Sigmav+h\Iv)^{-1} (\xv - \muv)\|_2
$$
by using the fact that $\exp\left( -\frac{1}{2}d(\xv,\muv)\right) \leq 1$. Then, Assume the eigenvalue decomposition for $\Sigmav$ is $\Sigmav = \Uv \Lambdav \Uv^\top$ with $\Lambdav = \mathrm{diag} (\lambda_1,...,\lambda_D)$ where $\lambda_1 \geq \cdots \geq \lambda_D \geq C$, we have
$$
\begin{aligned}
\det (\Sigmav + h\Iv) & = \det (\Lambdav + h \Iv) = \prod_{d=1}^D \left(\lambda_d + h \right).\\
\end{aligned}
$$
and $\|(\Sigmav+h\Iv)^{-1} (\xv - \muv)\|_2 \leq \frac{1}{h+\lambda_D}\|\xv-\muv\|_2$.

So, we have
$$
\begin{aligned}
\|\Rv(\xv;q)\|_2 & \leq \frac{(1+\lambda_D / h)^{-1}}{h\sqrt{\prod_{d=1}^D (1 + \lambda_d / h)}} \|\xv - \muv\|_2 \\
& \leq \frac{1}{h(1+\lambda_D/h)^{D/2+1}}\|\xv - \muv\|_2
\end{aligned}
$$
Let $f(h) = h(1+\lambda_D/h)^{D/2+1}$, and it is easy to show that when $h = D\lambda_D / 2$, $f(h)$ attains its minimum, which is $f_{\min} = \frac{D\lambda_D}{2}(1+2/D)^{D/2+1} = (1+D/2)\lambda_D(1+2/D)^{D/2}$. Let $\lambda_{\min}(\Sigmav) = \lambda_D$ denote the smallest eigenvalue, we have
$$
\begin{aligned}
\|\Rv(\xv;q)\|_2 & \leq \frac{1}{(1+D/2)\lambda_{\min}(\Sigmav)(1+2/D)^{D/2}} \|\xv - \muv\|_2. \\
\end{aligned}
$$
By using the norm inequality that $\|\zv\|_\infty \leq \|\zv\|_2 \leq D^{1/2}\|\zv\|_\infty$, we have
$$
\|\Rv(\xv;q)\|_\infty \leq \frac{\sqrt{D}}{(1+D/2)\lambda_{\min}(\Sigmav)(1+2/D)^{D/2}} \|\xv - \muv\|_\infty. 
$$
This can be further simplified by noting that $\lim_{D \to \infty}(1+2/D)^{D/2} = e$, so for large $D$, we have the following inequality
$$
\begin{aligned}
\|\Rv(\xv;q)\|_2 & \leq \frac{1}{(1+D/2)\lambda_D(1+2/D)^{D/2}} \|\xv - \muv\|_2 \\
& \lesssim \frac{1}{(1+D/2)\lambda_{\min}(\Sigmav)e}\|\xv - \muv\|_2 \\
& \lesssim \frac{1}{D\lambda_{\min}(\Sigmav)}\|\xv - \muv\|_2 \\
\end{aligned}
$$
and corresponding
$$
\|\Rv(\xv;q)\|_\infty \lesssim \frac{1}{\sqrt{D}\lambda_{\min}(\Sigmav)}\|\xv - \muv\|_\infty. \\
$$

When $q$ is Gaussian whose smallest eigenvalue of $\Sigmav$ is greater than some constant $C$, corresponding $\Rv(\xv;q)$ decreases to zero vector as $1/D$ in $\|\cdot\|_2$ or as $1/\sqrt{D}$ in $\|\cdot\|_\infty$.


\subsection{Proof of Proposition 2}


The inequality can be decomposed as:
$$
\begin{aligned}
& ~~~~ P\left( \frac{\|\yv - \xv\|_\infty}{\|\yv-\xv\|_2^2} \geq \frac{1}{D^{\alpha}} \right) \\
& = P\left( \frac{\max_d |y_d-x_d|}{\|\yv-\xv\|_2^2} \geq \frac{1}{D^{\alpha}} \right) \\
& = P\left( \frac{\max_d |y_d-x_d|}{\|\yv-\xv\|_2^2} \geq \frac{1}{D^{\alpha}}, \|\yv - \xv\|_2^2 \leq b \right) \\
& ~~~~ + P\left( \frac{\max_d |y_d-x_d|}{\|\yv-\xv\|_2^2} \geq \frac{1}{D^{\alpha}}, \|\yv - \xv\|_2^2 > b \right) \\
& \leq P(\|\yv - \xv\|_2^2 \leq b) + P\left(\max_d |y_d-x_d| \geq \frac{b}{D^{\alpha}} \right) \\
& \leq P(\|\yv - \xv\|_2^2 \leq b) + \sum_{d = 1}^D P\left(|y_d-x_d| \geq \frac{b}{D^{\alpha}} \right) \\
& = P( \sum_{d=1}^D (y_d-x_d)^2 \leq b) + \sum_{d = 1}^D P\left(|y_d-x_d| \geq \frac{b}{D^{\alpha}} \right) \\
\end{aligned}
$$
holds for any $b$, even when $b$ is a function of $y$, i.e. a random variable. \\
We bound the first term by using the Azuma-Hoeffding inequality\cite{Azuma1967weighted}:
\begin{thm}
Suppose $Z_D, D \geq 1$ is a martingale such that $Z_0 = 0$ and $|Z_d-Z_{d-1}| \leq c_d, 1\leq d \leq D$ almost surely for some constants $c_d, 1\leq d \leq D$. Then, for every $t > 0$,
$$
P(Z_D > t) \leq \exp\left( - \frac{t^2}{2\sum_{d=1}^D c_d^2} \right),
$$
and 
$$
P(Z_D < -t) \leq \exp\left( - \frac{t^2}{2\sum_{d=1}^D c_d^2} \right).
$$
\end{thm}
To use it, we construct that $Z_D = \sum_{d=1}^D (y_d-x_d)^2 - \sum_{d=1}^D \mathbb{E}[(y_d-x_d)^2|y_{1:d-1}], \forall D \geq 0$ and $Z_0 = 0$ is a martingale. 

First we notice that $Z_D - Z_{D-1} = (y_D-x_D)^2 - \mathbb{E}[(y_D-x_D)^2|y_{1:D-1}]$, which satisfies $\mathbb{E}[Z_D|Z_{1:D-1}] - Z_{D-1} = \mathbb{E}[(y_D-x_D)^2|y_{1:D-1}] - \mathbb{E}[(y_D-x_D)^2|y_{1:D-1}] = 0$.

And then, since we assume $q$ is with bounded support, we have
\begin{align*}
& \mathbb{E}[|Z_D|] = \mathbb{E}_{p(y_{1:D})}[|Z_D|] \\
= & \mathbb{E}_{p(y_{1:D})} \left[\left|\sum_{d=1}^D \left( (y_d-x_d)^2 - \mathbb{E}_{p(y_d|y_{1:d-1})}[(y_d-x_d)^2] \right) \right| \right] \\
\leq & \sum_{d=1}^D \mathbb{E}_{p(y_{1:d})} \left[ \left| (y_d-x_d)^2 - \mathbb{E}_{p(y_d|y_{1:d-1})}[(y_d-x_d)^2] \right| \right]\\
\leq & \sum_{d=1}^D \mathbb{E}_{p(y_{1:d})} \left[ (y_d-x_d)^2 + \mathbb{E}_{p(y_d|y_{1:d-1})}[(y_d-x_d)^2] \right] \\
= & 2 \sum_{d=1}^D \mathbb{E}_{p(y_{1:d-1})} \left[ \mathbb{E}_{p(y_d|y_{1:d-1})}[(y_d-x_d)^2] \right] \\
= & 2 \sum_{d=1}^D \mathbb{E}_{p(y_d)} \left[ (y_d-x_d)^2 \right] \\
\leq & 8 D C^2 \\
\leq & \infty .
\end{align*}
So we show that $Z_D = \sum_{d=1}^D (y_d-x_d)^2 - \sum_{d=1}^D \mathbb{E}[(y_d-x_d)^2|y_{1:d-1}]$ is a martingale. \\

Now we show that
\begin{align*}
|Z_d - Z_{d-1}| = & |(y_d - x_d)^2 - \mathbb{E}_{p(y_d|y_{1:d-1})}[(y_d-x_d)^2]| \\
\leq & (y_d-x_d)^2 + \mathbb{E}_{p(y_d|y_{1:d-1})}[(y_d-x_d)^2] \\
\leq & 8 C^2.
\end{align*}
So, by choosing $b = \sum_{d=1}^D \mathbb{E}[(y_d-x_d)^2|y_1,...,y_{d-1}] - t$ (notice that $b$ here is indeed a random variable) and using the inequality, we have
$$
\begin{aligned}
&~~~~ P(\sum_{d=1}^D (y_d-x_d)^2 < b ) \\
& = P(\sum_{d=1}^D (y_d-x_d)^2 - \sum_{d=1}^D \mathbb{E}[(y_d-x_d)^2|y_1,...,y_{d-1}]  < -t) \\
& \leq  \exp\left( - \frac{t^2}{128 D C^4 } \right)
\end{aligned}
$$
When $t = D C_0 / 2$, we have
$$
P(\sum_{d=1}^D (y_d-x_d)^2 < b ) \leq  \exp\left( - \beta D \right),
$$
where $\beta = C_0^2 / (256 C^4)$.

Now we bound the second term. Notice that $b =\sum_{d=1}^D \mathbb{E}[(y_d-x_d)^2|y_1,...,y_{d-1}] -  \frac{1}{2}D C_0 \geq \frac{1}{2}\sum_{d=1}^D C_0 = DC_0/2 = b'$ almost surely as the assumption ($b$ is a random variable while $b'$ is a constant), we have
$$
\begin{aligned}
& ~~~~ P\left(|y_d - x_d| \geq \frac{b}{D^{\alpha}} \right) \leq P\left(|y_d - x_d| \geq \frac{b'}{D^{\alpha}} \right)\\
& = P\left(y_d \geq x_d + \frac{b'}{D^{\alpha}} \right) + P\left(y_d \leq x_d - \frac{b'}{D^{\alpha}} \right) \\
& = P\left(\exp(t_1 y_d) \geq \exp(t_1 (\frac{ b'}{D^{\alpha}}+ x_d)) \right) \\
&~~~~  + P\left(\exp(- t_2 y_d) \geq \exp( t_2 (\frac{b'}{D^{\alpha}} - x_d)) \right) \\
& \leq \frac{\mathbb{E}[\exp(t_1 y_d)]}{\exp\left(t_1 (\frac{ b'}{D^{\alpha}}+ x_d)\right)} + \frac{\mathbb{E}[\exp(-t_2 y_d)]}{\exp\left( t_2 (\frac{b'}{D^{\alpha}} - x_d)\right)} \\
& \leq \frac{\exp(t_1 \mu_d + \frac{1}{2} t_1^2 C^2) }{\exp\left(t_1 (\frac{ b'}{D^{\alpha}}+ x_d)\right)} + \frac{\exp(-t_2 \mu_d + \frac{1}{2} t_2^2 C^2)}{\exp\left( t_2 (\frac{b'}{D^{\alpha}} - x_d)\right)} \\
& = \frac{\exp(t_1 \mu_d + \frac{1}{2} t_1^2 C^2) }{\exp\left(t_1 (\frac{1}{2}D^{1-\alpha}C_0^2+ x_d)\right)} + \frac{\exp(-t_2 \mu_d + \frac{1}{2}t_2^2 C^2)}{\exp\left( t_2 (\frac{1}{2}D^{1-\alpha}C_0^2 - x_d)\right)} \\
& = \frac{\exp(\frac{1}{2}t_1^2 C^2 + t_1 (\mu_d - x_d))}{\exp\left( \frac{1}{2}t_1 D^{1- \alpha} C_0^2 \right)} + \frac{\exp(\frac{1}{2}t_2^2 C^2 - t_2 (\mu_d - x_d) )}{\exp \left( \frac{1}{2} t_2 D^{1- \alpha} C_0^2 \right)} \\
\end{aligned}
$$
holds for any $t_1, t_2 > 0$, where the first inequality holds because of the Markov inequality, and the second inequality holds according to the definition of the sub-Gaussian distribution. According to Hoeffding's Lemma, any bounded random variables $|Z| \leq C$ corresponds to $C$-sub-Gaussian distribution, which satisfies $\mathbb{E}[e^{t(Z-\mu)}] \leq \exp(t^2 C^2 / 2)$ for any $t \in \mathbb{R}$. 

Now let $t = t_1 = t_2$ and $\mu_d' = \mu_d-x_d$, we have
$$
\begin{aligned}
&~~~~ P\left(|y_d - x_d| \geq \frac{b}{D^{\alpha}} \right) \\
& \leq \frac{\exp(t^2 C^2 / 2 + t \mu_d')}{\exp\left( t D^{1- \alpha} C_0^2 / 2 \right)} + \frac{\exp(t^2 C^2 / 2 - t \mu_d' )}{\exp \left( 5 t D^{1- \alpha} \sigma^2 \right)} \\ 
& = \exp \left( - t D^{1-\alpha} C_0^2 / 2 + t^2C^2 / 2 \right) \left(e^{t \mu_d'} + e^{-t \mu_d'} \right) \\
& \leq 2\exp \left( -t D^{1-\alpha} C_0^2 / 2 + t^2 C^2 / 2 \right) \cosh (t \mu_d') \\
\end{aligned}
$$
By choosing $t = 2 / C_0^2$, we have
$$
P\left(|y_d-x_d| \geq \frac{b}{D^{\alpha}} \right) \leq L \exp ( - D^{1-\alpha} ) 
$$
where $2\exp \left(2C^2 / C_0^2 \right) \cosh (2\mu_{z_d} / C_0^2) \leq 2\exp \left(2C^2 / C_0^2 \right) \cosh (2\|\muv - \xv\|_\infty / C_0^2) \leq 2\exp \left(2C^2 / C_0^2 \right) \cosh (4 C/ C_0^2) = L$. \\
Combining these two terms, we have
$$
P\left( \frac{\|\yv - \xv\|_\infty}{\|\yv-\xv\|_2^2} \geq \frac{1}{D^{\alpha}} \right) \leq e^{-\beta D} + L D e^{-D^{1-\alpha}} 
$$
for some $\beta, L \geq 0$.
Now, we'd like to give a clean (but loose) bound by noticing that
$$
e^{-\beta D} + L D e^{-D^{1-\alpha}} \leq (L+1)D \max \{ e^{-D^{1-\alpha}}, e^{-\beta D}\}.
$$
By using some derivations, we can get another bound
$$
D e^{-D^{1-\alpha}} \leq e^{-(1-1/e)D^{1-\alpha}},
$$
and
$$
D e^{-\beta D} \leq \frac{1}{\beta}e^{-\frac{1}{2}\beta D}.
$$

Let $\delta' \geq (L+1) \max\{ e^{-(1-1/e)D^{1-\alpha}} , \frac{1}{\beta}e^{-\frac{1}{2}\beta D} \}$, we have
$$
D \geq \max\{ \exp(\frac{1}{1-\alpha})\frac{1}{1-1/e}\log \frac{L+1}{\delta'}, \frac{2}{\beta}\log \frac{L+1}{\beta \delta'} \}.
$$
As a result, for any $\delta' \in (0,1)$, there exists $D_0 = \max\{ \exp(\frac{1}{1-\alpha})\frac{1}{1-1/e}\log \frac{L+1}{\delta'}, \frac{2}{\beta}\log \frac{L+1}{\beta \delta'} \}$, such that for any $D > D_0$, we have $\frac{\|\yv - \xv\|_\infty}{\|\yv-\xv\|_2^2} \leq \frac{1}{D^{\alpha}}$ with at least probability $1 - \delta'$. 

Now, we begin to prove our proposition. By using the conclusion in section \ref{repulmag}, we can bound $\|\Rv(\xv;\hat{q}_M)\|_\infty$ as
$$
\| \Rv(\xv;\hat{q}_M) \|_\infty \leq \frac{2}{Me} \sum_{i=1}^M \frac{\|\xv-\yv\|_\infty}{\|\xv-\yv\|_2^2}.
$$
According to the union bound, we have 
$$
\begin{aligned}
& ~~~~P\left( \max_i \frac{\|\xv^{(i)} - \xv\|_\infty}{\|\xv^{(i)}-\xv\|_2^2} \geq D^{-\alpha} \right)\\
&  \leq \sum_{i=1}^M P\left( \frac{\|\xv^{(i)}- \xv\|_\infty}{\|\xv^{(i)}-\xv\|_2^2} \geq D^{-\alpha} \right) \\
&  \leq M P\left( \frac{\|\yv- \xv\|_\infty}{\|\yv-\xv\|_2^2} \geq D^{-\alpha} \right), \yv \sim q \\
\end{aligned}
$$
where the last inequality holds since $\{\xv^{(i)}\}_{i=1}^M$ are samples from $q$. Then, we can directly apply the conclusion with $\delta = M\delta'$. Then, we have, for any $\delta \in (0,1)$, there exists $D_0 = \max\{ \exp(\frac{1}{1-\alpha})\frac{1}{1-1/e}\log \frac{(L+1)M}{\delta}, \frac{2}{\beta}\log \frac{(L+1)M}{\beta \delta} \}$, such that for any $D > D_0$, we have 
$$
\| \Rv(\xv;\hat{q}_M) \|_\infty \leq \frac{2}{eD^\alpha} 
$$
with at least probability $1-\delta$.

\section{Detailed Derivation and Proof for Section 4}

\subsection{Derivation of Sub-KL Divergence}\label{sec:derivsubKL}
Given the condition that 
$$
\mathrm{KL}(q(x_d|\xv_{-d})q(\xv_{-d})\|p(x_d|\xv_{\Gamma_d})q(\xv_{-d}))] = 0,~\forall d,
$$
we have $q(x_d|\xv_{-d})=p(x_d|\xv_{\Gamma_d}) = p(x_d|\xv_{-d}),~\forall d$. When both $p$ and $q$ are differentiable, we have $\nabla_{x_d} \log q(x_d|\xv_{-d}) = \nabla_{x_d} \log p(x_d|\xv_{-d}), \forall d$. In other words, we have $\nabla_{\xv} \log q(\xv) = \nabla_{\xv} \log p(\xv)$, and thus $q(\xv) = e^C p(\xv)$. By using the normalization property of distribution, we have $C=0$ and thus $q(\xv) = p(\xv)$.

\subsection{Derivation of $q_{[\Tv]}(\zv_{\neg d})=q(\zv_{\neg d})$}\label{sec2}
Recall the change of variable theorem, we have 
$$
q_{[\Tv]}(\zv) = q(\Tv^{-1}(\zv))\left| \mathrm{det} (\nabla_\zv \Tv^{-1}) \right|.
$$
Since $\Tv_{\neg d}$ is an identity mapping from $\xv_{\neg d}$ to $\xv_{\neg d}$, $\nabla_{\zv} \Tv^{-1}$ is a block-wise triangular matrix and the determinant $\mathrm{det} (\nabla_\zv \Tv^{-1})$ satisfies
$$
\begin{aligned}
\mathrm{det} (\nabla_\zv \Tv^{-1}) & = \mathrm{det} (\nabla_{\zv_{\neg d}} \Tv_{\neg d}^{-1}) \cdot \mathrm{det} (\nabla_{z_d} T_d^{-1}) \\
& = \mathrm{det} (\nabla_{z_d} T_d^{-1}).
\end{aligned}
$$
As a result, we have
$$
q_{[\Tv]}(\zv) = q(\Tv^{-1}(\zv))\left|\mathrm{det} (\nabla_{z_d} T_d^{-1}) \right|.
$$
So $q_{[T]}(\zv_{\neg d}) =$
$$
\begin{array}{ll}
&~~~~ \int q_{[T]}(\zv) d z_d \\ 
& = q(\zv_{\neg d}) \int q(T_d^{-1}(z_d)|\zv_{\neg d})\left|\mathrm{det} (\nabla_{z_d} T_d^{-1}) \right| d z_d \\
& = q(\zv_{\neg d}) \int q_{[T_d]}(z_d | \zv_{\neg d}) d z_d = q(\zv_{\neg d}).
\end{array}
$$

\subsection{Proof of Proposition 3}
First we prove that
$$
\begin{aligned}
& \nabla_{\epsilon} \mathrm{KL}(q_{[\Tv]}\|p) = \\
& ~~~~\nabla_{\epsilon} \mathrm{KL}\big(q_{[T_d]}(z_d|\zv_{\Gamma_d})q(\zv_{\Gamma_d})\big\|p(z_d|\zv_{\Gamma_d})q(\zv_{\Gamma_d})\big).
\end{aligned}
$$
Given $\zv = \Tv(\xv) = [x_1,...,T_d(x_d),...,x_D]^\top$, as proved in Section \ref{sec2}, we have $q_{[\Tv]}(\zv) = q(\Tv^{-1}(\zv))|\mathrm{det}(\nabla_{z_d} T_d^{-1})|$ and thus
\begin{equation}\label{Eq:KLdecomp}
\begin{aligned}
\mathrm{KL}(q_{[\Tv]}\|p) & = \mathrm{KL}\big(q_{[T_d]}(z_d|\zv_{\neg d})q(\zv_{\neg d})\big\|p(z_d|\zv_{\Gamma_d})q(\zv_{\neg d})\big) \\
&~~~~+\mathrm{KL}\big(q(\zv_{\neg d})\|p(\zv_{\neg d})\big).
\end{aligned}
\end{equation}
When $T_d:x_d \to x_d + \epsilon \phi_d(x_{S_d})$ where $S_d = \{d\} \cup \Gamma_d$, we can further decompose the right handside of Eq. (\ref{Eq:KLdecomp}), i.e., $\mathrm{KL}\big(q_{[T_d]}(z_d|\zv_{\neg d})q(\zv_{\neg d})\big\|p(z_d|\zv_{\Gamma_d})q(\zv_{\neg d}\big) = $
$$
\begin{aligned}
& ~~ \mathrm{KL}\big(q_{[T_d]}(z_d|\zv_{\neg d})q(\zv_{\neg d})\big\|q_{[T_d]}(z_d|\zv_{\Gamma_d})q(\zv_{\neg d})\big) \\
& + \mathrm{KL}\big(q_{[T_d]}(z_d|\zv_{\Gamma_d})q(\zv_{\neg d})\big\|p(z_d|\zv_{\Gamma_d})q(\zv_{\neg d})\big).
\end{aligned}
$$
By using the change of variable, we can find out that
$$
\begin{aligned}
&~~~~\mathrm{KL}\big(q_{[T_d]}(z_d|\zv_{\neg d})q(\zv_{\neg d})\big\|q(z_d|\zv_{\Gamma_d})q(\zv_{\neg d})\big) \\
& = \int q_{[\Tv]}(\zv) \log \frac{q_{T_d}(z_d|\zv_{\neg d})}{q_{T_d]}(z_d|\zv_{\Gamma_d})} d\zv \\
& = \int q(\xv) \log \frac{q(x_d|\xv_{\neg d})/|\mathrm{det}(\nabla_{x_d} T_d)|}{q(x_d|\xv_{\Gamma_d})/|\mathrm{det}(\nabla_{x_d} T_d)|} d\xv \\
& = \int q(\xv) \log \frac{q(x_d|\xv_{\neg d})}{q(x_d|\xv_{\Gamma_d})} d\xv \\
& = \mathrm{KL}\big(q(x_d|\xv_{\neg d})q(\xv_{\neg d})\big\|q(x_d|\xv_{\Gamma_d})q(\xv_{\neg d})\big), \\
\end{aligned}
$$
which is unrelated with $T_d$ (and thus unrelated with $\epsilon$). As a result, we have
$$
\begin{aligned}
& \nabla_{\epsilon} \mathrm{KL}(q_{[\Tv]}\|p) = \\
& ~~~~\nabla_{\epsilon} \mathrm{KL}\big(q_{[T_d]}(z_d|\zv_{\Gamma_d})q(\zv_{\Gamma_d})\big\|p(z_d|\zv_{\Gamma_d})q(\zv_{\Gamma_d})\big).
\end{aligned}
$$
Now we derive the optimal $\phi_d^*$ for $\min_{\|\phi_d\|_{\mathcal{H}_d} \leq 1} \nabla_\epsilon \mathrm{KL}(q_{[\Tv]}\|p) |_{\epsilon = 0}$. Notice that
$$
\begin{aligned}
&~~~~\mathrm{KL}\big(q_{[T_d]}(z_d|\zv_{\Gamma_d})q(\zv_{\neg d})\big\|p(z_d|\zv_{\Gamma_d})q(\zv_{\neg d})\big) \\
& = \int q_{[T_d]}(z_d|\zv_{\Gamma_d})q(\zv_{\Gamma_d})\log \frac{q_{[T_d]}(z_d|\zv_{\Gamma_d})}{p(z_d|\zv_{\Gamma_d})} d\zv \\
& = \mathbb{E}_{q(\zv_{\Gamma_d})}\Big[ \mathrm{KL}\big(q_{[T_d]}(z_d|\zv_{\Gamma_d})\|p(z_d|\zv_{\Gamma_d})\big) \Big].
\end{aligned}
$$
Following the proof of Theorem 3.1 in \cite{Liu2016SVBP}, we have
$$
\begin{aligned}
& ~~~~ \nabla_{\epsilon} \mathrm{KL}\big(q_{[T_d]}(z_d|\zv_{\Gamma_d})\|p(z_d|\zv_{\Gamma_d})\big) |_{\epsilon = 0}\\
& = -\mathbb{E}_{q(y_d|\yv_{\Gamma_d})}\big[ \phi_d(\yv_{S_d})\nabla_{y_d} \log p(y_d|\yv_{\Gamma_d}) + \nabla_{y_d} \phi_d(\yv_{S_d}) \big].
\end{aligned}
$$
Combing the above three equations together, we have $\nabla_\epsilon \mathrm{KL}(q_{[\Tv]}\|p)|_{\epsilon = 0} = $
$$
-\mathbb{E}_{q(y_d|\yv_{\Gamma_d})q(\yv_{\Gamma_d})}\big[ \phi_d(\yv_{S_d})\nabla_{y_d} \log p(y_d|\yv_{\Gamma_d}) + \nabla_{y_d} \phi_d(\yv_{S_d}) \big]
$$
and $\min_{\|\phi_d\|_{\mathcal{H}_d} \leq 1} \nabla_\epsilon \mathrm{KL}(q_{[\Tv]}\|p) |_{\epsilon = 0}$ corresponds to
$$
\max_{\|\phi_d\|_{\mathcal{H}_d} \leq 1} \mathbb{E}_{q(\yv_{S_d})}\big[ \phi_d(\yv_{S_d})\nabla_{y_d} \log p(y_d|\yv_{\Gamma_d}) + \nabla_{y_d} \phi_d(\yv_{S_d}) \big].
$$
By using the reproducing property of the RKHS $\mathcal{H}_d$, we have $\phi_d(\yv_{S_d}) = \langle \phi_d(\cdot), k_d(\cdot, \yv_{S_d})\rangle_{\mathcal{H}_d}$, and thus
$$
\begin{aligned}
&~~~~\mathbb{E}_{q(\yv_{S_d})}\big[ \phi_d(\yv_{S_d})\nabla_{y_d} \log p(y_d|\yv_{\Gamma_d}) + \nabla_{y_d} \phi_d(\yv_{S_d}) \big] \\
& = \Big\langle \phi_d(\cdot), \mathbb{E}_{q(\yv_{S_d})}\big[ k_d(\cdot,\yv_{S_d})\nabla_{y_d} \log p(y_d|\yv_{\Gamma_d}) + \nabla_{y_d} k_d(\cdot,\yv_{S_d}) \big] \Big\rangle_{\mathcal{H}_d} .
\end{aligned}
$$
Following the derivation in \cite{Liu2016KSD} and \cite{Chwialkowski2016KSD}, we can show the optimal solution is $\phi_d^*/\|\phi_d^*\|_{\mathcal{H}_d}$ where
$$
\begin{aligned}
\phi_d^*(\xv_{S_d}) = & \mathbb{E}_{q(\yv_{S_d})}\big[ k_d(\xv_{S_d},\yv_{S_d})\nabla_{y_d} \log p(y_d|\yv_{\Gamma_d}) \\
& + \nabla_{y_d} k_d(\xv_{S_d},\yv_{S_d}) \big].
\end{aligned}
$$

\section{More Experimental Results}

\subsection{Toy Example for SVGD with the IMQ Kernel}

Fig. \ref{Fig1} shows the toy example for SVGD with the IMQ kernel. We can find out that the behavior of the IMQ kernel resembles that of the RBF kernel.

\begin{figure}[!htb]
	\centering
  \includegraphics[width=0.5\textwidth]{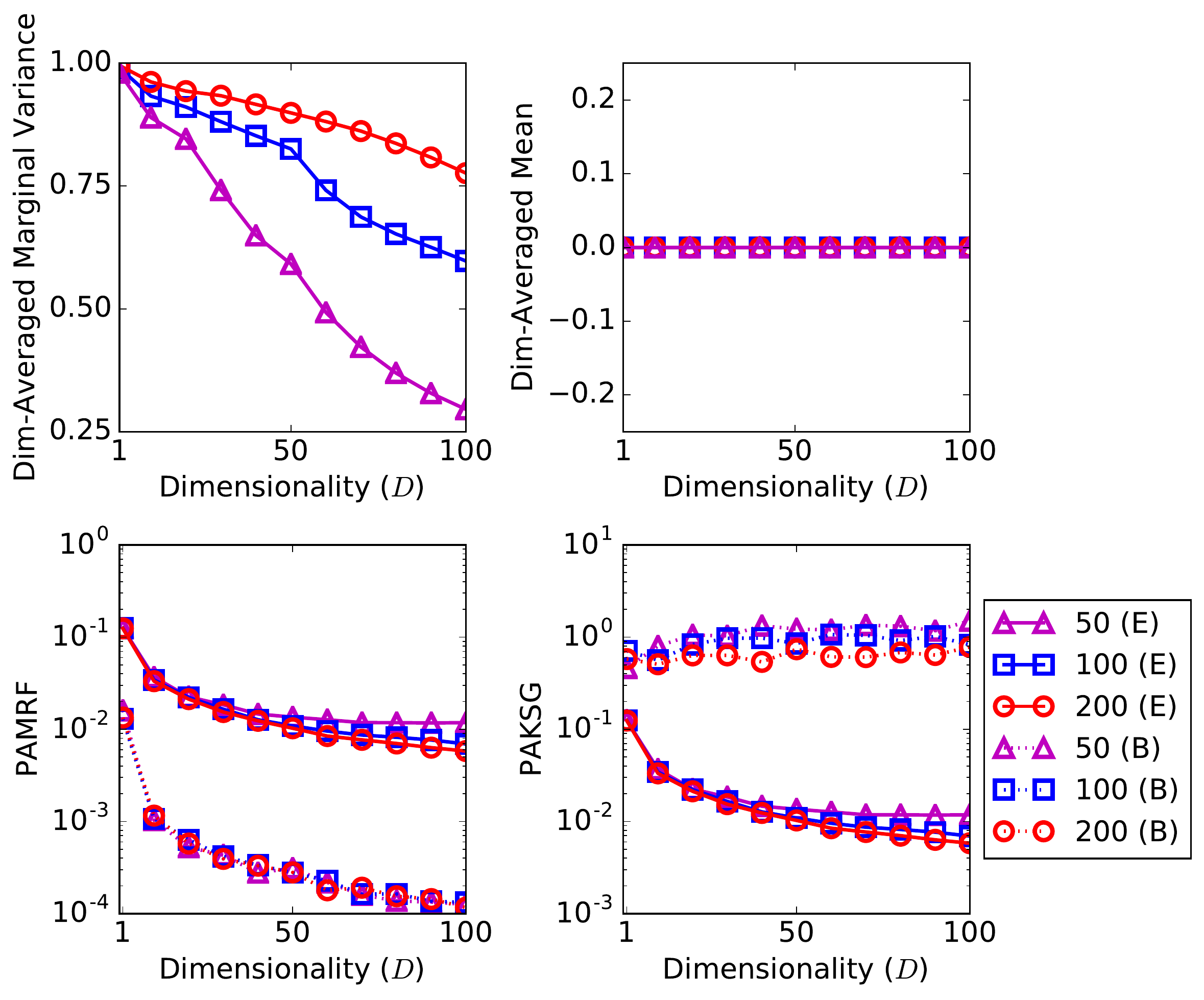}
  \vspace{-.3cm}
	\caption{Results for inferring $p(\xv)=\mathcal{N}(\xv|\zerov,\Iv)$ using SVGD with the IMQ kernel, where particles are initialized by $\mathcal{N}(\xv|\zerov,25 \Iv)$. Top two figures show the dimension-averaged marginal variance $\frac{1}{D}\sum_{d=1}^D \mathrm{Var}_{\hat{q}_M} (x_d)$ and mean $\frac{1}{D}\sum_{d=1}^D \mathbb{E}_{\hat{q}_M}[x_d]$ respectively, and bottom two figures show the particle-averaged magnitude of the repulsive force (PAMRF) $\frac{1}{M} \sum_{i=1}^M \|\Rv(\xv^{(i)};\hat{q}_M)\|_\infty$ and kernel smoothed gradient (PAKSG) $\frac{1}{M} \sum_{i=1}^M \|\Gv(\xv^{(i)};p,\hat{q}_M)\|_\infty$ respectively, at both the beginning (dotted;B) and the end of iterations (solid;E) with different number of particles $M = 50, 100$ and $200$.}
	\label{Fig1}
\end{figure}

\subsection{The Impact of Bandwidth}

\begin{figure}
	\centering
  \includegraphics[width=0.5\textwidth]{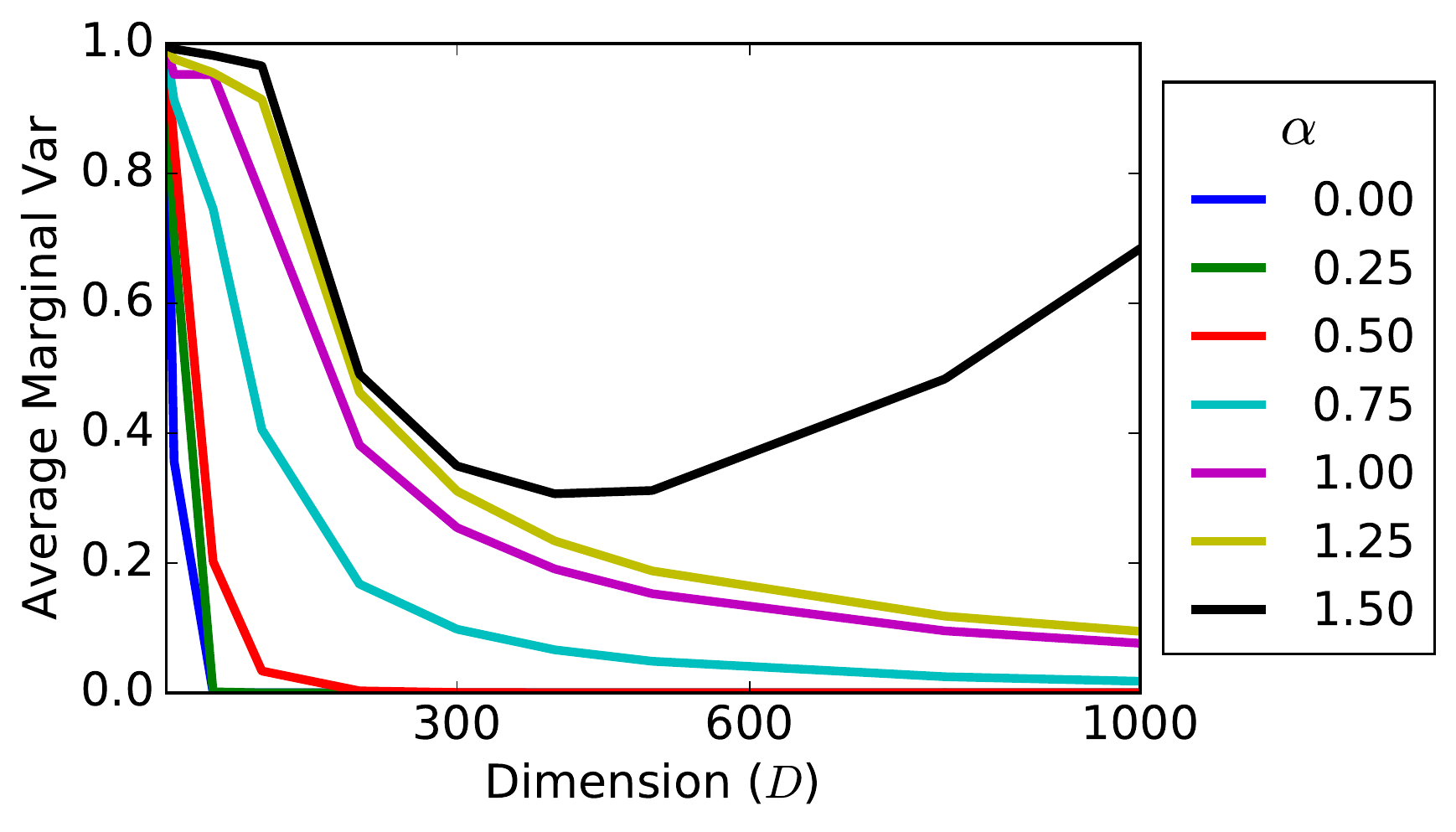}
	\caption{The average marginal variance of particles generated by SVGD with different bandwidth versus dimension.}
	\label{fig0}
\end{figure}
\begin{figure*}
	\centering
    \includegraphics[width=1\textwidth]{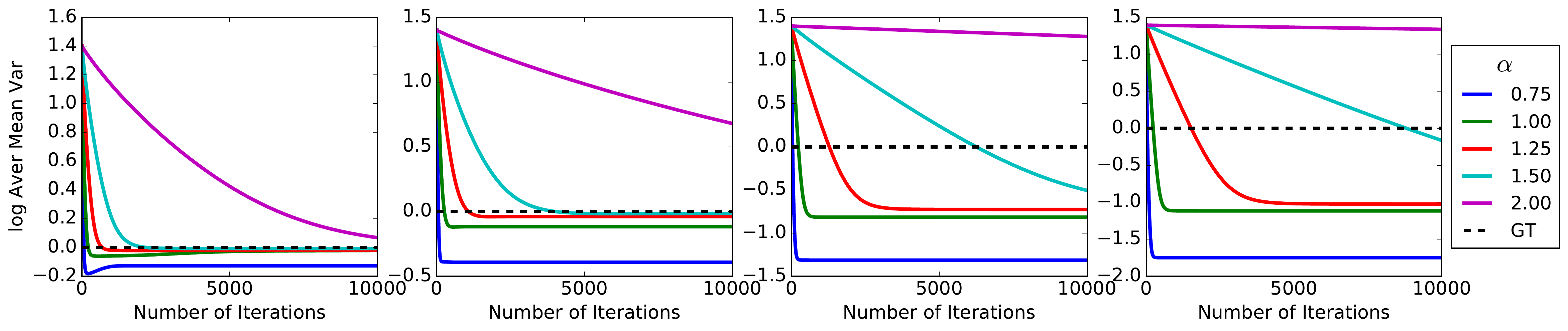}
	\caption{The convergence performance of SVGD with different bandwidth is evaluated for different dimension $D = 50, 100, 500, 1000$ arrange from left to right. ``GT'' denotes the ground truth, which equals one.}
	\label{fig00}
\end{figure*}

Bandwidth plays an important role in kernel methods. 
In this section, we provide additional experimental results for the impact of bandwidth over the performance of SVGD.

In this experiment, we set the target to be $p(\xv) = \mathcal{N}(0,\Iv)$ as a $D$ dimensional isotropic Gaussian distribution, and use $M = 100$ particles initialized as i.i.d examples from $q_0(x) = \mathcal{N}(x|0,25\Iv)$. We use the RBF kernel $k(x,y) = \exp(\frac{-\|x-y\|_2^2}{2h})$, in which the bandwidth $h = D^{\alpha - 1}\cdot\mathrm{med}^2$ with $\alpha = 1$ the median heuristic, $\alpha > 1$ the overestimated bandwidth and $\alpha < 1$ the underestimated bandwidth. We evaluate the quality of particles in marginal approximation by using the average marginal variance $\frac{1}{D}\sum_{d=1}^D \mathrm{Var}_{\hat{q}_M}(x_d)$, which measures the extent to which the particles are diverse to each other in marginals. The average marginal variance of $p(\xv)$ is $1$. For all experiments, we use Adagrad \cite{Duchi2011adagrad} for step size and execute 10000 iterations to get final particles.


Fig. \ref{fig0} demonstrates the relationship among marginal particle diversity, bandwidth choices and dimensions. An interesting observation is that there exists an inflection point around $D = 400$ in the curve of overestimated bandwidth ($\alpha = 1.5$). The reason is that larger bandwidth leads to smaller $\hat{\phi}^*(\xv)$ and thus slower convergence, and this phenomenon deteriorates as dimension increases. Thus, for the bandwidth $\alpha=1,5$ and $D > 400$, SVGD cannot converge with 10000 iterations. Excluding this unconverged case, we can find out that as dimension increases, the approximation deteriorates no matter which bandwidth is chosen. 

Fig. \ref{fig00} demonstrates the dynamic of SVGD with different bandwidth. To highlight the difference, we use the log scale in Y axis. As is shown, bandwidth plays an important role in the convergence of SVGD: smaller bandwidth leads to faster convergence. And the gap between different bandwidth becomes larger as dimension increases. Another observation is, when converged, larger bandwidth corresponds to higher marginal variance, which implies more diverse particles and better marginal approximation. Among these bandwidth choices, the median heuristic ($\alpha = 1$) is
somehow the best one for two reasons: (1) It converges almost as fast as the underestimated bandwidth ($\alpha < 1$); (2) It achieves almost the best marginal variance. Though overestimated bandwidth ($\alpha > 1$) achieves slightly better performance than the median heuristic when converged, the gap is not as large as that between the median heuristic and underestimated bandwidth. For example, in the rightmost figure, the gap of the average marginal variance between $\alpha = 1.25$ and $\alpha = 1$ is much smaller than that between $\alpha = 0.75$ and $\alpha = 1$.

\subsection{Synthetic Markov Random Fields}
Fig. \ref{Fig4} compares EP with other methods mentioned in the main body. Due to the strong Gaussian assumption, EP achieves the highest RMSE compared to other methods.

\begin{figure*}[!htb]
	\centering
	\includegraphics[width=1\textwidth]{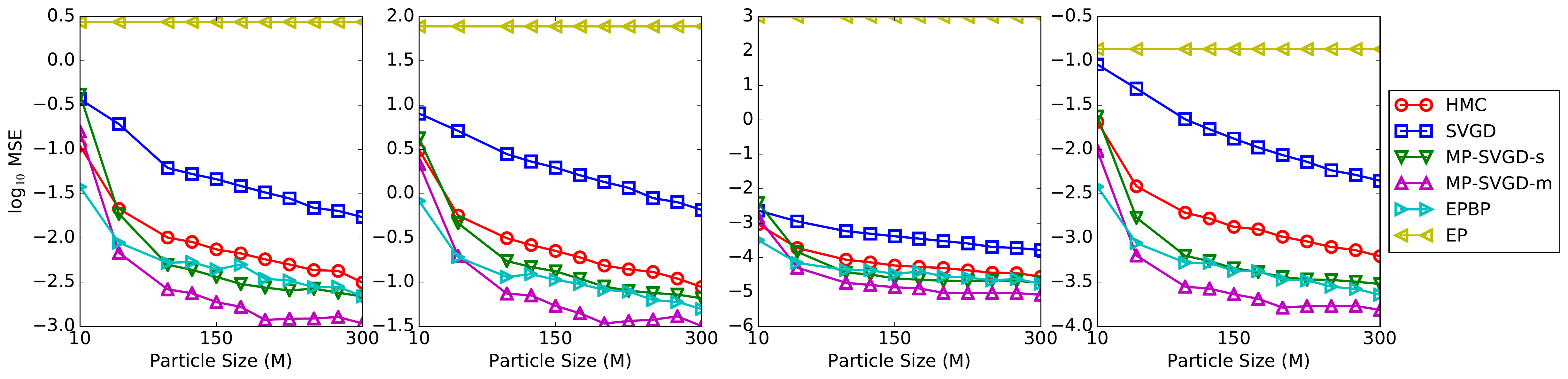}	
	\caption{A quantitative comparison of inference methods with varying number of particles. Performance is measured by the MSE of the estimation of expectation $\mathbb{E}_{\xv \sim \hat{q}_M}[\fv(\xv)]$ for test functions $\fv(\xv) = \xv$, $\xv^2$, $1/ (1+\exp(\omegav \circ \xv + \bv))$ and $\cos(\omegav \circ \xv + \bv)$, arranged from left to right, where $\circ$ denotes the element-wise product. Results are averaged over 10 random draws of $\omegav$ and $\bv$, where $\omegav, \bv \in \mathbb{R}^{100}$ with $\omega_d \sim \mathcal{N}(0,1)$ and $b_d \in \mathrm{Uniform}[0, 2\pi]$, $\forall d \in \{1,...,100\}$.}
	\label{Fig4}
\end{figure*}

\begin{figure*}[!htb]
	\centering
	\includegraphics[width=0.1373\textwidth]{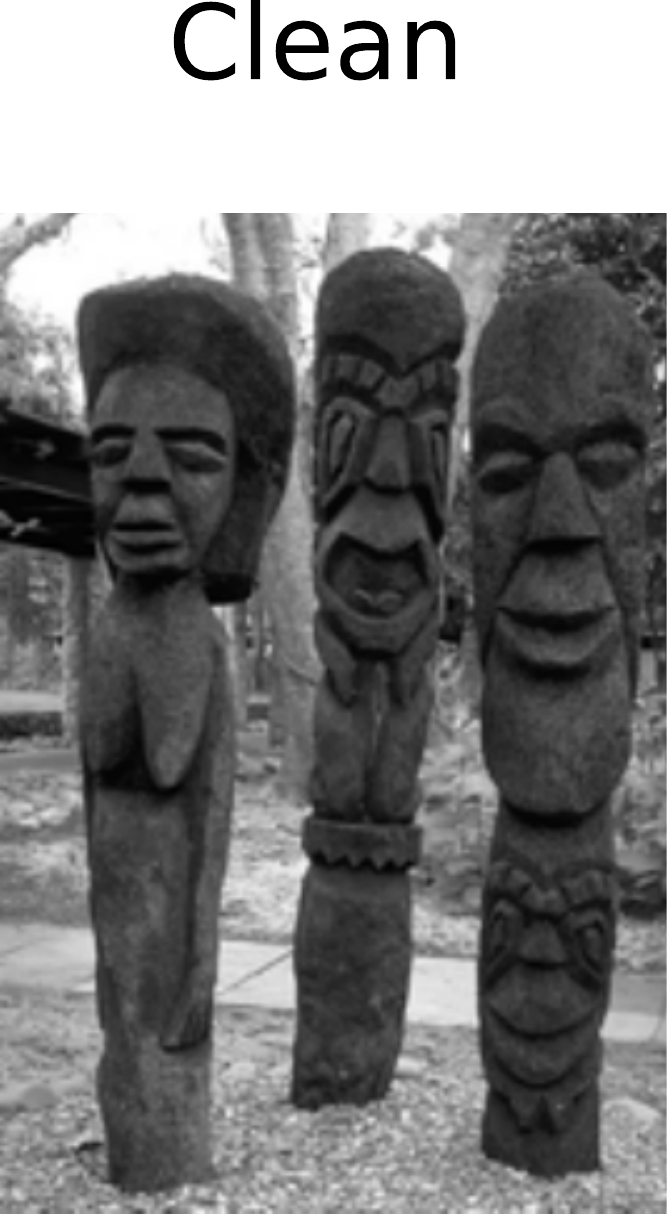}
	\includegraphics[width=0.16\textwidth]{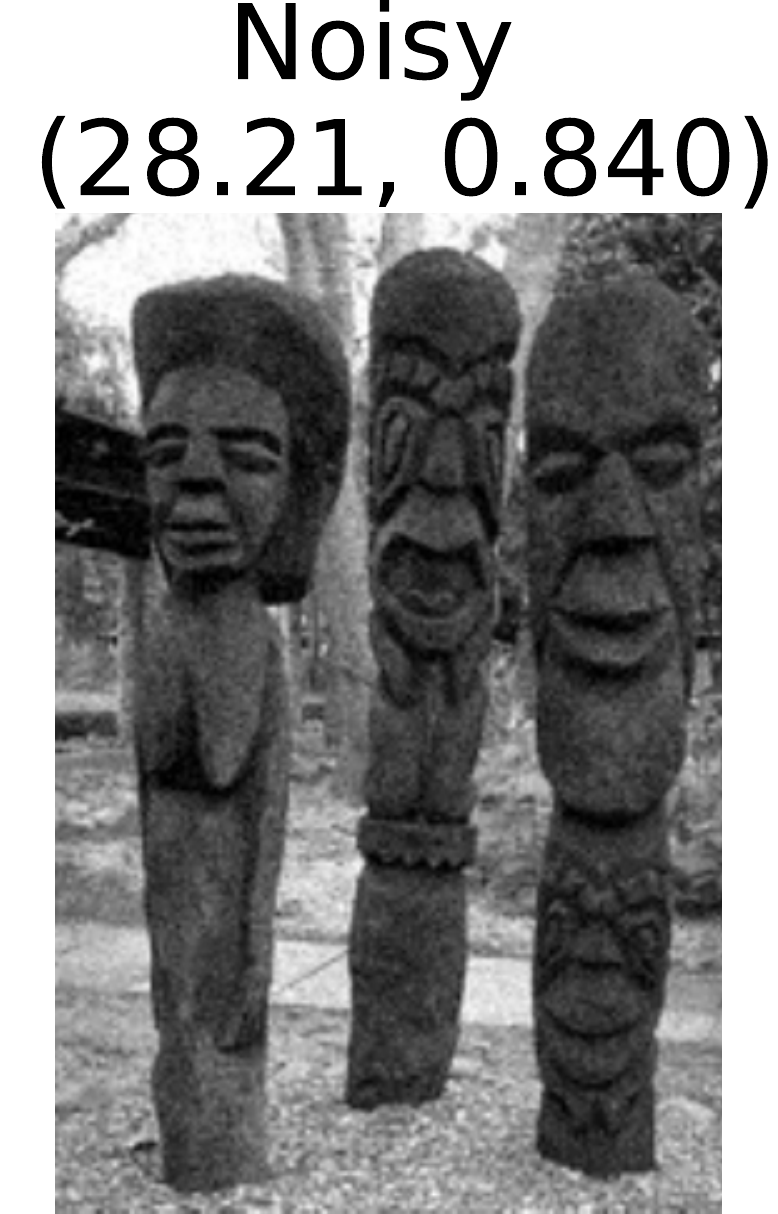}
	\includegraphics[width=0.16\textwidth]{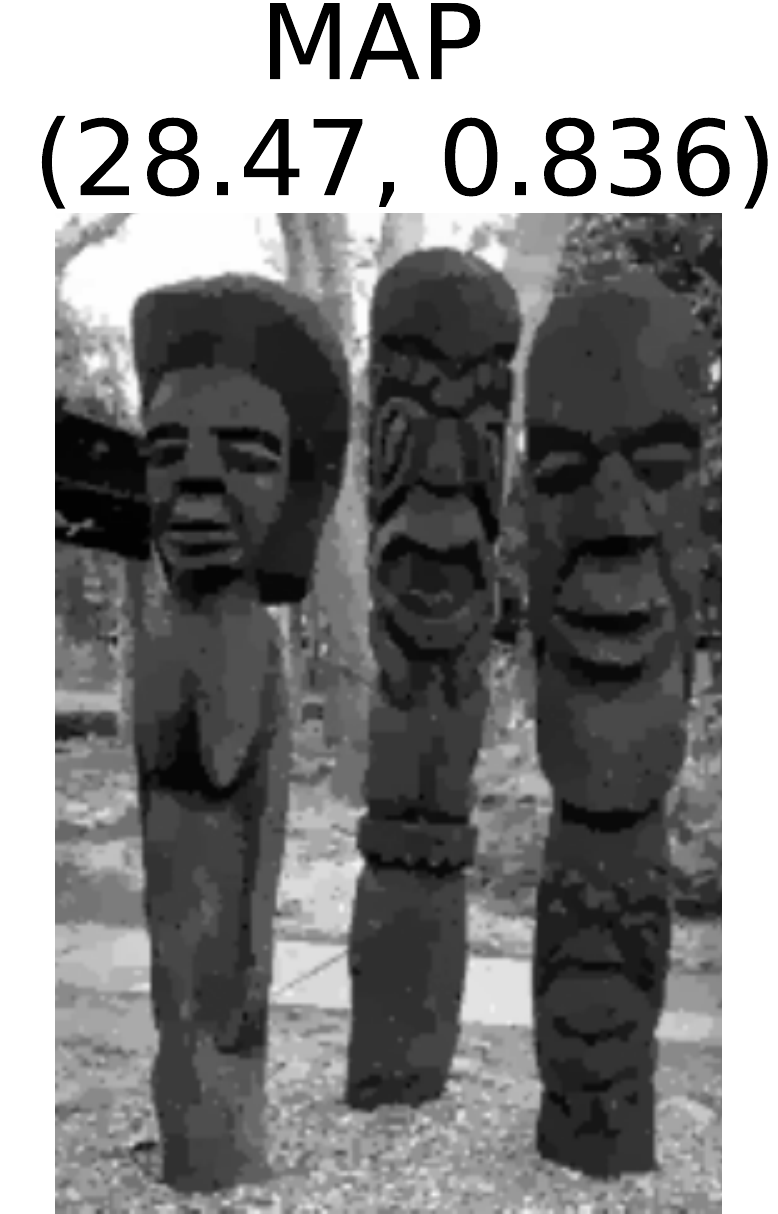}
	\includegraphics[width=0.16\textwidth]{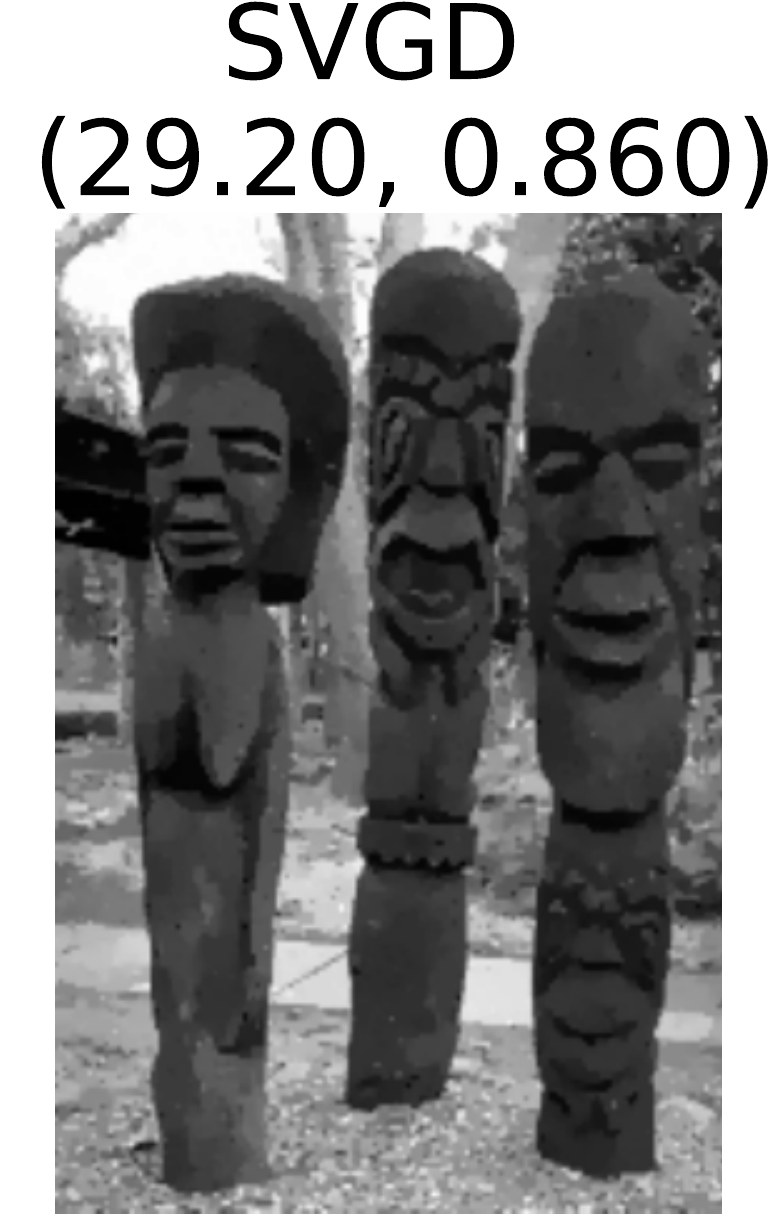}
	\includegraphics[width=0.16\textwidth]{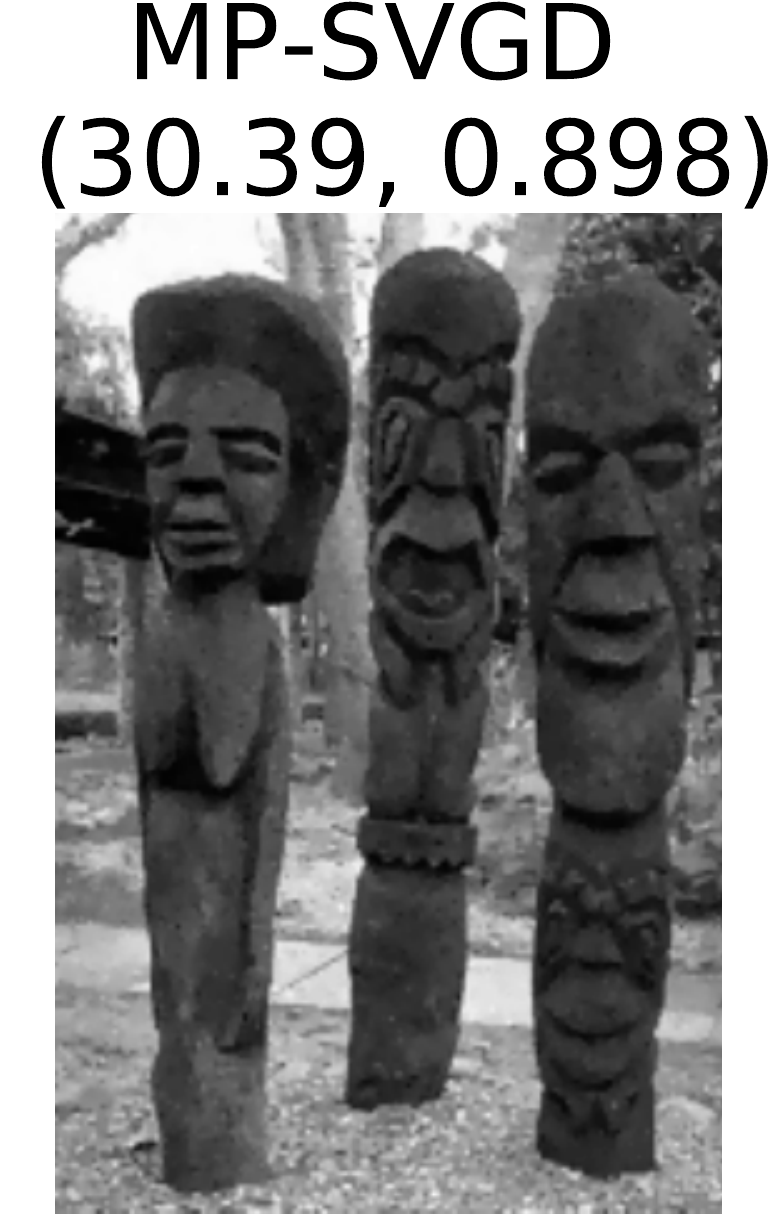}
	\includegraphics[width=0.16\textwidth]{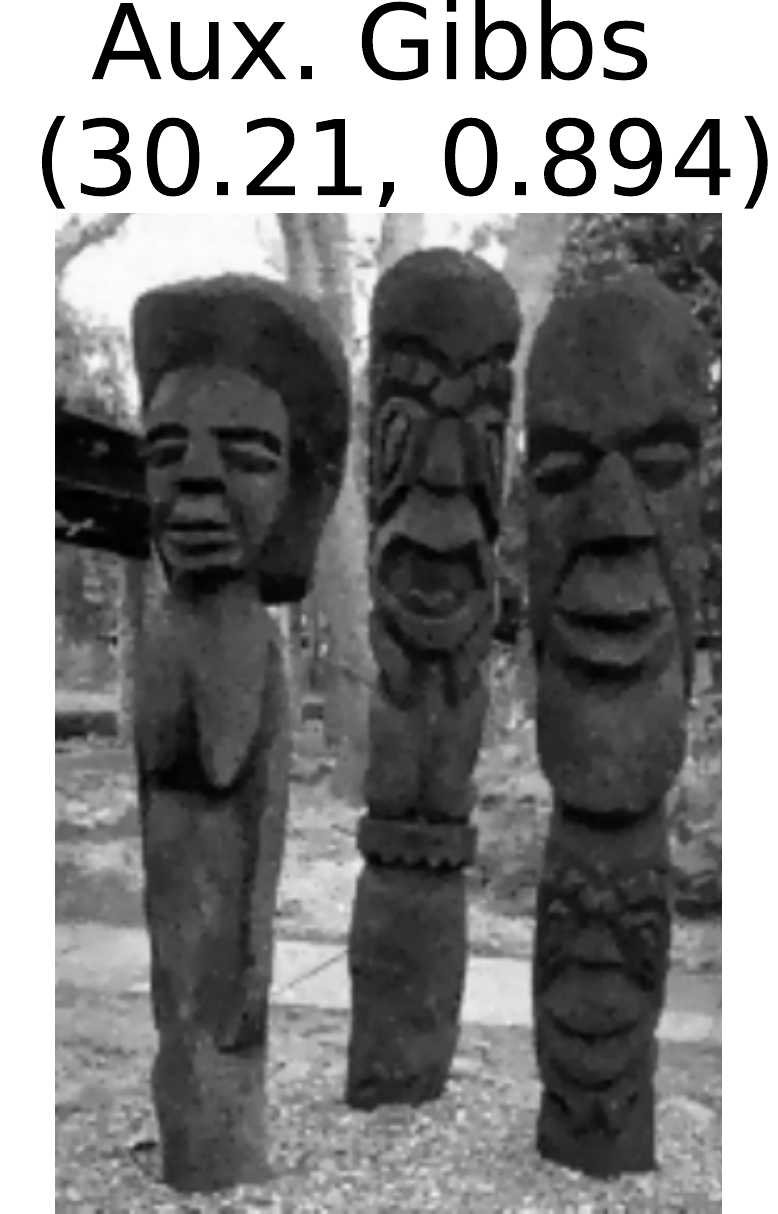}	
	\\	
	\vspace{0.3cm}
	\includegraphics[width=0.1373\textwidth]{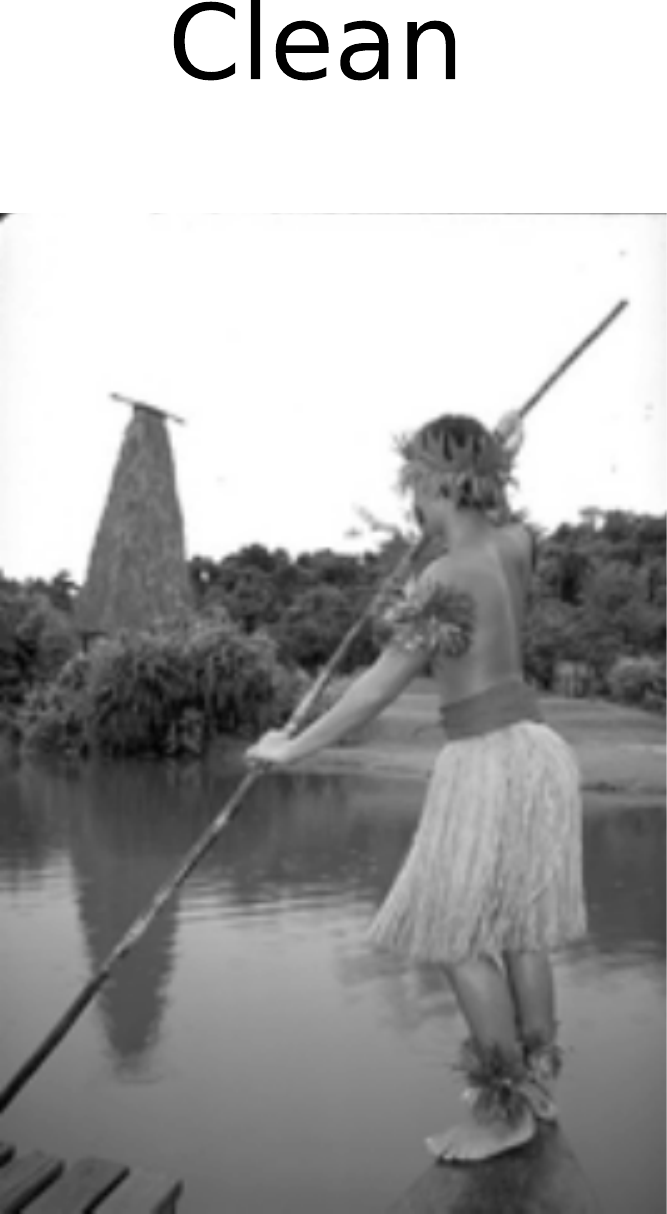}
	\includegraphics[width=0.16\textwidth]{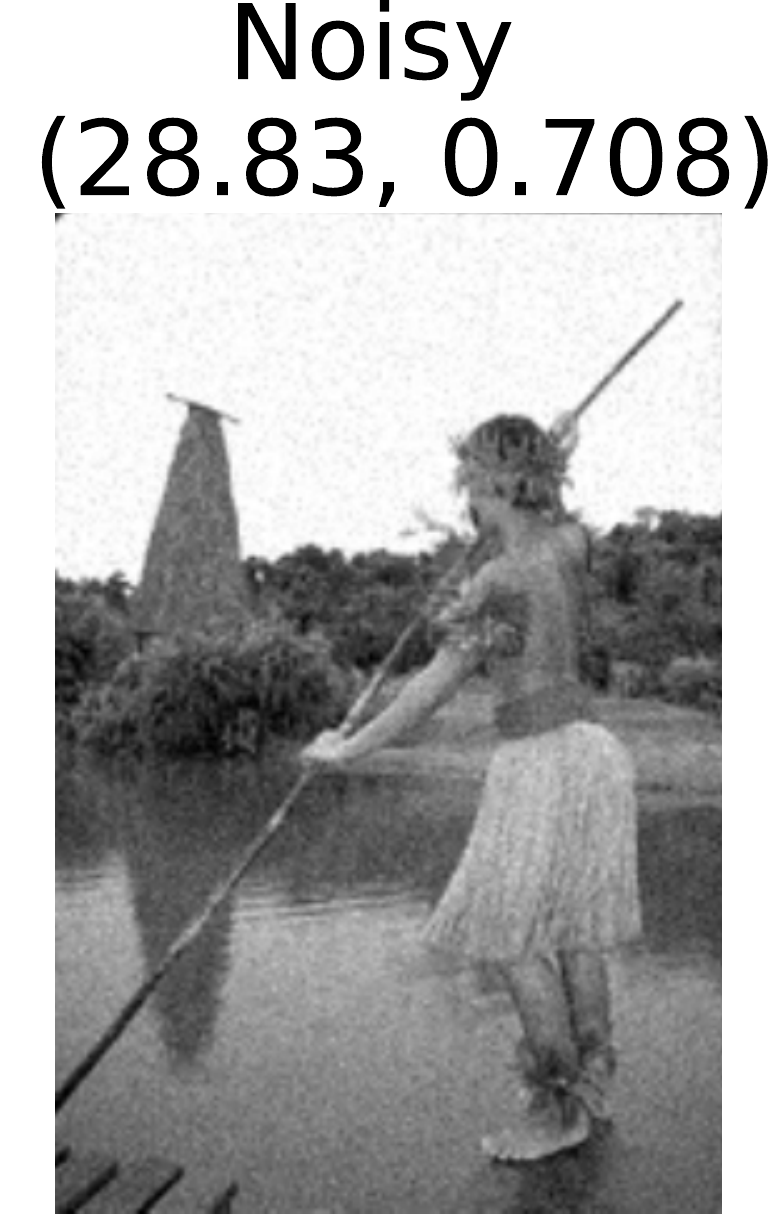}
	\includegraphics[width=0.16\textwidth]{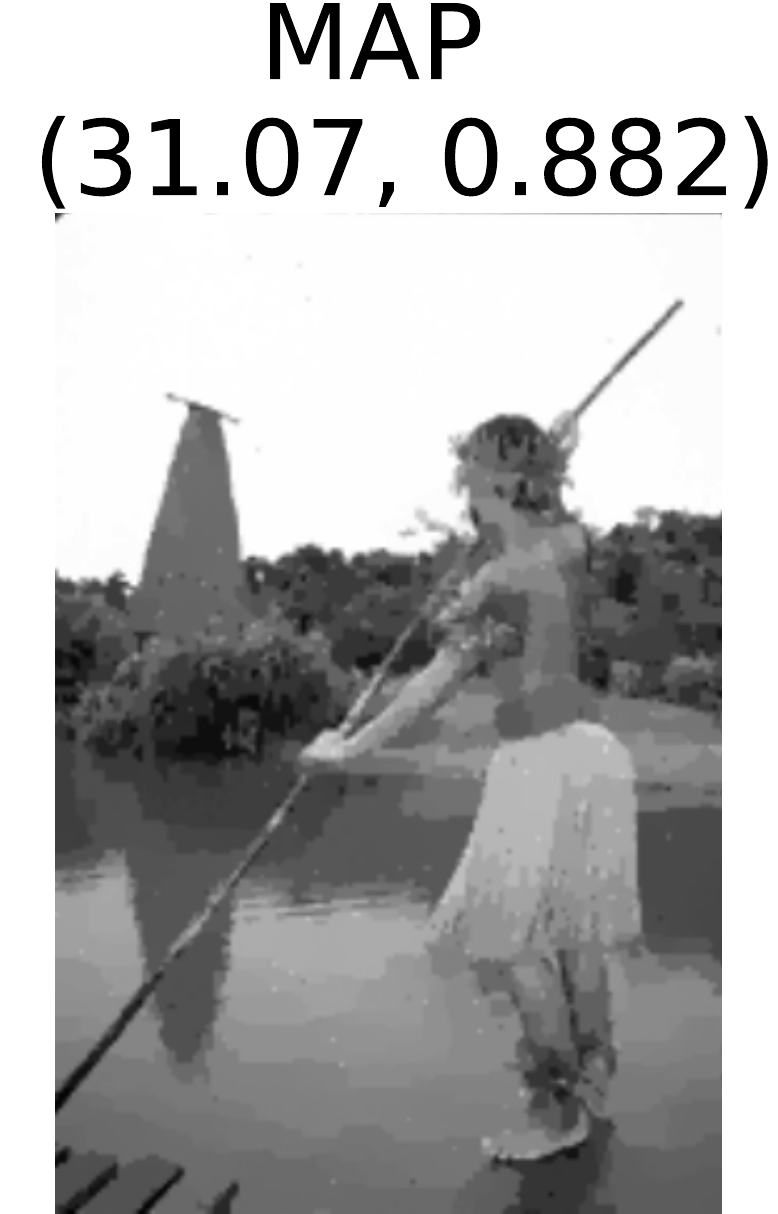}
	\includegraphics[width=0.16\textwidth]{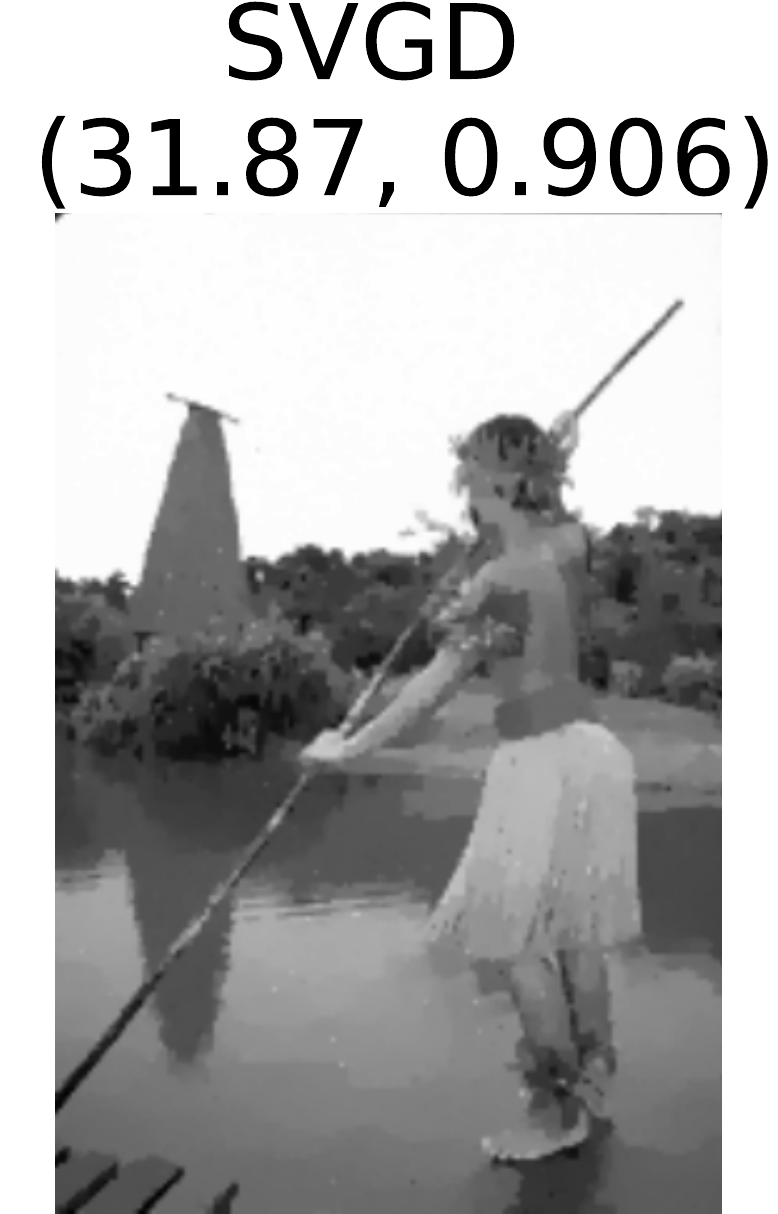}
	\includegraphics[width=0.16\textwidth]{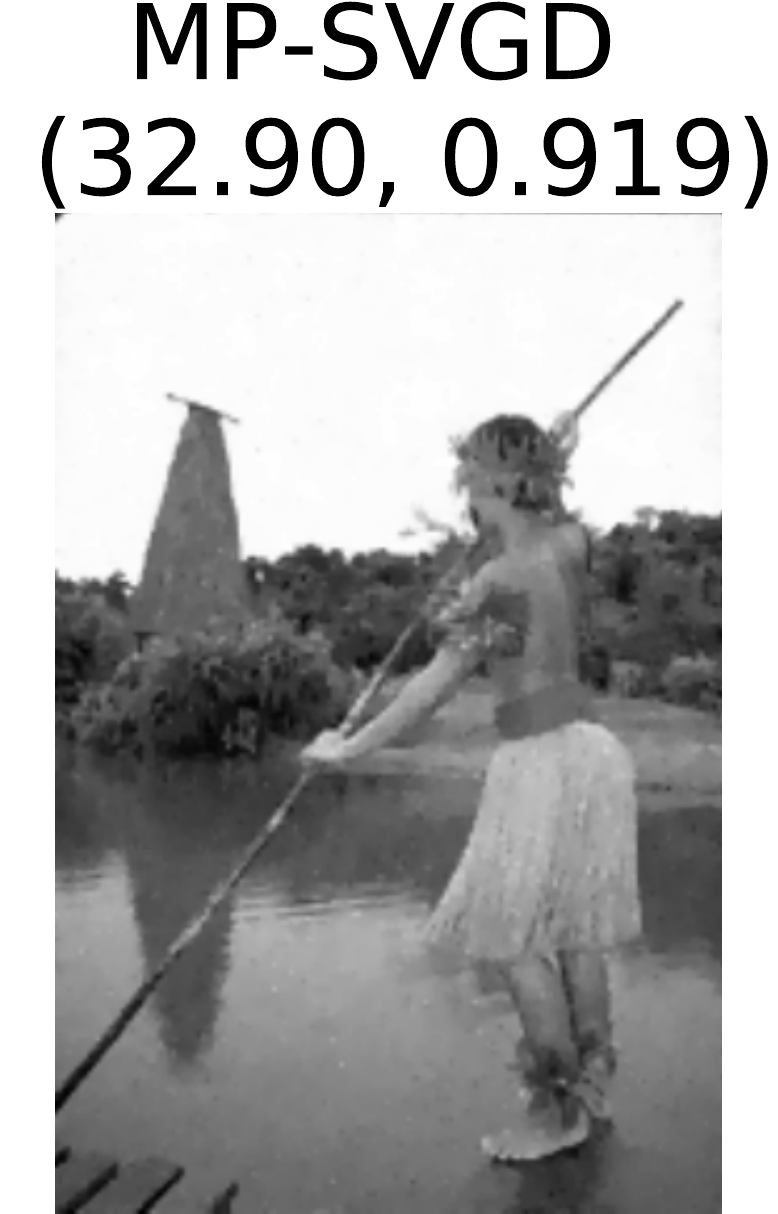}
	\includegraphics[width=0.16\textwidth]{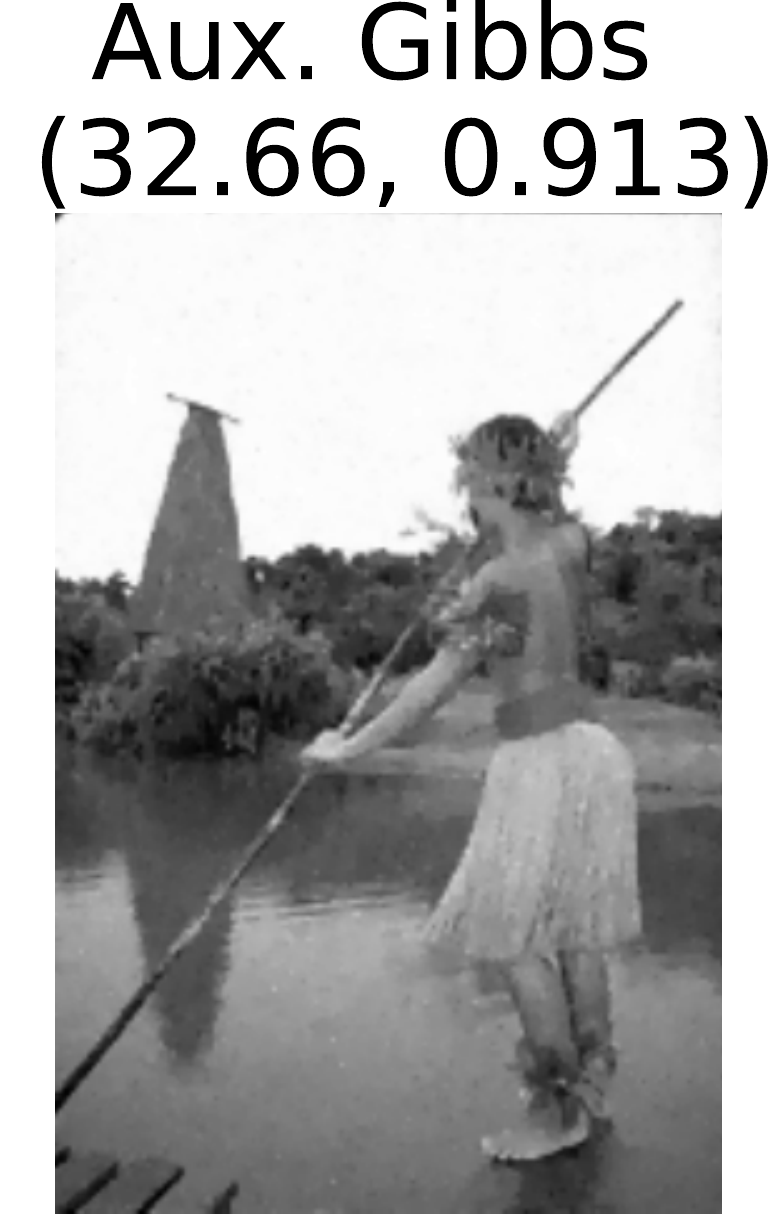}	
	\\
    \vspace{0.3cm}
	\includegraphics[width=0.1373\textwidth]{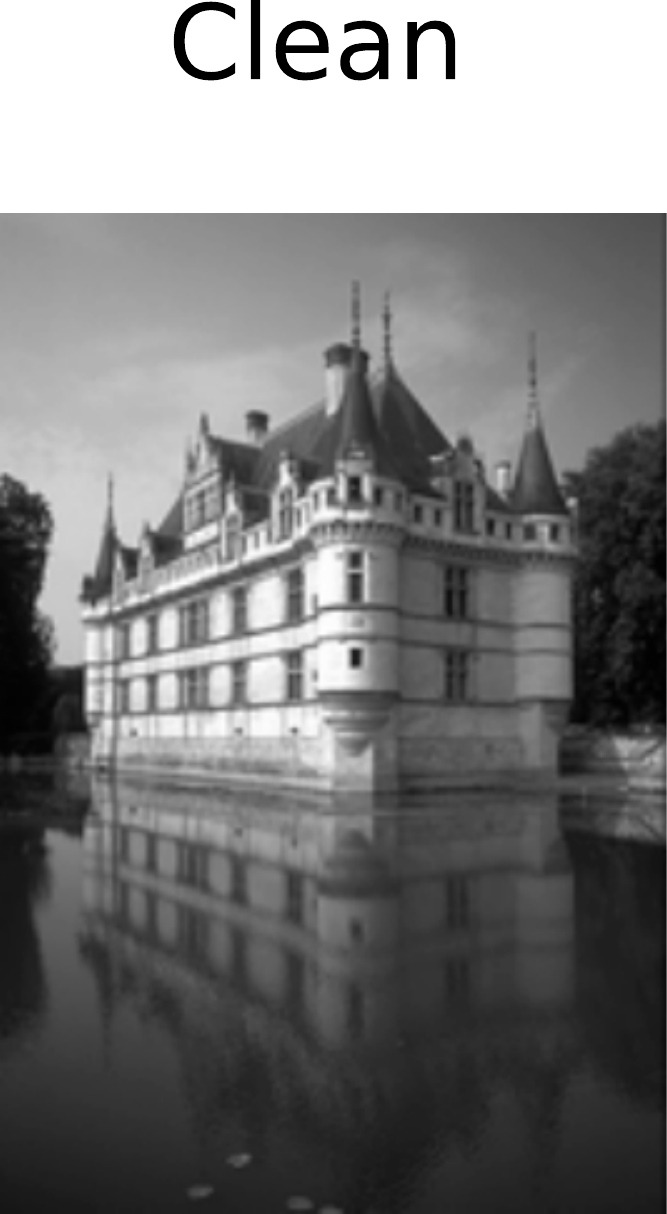}
	\includegraphics[width=0.16\textwidth]{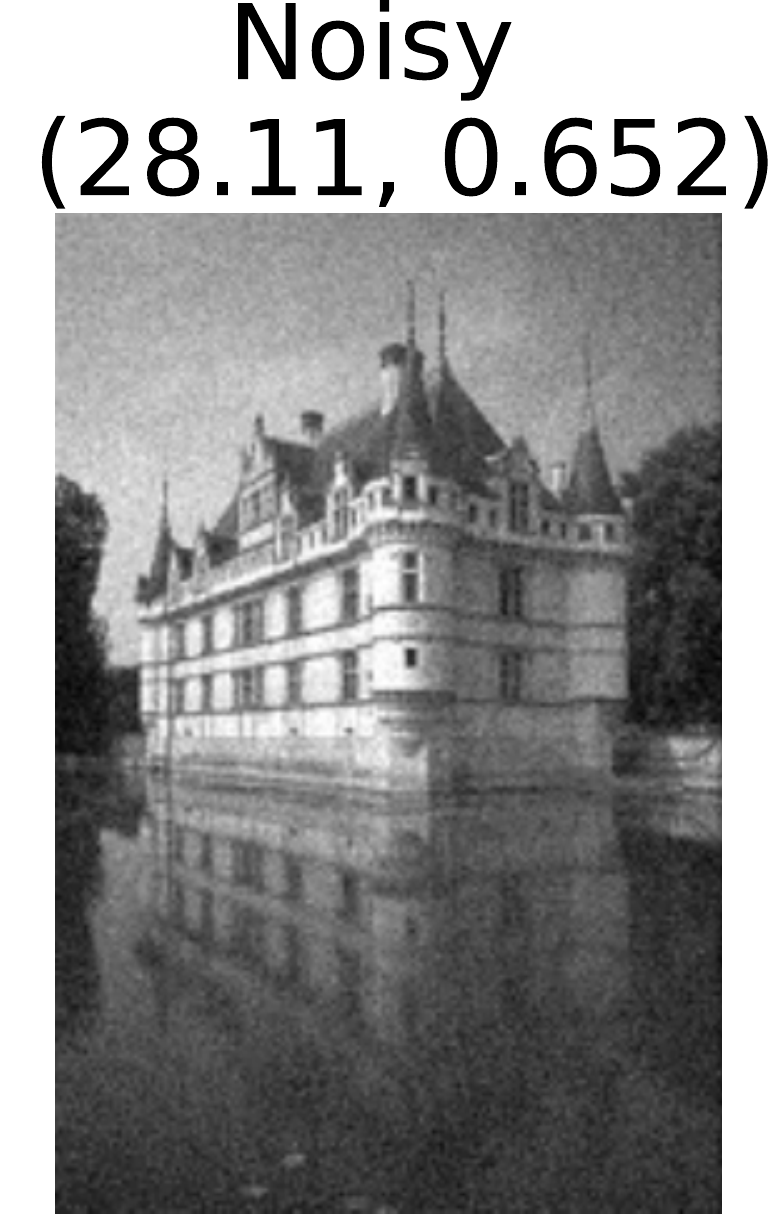}
	\includegraphics[width=0.16\textwidth]{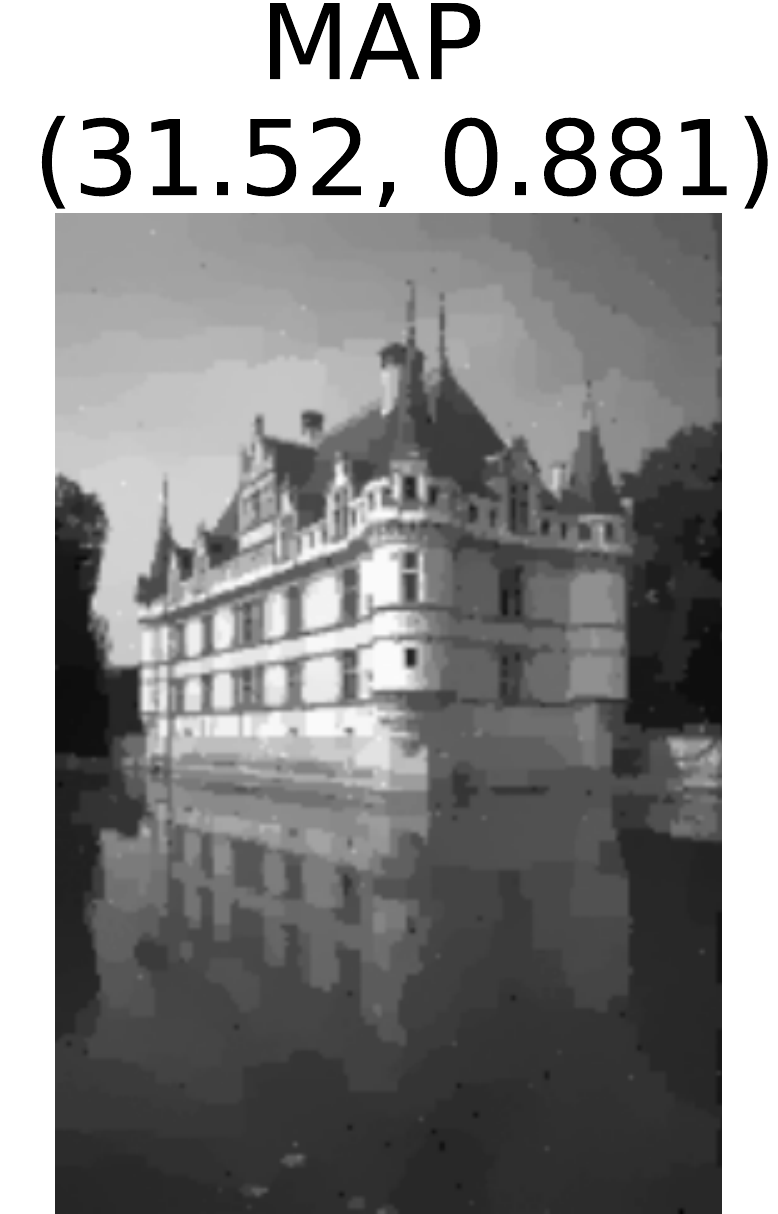}
	\includegraphics[width=0.16\textwidth]{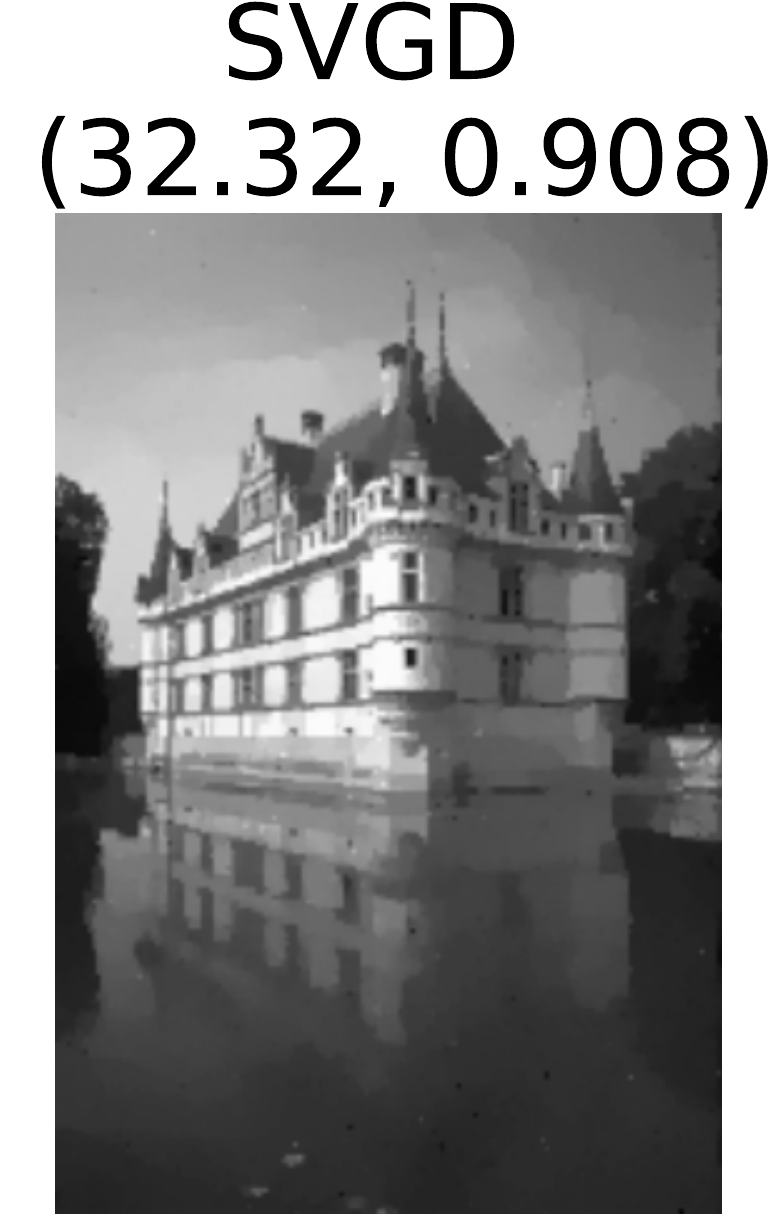}
	\includegraphics[width=0.16\textwidth]{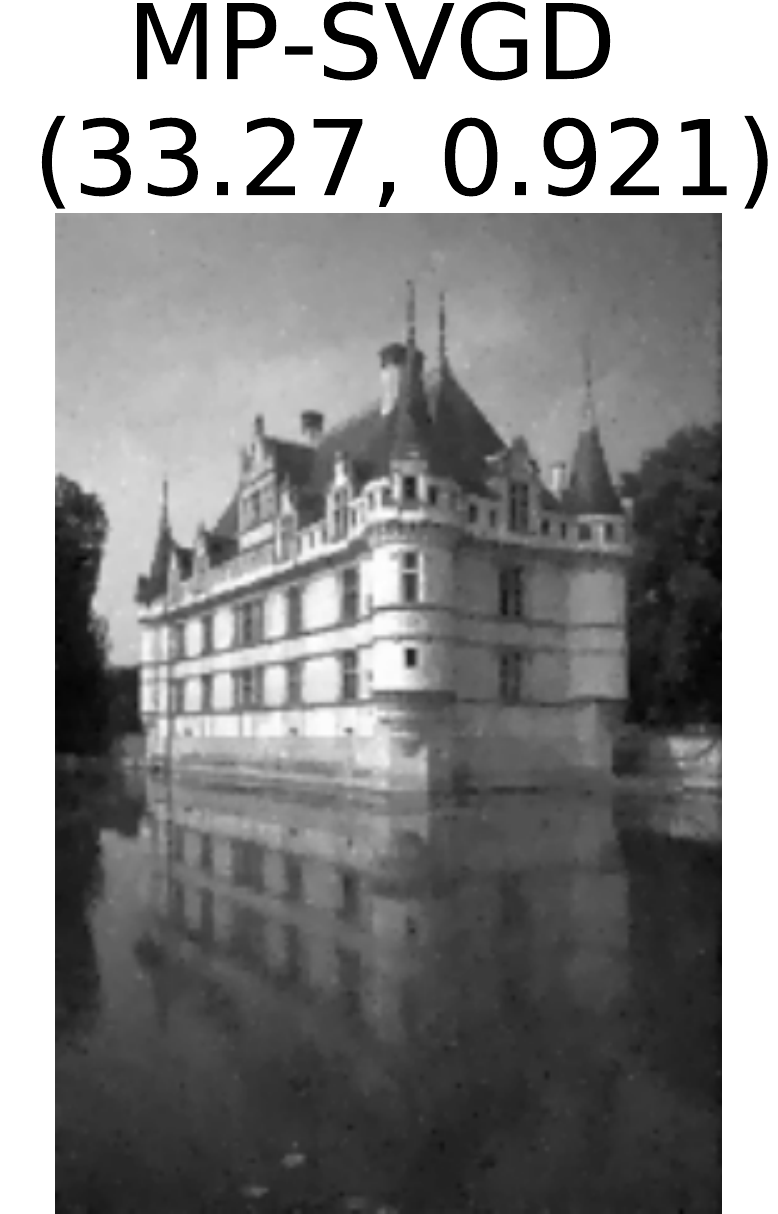}
	\includegraphics[width=0.16\textwidth]{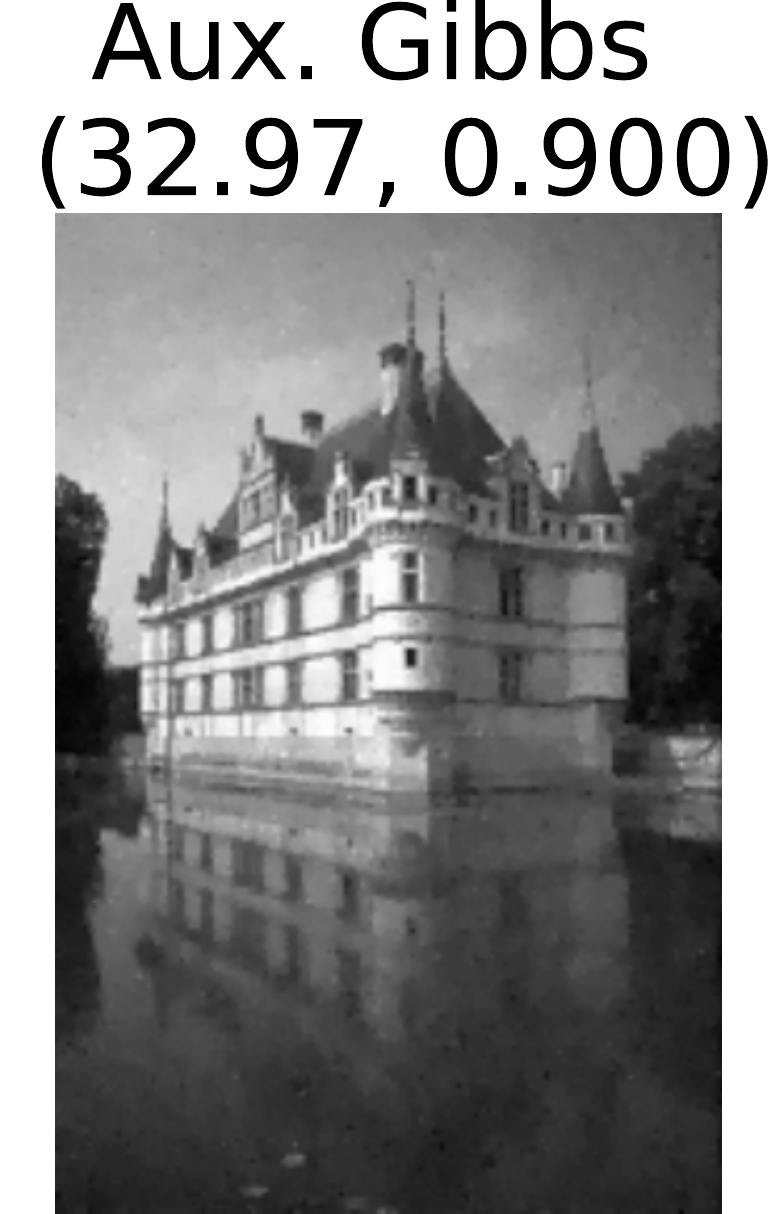}				
	\caption{Extra denoising results on BSD dataset using 50 particles, $240 \times 160$ pixels, $\sigma_n = 10$. The number in bracket is PSNR and SSIM.}
	\label{Fig5}
\end{figure*}

\subsection{Image Denoising}
Fig. \ref{Fig5} shows more denoising examples apart from {\it Lena} in the main body.

\bibliography{ref}
\bibliographystyle{icml2018}